\documentclass{article}

\PassOptionsToPackage{numbers, compress}{natbib}

\usepackage[preprint]{neurips_2025}

\usepackage[utf8]{inputenc} 
\usepackage[T1]{fontenc}    
\usepackage{hyperref}       
\usepackage{url}            
\usepackage{booktabs}       
\usepackage{amsfonts}       
\usepackage{nicefrac}       
\usepackage{microtype}      
\usepackage{xcolor}         

\usepackage{wrapfig}

\usepackage{comment}

\RequirePackage{algorithm}
\RequirePackage{algorithmic}

\usepackage{amsmath}
\usepackage{amssymb}
\usepackage{mathtools}
\usepackage{amsthm}
\usepackage[caption=false,font=normalsize,labelfont=sf,textfont=sf]{subfig}

\usepackage[capitalize,noabbrev]{cleveref}

\usepackage{xr}
\makeatletter
\newcommand*{\addFileDependency}[1]{
  \typeout{(#1)}
  \@addtofilelist{#1}
  \IfFileExists{#1}{}{\typeout{No file #1.}}
}
\makeatother

\theoremstyle{plain}
\newtheorem{theorem}{Theorem}[section]
\newtheorem{proposition}[theorem]{Proposition}
\newtheorem{lemma}[theorem]{Lemma}

\theoremstyle{definition}
\newtheorem{definition}[theorem]{Definition}
\newtheorem{assumption}[theorem]{Assumption}
\theoremstyle{remark}
\newtheorem{remark}[theorem]{Remark}

\title{The Actor-Critic Update Order Matters for PPO in Federated Reinforcement Learning}

\author{%
  Zhijie Xie \\
  Department of Electronic and Computer Engineering\\
  The Hong Kong University of Science and Technology\\
  Hong Kong, China \\
  \texttt{zhijie.xie@connect.ust.hk} \\
  \And
  Shenghui Song \\
  Department of Electronic and Computer Engineering\\
  The Hong Kong University of Science and Technology\\
  Hong Kong, China \\
  \texttt{eeshsong@ust.hk} \\
}

\begin{document}

\maketitle

\begin{abstract}
In the context of Federated Reinforcement Learning (FRL), applying Proximal Policy Optimization (PPO) faces challenges related to the update order of its actor and critic due to the aggregation step occurring between successive iterations. In particular, when local actors are updated based on local critic estimations, the algorithm becomes vulnerable to data heterogeneity. As a result, the conventional update order in PPO (critic first, then actor) may cause heterogeneous gradient directions among clients, hindering convergence to a globally optimal policy. To address this issue, we propose FedRAC, which reverses the update order (actor first, then critic) to eliminate the divergence of critics from different clients. Theoretical analysis shows that the convergence bound of FedRAC is immune to data heterogeneity under mild conditions, i.e., bounded level of heterogeneity and accurate policy evaluation. Empirical results indicate that the proposed algorithm obtains higher cumulative rewards and converges more rapidly in five experiments, including three classical RL environments and a highly heterogeneous autonomous driving scenario using the SUMO traffic simulator. 
\end{abstract}

\section{Introduction\label{sec:Introduction}}


As one of the most prominent Reinforcement Learning (RL) algorithms, Actor-Critic methods are equipped with rich theory, ranging from the Policy Gradient (PG) theorem \cite{NIPS1999_6449f44a,NIPS1999_464d828b} to the convergence analysis of neural Proximal Policy Optimization (PPO) and Trust Region Policy Optimization (TRPO) \cite{NEURIPS2019_227e072d}.
However, the theoretical properties of the federated Actor-Critic framework remained largely unexplored for an extended period until recently. FedPG-BR tackled the fault tolerance issue of federated Policy Gradient (PG) methods, but they assume homogeneous clients \cite{fan2021fault}. PAvg provided the first convergence proof for federated PG methods with environment heterogeneity \cite{pmlr-v151-jin22a}. FedKL proposed a Kullback-Leibler (KL) divergence-regularized objective for heterogeneous FRL clients and achieved a similar monotonicity property as the original TRPO for centralized RL \cite{10038492}. More recently, FedNPG/FedNAC investigated the convergence properties of federated natural PG and Actor-Critic with neural network parameterization \cite{NEURIPS2024_dbdea785}.

Despite the above progress, many challenges still need to be tackled before we can fully exploit the power of FRL, especially those inherited from either FL or RL. Notably, the data heterogeneity issue persists, as different clients may operate under different environment dynamics, leading to various local objectives and consequently different gradient estimations \cite{10038492}. Furthermore, weight divergence is observed even in scenarios with Independent and Identically Distributed (IID) data \cite{DBLP:journals/corr/abs-1806-00582}. More importantly, the unique structure of the Actor-Critic framework complicates the synchronization of different models, especially in terms of the update order of the actor and critic.

To this end, a pertinent question arises: should the actor or the critic be updated first in FRL? This is not a critical question for centralized PPO, as alternatively updating the actor and critic does not affect the performance \cite{NIPS1999_6449f44a,10160088}. However, choosing the update order is nontrivial for federated PPO as the aggregation step breaks the update sequence of actors and critics into successive rounds, resulting in different sequences for different update orders: 1) when clients update their critics first, their actors are updated by locally accurate but globally heterogeneous critics. This facilitates local learning but may hinder global consensus; 2) when clients update their actors first, their actors are updated by the global critic obtained from the previous round. This may hinder local learning but facilitate global consensus. FedNAC \cite{NEURIPS2024_dbdea785} noticed this breakdown and necessitated a two-stage communication scheme, i.e., one for synchronizing the global critic and another for synchronizing the global actor. The two-stage communication scheme is free from choosing the update order but is exceedingly costly for real-world applications, e.g., edge computing and vehicular communication. As far as the author knows, the question regarding the optimal update order for federated Actor-Critic remains unanswered, motivating us to study the update sequence of PPO's actor and critic within one communication round. The main contributions of this work are summarized as follows:
\begin{itemize}
\item We propose FedRAC for training PPO learners in heterogeneous networks, where the widely utilized update order (critic first, then actor) is reversed. In particular, every client updates its actor based on the same gradient estimator, i.e., the global critic from the last communication round, to mitigate the negative effect of data heterogeneity.

\item We prove that FedRAC can converge at the same rate of $\mathcal{O}(\log \left\vert \mathcal{A} \right\vert / T)$ as the baseline, but is less sensitive to data heterogeneity. Our analysis covers both Softmax and Gaussian policies and investigates them in a unified view.
Moreover, the theoretical analysis indicates that FedRAC is sensitive to the accuracy of the critic estimation.

\item Experiment results show that FedRAC outperforms the baseline across various environments, including a continuous autonomous vehicle training problem. Additionally, it is demonstrated that integrating FedRAC with other FL/FRL algorithms, such as FedAvg, FedProx, and SCAFFOLD, only requires minor modifications. However, the efficacy of FedRAC is significantly influenced by the quality of policy evaluation.
\end{itemize} 

\section{Problem Formulation \label{sec:problem_formulation}}

In this section, we introduce some preliminaries and formulate the FRL problem.

\subsection{RL Problem, Actor-Critic Methods, and PPO}

We formulate the RL problem in the $n$-th client as the discrete-time continuous Markov Decision Process (MDP), which is denoted as a 6-tuple $(\mathcal{S},\mathcal{A},\mu,P_{n},\mathcal{R},\gamma)$, where $\mathcal{S}$, $\mathcal{A}$, and $\mu$ are the state space, the action space, and the initial state distribution, respectively. The transition function $P_{n}(s^{\prime}\vert s,a):\mathcal{S}\times\mathcal{S}\times\mathcal{A}\to\mathbb{R}$ gives the probability that the MDP transits from state $s$ to $s^{\prime}$ after taking action $a$, the reward function $\mathcal{R}(s,a):\mathcal{S}\times\mathcal{A}\to\left[ -R_{\max},R_{\max} \right]$ is bounded by $R_{\max} < \infty$ and gives the expected immediate reward for the state-action pair $(s,a)$, and $\gamma\in(0,1)$ is the discount factor. An agent interacts with the environment to collect data and optimize its control policy $\pi$ in terms of the expected discounted reward
\begin{alignat}{1}
\eta_{n}(\pi) &= \mathbb{E}_{s_{0} \thicksim \mu,a_{t} \thicksim \pi,s_{t+1} \thicksim P_{n}} \left[\sum_{t=0}^{\infty}\gamma^{t}\mathcal{R}(s_{t},a_{t})\right],\label{eq:rlobj}
\end{alignat}
where the notation $\mathbb{E}_{s_{0} \thicksim \mu_{n},a_{t} \thicksim \pi,s_{t+1} \thicksim P_{n}}$ indicates that the reward is averaged over all states and actions according to the initial state distribution, the transition probability, and the policy. While (\ref{eq:rlobj}) is a concise definition, it is not immediately useful for optimization. Alternatively, almost all model-free RL methods involve estimating the value functions. We define the state-value function of policy $\pi$ as
$V^{\pi}_{n}(s) =\mathbb{E}_{\pi,P_{n}}\left[\sum_{t=0}^{\infty} \gamma^{t}\mathcal{R}(s_{t},a_{t}) \vert s_{0}=s\right],$
where the expectation is taken over actions sampled from policy $\pi$ and states sampled from the transition probability $P_{n}$. Then, we can rewrite (\ref{eq:rlobj}) as $\eta_{n}(\pi) = \mathbb{E}_{s_{0} \thicksim \mu} \left[ V^{\pi}_{n}(s_{0}) \right]$. In addition, we define the action-value function as: $Q^{\pi}_{n}(s,a) = \mathbb{E}_{\pi,P_{n}} \left[\sum_{t=0}^{\infty} \gamma^{t}\mathcal{R}(s_{t},a_{t}) \vert s_{0}=s,a_{0}=a\right]$.
By the Policy Gradient (PG) theorem \cite{NIPS1999_464d828b}, the gradient of (\ref{eq:rlobj}) with respect to the parameters $\theta$ of a differentiable policy $\pi^{\theta}$ is defined as $\hat{g} = \frac{1}{1 - \gamma} \mathbb{E}_{\rho_{\pi^{\theta},n},\pi^{\theta}}\left[ \nabla_{\theta} \log \pi^{\theta}(a\vert s)Q^{\pi^{\theta}}_{n}(s,a)\right]$, where $\rho_{\pi^{\theta},n}(s) = \sum_{t=0}^{\infty}\gamma^{t} \mbox{Pr}(s_{t}=s|\pi^{\theta},\mu,P)$
denotes the (unnormalized) discounted state visitation frequency, $\nabla_{\theta}$ denotes the derivative with respect to $\theta$, and $\mathbb{E}_{\rho_{\pi^{\theta},n},\pi^{\theta}}\left[ \dots \right]$ indicates that states are sampled from the (normalized) discounted visitation frequency $(1 - \gamma)\rho_{\pi^{\theta},n}$.


One popular class of algorithm for solving (\ref{eq:rlobj}) is the Actor-Critic methods \cite{NIPS1999_6449f44a, NIPS1999_464d828b}, where the actor refers to the policy $\pi$ and the critic refers to the value function. The Actor-Critic methods iteratively update the policy as follows: 1) Interacting with the environment to collect data; 2) Consulting the critic about the gradient in $\hat{g}$; and 3) Updating the policy with $\hat{g}$. To optimize the parameterized policy $\pi^{\theta}$, a parameterized critic $Q^{w}$ is usually trained to estimate how the policy should be updated. In other words, the critic approximates the action-value function $Q^{\pi^{\theta}}_{n}$ in $\hat{g}$. Among all Actor-Critic methods, PPO has seen success in a wide range of applications. It utilizes a theoretically justified KL-regularized objective function which prevents the updated policy from going too far away from the current policy. We will introduce this objective in Section \ref{subsec:policy_improvement}.

It is worth noticing that the above update order (critic first, then actor) is not the only option. Although this update order is theoretically justified \cite{Degris2012LinearOA,NIPS1999_6449f44a}, it is not a critical decision for centralized RL unless parameters are shared between the policy and the value function \cite{pmlr-v139-cobbe21a,DBLP:journals/corr/SchulmanWDRK17}.
This is mainly because the generated sequence of policies and value functions remains the same as long as the actor and critic are updated in an alternative manner. However, as we will point out in Section \ref{sec:Methods} and \ref{sec:TheoreticalAnalysis}, the order of updates can significantly affect the learning dynamics of FRL because of the aggregation step.

\subsection{FRL Problem}

In the common FRL setting, one central server federates the learning process of $N$ clients, which are equipped with RL agents to handle their local tasks. Server and clients collaboratively maximize the following objective function:
\begin{alignat}{1}
f(\theta) &= \sum_{n=1}^{N} q_{n} \eta_{n}(\theta) = \sum_{n=1}^{N} q_{n} \mathbb{E}_{s_{0} \thicksim \mu} \left[ V^{\pi^{\theta}}_{n}(s_{0}) \right] = \mathbb{E}_{n \thicksim \mathcal{C}, s_{0} \thicksim \mu} \left[ V^{\pi^{\theta}}_{n}(s_{0}) \right],
\label{eq:flobj}
\end{alignat}
where $\mathbb{E}_{n \thicksim \mathcal{C}} \left[ \dots \right]$ denotes expectation over clients sampled according to weights $q_{n}$, $\eta_{n}(\theta) = \eta_{n}(\pi^{\theta})$ is the local objective function of the $n$-th client with the policy parameters $\theta$, and $V_{n}^{\pi^{\theta}}$ is the state-value function of the policy $\pi^{\theta}$ in the n-th client. $q_{n}$ is a weighting factor often set to $q_{n}=l_{n}/\sum_{i=1}^N l_{i}$, with $l_n$ denoting the number of data points collected at the $n$-th client.


\section{Federated Neural PPO with Reversed Actor-Critic Learners\label{sec:Methods}}

In this section, we explore the implementation of federated neural PPO and propose a new approach to tackling the data heterogeneity issue. In section \ref{subsec:parameterization}, we introduce the neural network parameterization for theoretical analysis. In sections \ref{subsec:policy_improvement} and \ref{subsec:policy_evaluation}, we briefly review policy improvement and policy evaluation, respectively, which are the fundamental components of Actor-Critic algorithms. In section \ref{subsec:fedrac}, we propose FedRAC based on above components.
 
\subsection{Value Function and Policy Parameterization \label{subsec:parameterization}}

We utilize the following two-layer ReLU-activated neural networks to parameterize the state value function and policy, respectively, as
\begin{alignat}{1}
u_{w}(s) = \frac{1}{\sqrt{m}} \sum_{i}^{m} b_{i} \cdot \sigma(w_{i}^{T} (s) ), \quad f_{\theta}(s,a) = \frac{1}{\sqrt{m}} \sum_{i}^{m} b_{i} \cdot \sigma(\theta_{i}^{T} (s,a) ), \forall a \in \mathbb{R}^{d_{a}}, s \in \mathbb{R}^{d_{s}}, \label{eq:nn}
\end{alignat}
where $m$ is the width of the network, $\theta = (\theta_{i}, \dots \theta_{m})^{T} \in \mathcal{R}^{m (d_{s}+d_{a})}$, $w = (w_{i}, \dots w_{m})^{T} \in \mathcal{R}^{m d_{s}}$ denotes the input weights, and $b_{i} \in \{-1, 1\}, \forall i \in [m]$ is the output weight. Let $\vartheta$ be the placeholder for $\theta$ and $w$, and define $d_{\theta} = d_{s} + d_{a}, d_{w} = d_{s}$. We randomly initialize the parameters as follows:
\begin{alignat}{1}
\mathbb{E}_{\text{init}} \left[ \vartheta^{0}_{i,j} \right] = 0, \mathbb{E}_{\text{init}} \left[ \left( \vartheta^{0}_{i,j} \right)^{2} \right] = \frac{1}{d \cdot m}, b_{i} \thicksim \text{Unif}(\{-1,1\}), 0 < \left\Vert \vartheta^{0}_{i} \right\Vert_{2} \le \hat{R}_{\vartheta}, \label{eq:initialization}
\end{alignat}
$\forall i \in [m], j \in [d_{\vartheta}], 0 < \hat{R}_{\vartheta} < \infty$, where $\mathbb{E}_{\text{init}} \left[ \dots \right]$ denotes the expectation over random initialization. We adopt the technique from \cite{NEURIPS2019_98baeb82,wang2019neural,NEURIPS2019_227e072d} to simplify the analysis and ensure convergence. In particular, we assume that $\left\Vert (s,a) \right\Vert_{2} \le 1, \forall s \in \mathcal{S}, a \in \mathcal{A}$, keep the second layer $b_{i}, \forall i \in [m]$ fixed and only update the first layer $\vartheta$. After applying gradient descent, we project $\vartheta$ to the closest point in a $l_{2}$-ball centered at the initialization point, i.e., $\vartheta \in \mathcal{B}_{\mathcal{R_{\vartheta}}} = \left\{ \vartheta: \left\Vert \vartheta - \vartheta^{0} \right\Vert_{2} \le \mathcal{R}_{\vartheta} \right\}$.

With the above neural network parameterization, we introduce two commonly used stochastic policies: 1) Softmax policy for MDP with discrete action space \cite{pmlr-v70-haarnoja17a}, i.e., $\pi^{\theta^{t}}(a \vert s) = \frac{\exp(\tau^{-1}_{t} f_{\theta^{t}}(s,a))}{\sum_{a^{\prime} \in \mathcal{A}} \exp(\tau^{-1}_{t} f_{\theta^{t}}(s,a^{\prime}))}$,
where $\tau_{t} > 0$ is the temperature parameter which is monotonically decreasing with respect to the number of rounds $t$ and not trainable; 2) Gaussian policy for MDP with continuous action space \cite{DBLP:journals/corr/SchulmanWDRK17}, i.e., $\pi^{\theta^{t}}(a \vert s) \propto \frac{1}{\sqrt{2 \pi} \sigma_{t}} \exp(\frac{(a - f_{\theta^{t}}(s,a))^{2}}{2\sigma^{2}_{t}})$,
where $\sigma_{t}$ is the standard deviation, which is monotonically decreasing with respect to the number of rounds $t$ and not trainable. With a slight abuse of notation, we use $f_{\theta^{t}}(s,a)$ to denote $f_{\theta^{t}}(s)$ for Gaussian policy. Without loss of generality, we restrict our discussion to MDPs with 1-dimensional action spaces when we analyze continuous action spaces.

\subsection{Policy Improvement \label{subsec:policy_improvement}}

In this section, we present the policy improvement problem of the $n$-th client. Given an estimation of the action-value function $Q^{w}$, we consider the following objective as in neural PPO \cite{DBLP:journals/corr/SchulmanLMJA15,NEURIPS2019_227e072d}: 
\begin{alignat}{1}
L_{n}(\theta) &= \mathbb{E}_{s \thicksim \rho_{\pi^{\theta^{t}},n}} \left[ \left< Q^{w}(s,\cdot), \pi^{\theta}(\cdot \vert s) \right> - \beta_{t} \cdot D_{KL}(\pi^{\theta}(\cdot \vert s) \Vert \pi^{\theta^{t}}(\cdot \vert s)) \right],
\label{eq:policy_improvement_obj}
\end{alignat}
where $D_{KL}(\pi(\cdot \vert s) \Vert \pi^{\prime}(\cdot \vert s)) = \sum_{a} \pi(a|s) \log \frac{\pi(a|s)}{\pi^{\prime}(a|s)}$ denotes the KL divergence between two discrete probability distributions $\pi(\cdot\vert s)$ and $\pi^{\prime}(\cdot\vert s)$. For the continuous case, we replace the summation by integration. Note that we do not require the parameterized action-value function $Q^{w}$ to be the estimator of the action-value function $Q^{\pi^{\theta}}$ of any policy $\pi^{\theta}$. Moreover, (\ref{eq:policy_improvement_obj}) uses the expectation of the (normalized) discounted visitation frequency such that it is rescaled by $1 - \gamma$ compared with the original version of TRPO. In the original theorem of TRPO, the KL-regularized term stems from the total variation divergence, so the order of $\pi^{\theta}$ and $\pi^{\theta^{t}}$ in the KL divergence is interchangeable. Neural PPO \cite{NEURIPS2019_227e072d} has proven a weakened version of the following lemma. See Appendix \ref{proof:lemma:minimizer} for the proof.
\begin{lemma}
(Minimizer of (\ref{eq:policy_improvement_obj})).
\label{lemma:minimizer}
The following problem of minimizing the MSE
\begin{alignat}{1}
\theta^{t+1}_{n} &\leftarrow \arg\min_{\theta \in \mathcal{B}_{\mathcal{R_{f}}}} \mathbb{E}_{s \thicksim \rho_{\pi^{\theta^{t}},n}, a \thicksim \pi^{\theta^{t}}} \left[ \left( \log \pi^{\theta} - \left( \beta_{t} Q^{w} + \log \pi^{\theta^{t}} \right) \right)^{2} \right] \label{eq:policy_improvement}
\end{alignat}
is a minimizer of (\ref{eq:policy_improvement_obj}) with either Softmax or Gaussian policy.
\end{lemma}

\subsection{Policy Evaluation \label{subsec:policy_evaluation}}

There are many iterative methods for estimating the action-value function of a given policy $\pi^{\theta}$, e.g., neural Temporal Difference \cite{sutton1988learning}, Deep Q Network \cite{mnih2015human}, and Deep Double Q-network \cite{van2016deep}. We omit the implementation in this work as our study is not limited to any policy evaluation method. Instead, we explicitly state the local learning target of the $n$-th client given the $t$-th global policy $\pi^{\theta^{t}}$ as,
\begin{alignat}{1}
w^{t}_{n} &\leftarrow \arg\min_{w \in \mathcal{B}_{\mathcal{R_{Q}}}} \mathbb{E}_{s \thicksim \rho_{\pi^{\theta^{t}},n}, a \thicksim\pi^{\theta^{t}}} \left[ \left( Q^{w}(s,a) - Q^{\pi^{\theta^{t}}}_{n}(s,a) \right)^{2} \right]. \label{eq:policy_evaluation}
\end{alignat}

\subsection{FedRAC \label{subsec:fedrac}}

In centralized Actor-Critic, the policy evaluation step is supposed to obtain a sufficiently accurate estimation $Q^{w}$ of the exact action-value function $Q^{\pi}$ before the policy improvement step. It is straightforward to migrate this scheme to the FRL setting where local clients perform Actor-Critic with their local environment dynamics, and we call this scheme the baseline. However, we notice that the objective function of FRL in (\ref{eq:flobj}) suggests to optimize a joint value-function in the following form $\sum_{n=1}^{N} q_{n} \mathbb{E}_{s_{0} \thicksim \mu} \left[ V^{\pi^{\theta}}_{n}(s_{0}) \right]$, which indicates that the local estimation of gradient direction may be different from the globally desired one. To reduce the impact of data heterogeneity, the local update should be guided by the global estimation of the gradient instead of the local one. 

In light of this idea, we propose Federated Reinforcement Learning with Reversed Actor-Critic (FedRAC) algorithm as illustrated in Algorithm \ref{alg:FedRAC}, which switches the update order (lines 6-7) as compared to the conventional update rules of Actor-Critic (lines 5-6). In contrast to (\ref{eq:policy_evaluation}), the ideal policy evaluation problem for FedRAC is the following one, as will be shown in the next section:
\begin{alignat}{1}
w^{t}_{n} &\leftarrow \arg\min_{w \in \mathcal{B}_{\mathcal{R_{Q}}}} \mathbb{E}_{s \thicksim \rho_{\pi^{\theta^{t}},n}, a \thicksim\pi^{\theta^{t}}} \left[ \left( Q^{w^{t}} - \sum_{i} \frac{q_{i} \rho_{\pi^{\ast},i}(s)}{Z_{\pi^{\ast}}(s)} Q^{\pi^{\theta^{t}}}_{i} \right)^{2} \right], \label{eq:policy_evaluation_fedrac}
\end{alignat}
where $Z_{\pi^{\ast}}(s) = \sum_{n} q_{n} \rho_{\pi^{\ast},n}(s)$ is a state-wise weight that integrates information from all clients and satisfies $\sum_{s} Z_{\pi^{\ast}}(s) = 1$. It can be observed that FedRAC's critics aim to approximate weighted averages of all local critics, which vary in the sampling distribution $\rho_{\pi^{\theta^{t}},n}$. Intuitively, (\ref{eq:policy_evaluation_fedrac}) forces clients to move toward a similar gradient direction to address the data heterogeneity issue. Note that information about other clients is inaccessible to one client in practice, and we only consider (\ref{eq:policy_evaluation_fedrac}) for theoretical analysis purposes.
In Section \ref{subsec:error_propagation}, we formulate the learning target of FedRAC's policy evaluation using only local information as in (\ref{eq:policy_evaluation}), but derive the condition required for bounding (\ref{eq:policy_evaluation_fedrac}).

It is worth noting that \citet{10160088} investigated the problem of training time allocation between actor and critic. In particular, the authors proposed Critic-Actor to reverse the time scales of the actor and critic. This differs from changing the update order of actor and critic as considered in FedRAC.


\begin{algorithm}[H]
\caption{\label{alg:FedRAC}FedRAC (Blue) and the Baseline (Red).}
\begin{algorithmic}[1]
	\REQUIRE{$T,K,B,E,\theta^0,w^{0}$}
	\FOR{round t = 0,1,...T}
        \STATE{Synchronize the global parameters $\theta^{t}$ and $w^{t}$ to a random subset $\mathcal{K}$ of $K$ clients.}
            
        \FOR{client $k \in \mathcal{K}$}
            \STATE{Run policy $\pi^{\theta^{t}}$ to collect B timesteps.}
            \color{red}
            \STATE{Baseline: Optimize $w^{t}$ for $E$ epochs to obtain $w_{k}^{t+1}$ in (\ref{eq:policy_evaluation}).}
            \color{black}
            \STATE{Common: Optimize $\theta^{t}$ for $E$ epochs to obtain $\theta_{k}^{t+1}$ in (\ref{eq:policy_improvement}).}
            \color{blue}
            \STATE{FedRAC: Optimize $w^{t}$ for $E$ epochs to obtain $w_{k}^{t+1}$ in (\ref{eq:policy_evaluation}).}
            \color{black}
            \STATE{Upload $\theta_{k}^{t+1}$, $w_{k}^{t+1}$, and $l_k$ to the central server.}
        \ENDFOR
        \STATE{The central server perform aggregation: $\theta^{t+1}=\sum_{n=1}^{N}\frac{l_n}{L}\theta_{n}^{t+1}$ and $w^{t+1}=\sum_{n=1}^{N}\frac{l_n}{L}w_{n}^{t+1}$.}
	\ENDFOR
\end{algorithmic}
\end{algorithm}

\section{Theoretical Analysis\label{sec:TheoreticalAnalysis}}

In this section, we demonstrate the advantage of the proposed update order and establish the global convergence of FedRAC and the baseline. In our analysis framework, FRL has two sources of error: 1) the approximation error of policy evaluation and improvement, and 2) the error induced by policy aggregation.
We bound 1) and 2) in Section \ref{subsec:error_propagation} and \ref{subsec:aggregation_error}, respectively,
After bounding these errors, we obtain the global convergence rate of FedRAC in section \ref{subsec:global_convergence}. 
For ease of illustration, we collectively define a few terms and variables in Table \ref{tab:term-table}.

\begin{table}[!b]
  \caption{Useful Terms/Variables}
  \label{tab:term-table}
  \centering
  \begin{tabular}{rcl}
    \toprule
 $D_{KL}^{\ast,0}$ & $=$ & $\mathbb{E}_{\text{init}, n \thicksim \mathcal{C}, s \thicksim \rho_{\pi^{\ast},n}} \left[ D_{KL}\left( \pi^{\theta^{\ast}}(\cdot \vert s) \Vert \pi^{\theta^{0}}(\cdot \vert s) \right) \right]$ \\
 $\psi_{n}^{t}$ & $=$ & $\mathbb{E}_{s \thicksim \rho_{\pi^{\theta^{t}}, n}, a \thicksim \pi^{\theta^{t}}} \left[ \left\vert \frac{\mathrm{d} \left( \rho_{\pi^{\ast}, n} \pi^{\ast} \right)}{\mathrm{d} \left( \rho_{\pi^{\theta^{t}}, n} \pi^{\theta^{t}} \right)} - \frac{\mathrm{d} \rho_{\pi^{\ast}, n}}{\mathrm{d} \rho_{\pi^{\theta^{t}}, n}} \right\vert^{2} \right]^{1/2}$ \\
 $\Xi_{n}^{t}$ & $=$ & $\mathbb{E}_{i \thicksim \mathcal{C}, s \thicksim \rho_{\pi^{\theta^{t}},i}} \left[ \left\vert \frac{\rho_{\pi^{\ast},n}(s)}{\rho_{\pi^{\ast},i}(s)} \right\vert^{2} \right] \cdot \mathbb{E}_{i \thicksim \mathcal{C}, s \thicksim \rho_{\pi^{\theta^{t}},i}} \left[ \left\vert \frac{\rho_{\pi^{\ast},i}(s)}{Z_{\pi^{\ast}}(s)} \right\vert^{8} \right]^{1/2}$ \\ 
 $\mathcal{E}^{t}$ & $=$ & $\mathbb{E}_{i \thicksim \mathcal{C},s \thicksim \rho_{\pi^{\theta^{t}},i},a \thicksim \pi^{\theta^{t}}} \left[ \left\vert Q^{\pi^{\theta^{t}}}_{i}(s,a) \right\vert^{4} \left\vert 1 -\frac{\rho_{\pi^{\ast},i}(s)}{Z_{\pi^{\ast}}(s)} \right\vert^{4} \right] \cdot \mathbb{E}_{i \thicksim \mathcal{C},s \thicksim \rho_{\pi^{\theta^{t}},i}} \left[ \left\vert \frac{\rho_{\pi^{\ast},i}(s)}{Z_{\pi^{\ast}}(s)} \right\vert^{8} \right]^{-1/2}$ \\ 
 $\phi^{t+1}_{n}$ & $=$ & $\mathbb{E}_{\text{init}, s \thicksim \rho_{\pi^{\theta^{t}},n}, a \thicksim \pi^{\theta^{t}}} \left[ \left( \log \frac{\sigma_{t}}{\sigma_{t+1}} + \frac{1}{2}\left( \frac{1}{\sigma_{t+1}^{2}} - \frac{1}{\sigma_{t}^{2}} \right)\left( f_{\theta^{0}}^{2}(s,a) - a^{2} \right) \right)^{2} \right]$ \\
 $\phi^{t+1}_{n,\max}$ & $=$ & $\mathbb{E}_{\text{init}, s \thicksim \rho_{\pi^{\theta^{t}},n}} \left[ \max_{a \in \mathcal{A}}\left( \log \frac{\sigma_{t}}{\sigma_{t+1}} + \frac{1}{2}\left( \frac{1}{\sigma_{t+1}^{2}} - \frac{1}{\sigma_{t}^{2}} \right)\left( f_{\theta^{0}}^{2}(s,a) - a^{2} \right) \right)^{2} \right]$ \\
    \bottomrule
  \end{tabular}
\end{table}

\subsection{Error Propagation \label{subsec:error_propagation}}

The following assumption is related to the concentrability coefficient in previous works \cite{NIPS2010_65cc2c82,NIPS2007_da0d1111,pmlr-v70-tosatto17a,JMLR:v9:munos08a,pmlr-v120-yang20a,wang2020statisticallimitsofflinerl,NEURIPS2019_227e072d}. It states that the state visitation frequency of every local MDP under the optimal and parameterized policy should be similar enough for a well-defined FRL objective (\ref{eq:flobj}).
\begin{assumption}
\label{assumption:concentrability_coefficient}
(Bounded Concentrability Coefficient). We assume the Radon-Nikodym derivative of $\rho_{\pi^{\ast}, n} \pi^{\ast}$ w.r.t. $\rho_{\pi^{\theta^{t}}, n} \pi^{\theta^{t}}$ and that of $\rho_{\pi^{\ast}, n}$ w.r.t. $\rho_{\pi^{\theta^{t}}, n}$ are bounded for all parameterized policies $\pi^{\theta}$ in the function class $\mathcal{F}_{R_{\theta},m} = \left\{ \pi^{\theta}:\theta \in \mathcal{B}_{R_{\theta}} \right\}$. In consequence, the squared difference between these two quantities is bounded, i.e., $\psi_{n}^{t} < \infty, \forall n \in [N]$.
\end{assumption}
The following lemma shows that the concentrability coefficient is useful for understanding how the error introduced by policy evaluation and improvement propagates into the policy space.
\begin{lemma}
\label{lemma:error_propagation}
(Error Propagation). Suppose that Assumption \ref{assumption:concentrability_coefficient} holds. Define $\upsilon_{t}(s,a) = \frac{\sigma_{t}^{2}}{a - f_{\theta^{0}}(s,a)}$, 
$\tilde{\epsilon}_{n}^{t+1} = \sqrt{\left( \frac{\sup \left\vert \upsilon_{t+1}^{-1}(s,a) \right\vert \epsilon_{n}^{t+1}}{3} \right)^{2} - \phi^{t+1}_{n} - \frac{R_{\theta}^{4}}{4 \sigma_{t+1}^{4}}}$, and $
\hat{\epsilon}_{n}^{t+1} = \left( (\dot{\epsilon}_{n}^{t+1})^{2} - 8 R_{w}^{2} \right) \left( \Xi_{n}^{t} \right)^{-\frac{1}{2}} - \left( \mathcal{E}^{t} \right)^{\frac{1}{2}}$. For a given action-value function $Q^{w}$, suppose that the policy improvement error satisfies $\mathbb{E}_{s \thicksim \rho_{\pi^{\theta^{t}},n}, a \thicksim \pi^{\theta^{t}}} \left[ \left( \left( \beta_{t}^{-1} Q^{w}(s,a) + c_{t}^{-1} f_{\theta^{t}}(s,a) \right) - c_{t+1}^{-1} f_{\theta^{t+1}_{n}}(s,a) \right)^{2} \right]^{1/2} \le c_{t+1}^{-1} e_{n}^{t+1}$, where $c_{t}(s,a) = \tau_{t},e_{n}^{t+1} = \epsilon_{n}^{t+1}$ for Softmax policy, and $c_{t}(s,a) = \upsilon_{t}(s,a),e_{n}^{t+1} = \tilde{\epsilon}^{t+1}_{n}$ for Gaussian policy. Suppose that the policy evaluation error satisfies
\begin{align}
& \mathbb{E}_{s \thicksim \rho_{\pi^{\theta^{t}},n}, a \thicksim \pi^{\theta^{t}}} \left[ \left( Q^{w^{t}_{n}}(s,a) - Q^{\pi^{\theta^{t}}}_{n}(s,a) \right)^{2} \right]^{1/2} \le \dot{\epsilon}_{n}^{t+1}, \label{eq:policy_evaluation_error_baseline}
\end{align}
and
\begin{alignat}{1}
& \mathbb{E}_{s \thicksim \rho_{\pi^{\theta^{t}},i}, a \thicksim \pi^{\theta^{t}}} \left[ \left( Q^{w^{t}_{i}}(s,a) - Q^{\pi^{\theta^{t}}}_{i}(s,a) \right)^{8} \right]^{1/4} \le \hat{\epsilon}_{n}^{t+1}, \forall i \in [N], \label{eq:policy_evaluation_error_fedrac}
\end{alignat}
for the baseline and FedRAC, respectively. Then, we have
\begin{alignat}{1}
& \left\vert \mathbb{E}_{s \thicksim \rho_{\pi^{\ast},n}} \left[ \left< \log \left( \pi^{\theta^{t+1}_{n}}(\cdot \vert s) / \pi^{t+1}_{n}(\cdot \vert s) \right), \pi^{\ast}(\cdot \vert s) - \pi^{\theta^{t}}(\cdot \vert s) \right> \right] \right\vert \le \varepsilon_{n}^{t}, \label{eq:lemma:error_propagation}
\end{alignat}
where $\varepsilon_{n}^{t} = \sup \left\vert c_{t+1}^{-1}(s,a) \right\vert \epsilon^{t+1}_{n} \psi^{t}_{n} + \beta_{t}^{-1} \dot{\epsilon}^{t+1}_{n} \psi^{t}_{n}$.
\end{lemma}
See Appendix \ref{proof:lemma:error_propagation} for the proof. \begin{remark}
Lemma \ref{lemma:error_propagation} shows that the difference between the ideal and the learned policy of the $n$-th client is bounded by the approximation error in policy evaluation and improvement. A related result for centralized neural PPO with Softmax policy can be found in \cite{NEURIPS2019_227e072d}. For FedRAC, there are $N^{2}$ constraints for policy evaluation error, and they can be much stricter than those $\dot{\epsilon}_{n}^{t+1}$ for the baseline. To see this, we consider the IID case, where all clients have the same state visitation frequency $\rho_{\pi,n},\forall n \in [N]$ given the same behavior policy $\pi$. Under such circumstances, $\hat{\epsilon}_{n}^{t+1}$ becomes $\left( \dot{\epsilon}^{t+1}_{n} \right)^{2} - 8 R_{w}^{2}$ which is of the same magnitude as $\dot{\epsilon}_{n}^{t+1}$ and the number of constraints reduces to $N$. Then, as the heterogeneity level increases, the condition (\ref{eq:policy_evaluation_error_fedrac}) tends to be more stringent for the clients with a large value of $\Xi_{n}^{t}$. It is shown in the proof that (\ref{eq:policy_evaluation_fedrac}) is bounded when (\ref{eq:policy_evaluation_error_fedrac}) is met.
\end{remark}

Next, Lemma \ref{lemma:stepwise_energy_difference} characterizes the worst-case difference between successive policies in local clients.
\begin{lemma}
\label{lemma:stepwise_energy_difference}
(Worst-case, Stepwise Log Probability Difference). Let $p^{\prime}=0,p^{\dagger}=20 R_{\theta}^{2} + 24$ and $p^{\prime}=1,p^{\dagger}=12 R_{\theta}^{2}$ for Softmax and Gaussian policies, respectively. Define
\begin{alignat}{1}
\varphi^{t+1}_{n} &= \mathbb{E}_{\text{init}, s \thicksim \rho_{\pi^{\ast},n}} \left[ \max_{a \in \mathcal{A}} \left( \left( f_{\theta^{0}}^{2}(s,a) + R_{\theta}^{2} \right) ( c_{t+1}^{-1}(s,a) - c_{t}^{-1}(s,a) )^{2} \right) \right], \nonumber \\
M_{n}^{t+1} &= \max_{a \in \mathcal{A}} \left( p^{\dagger} c_{t+1}^{-2}(s,a) \right) + (8 - 2 p^{\prime}) \varphi_{n}^{t+1} + p^{\prime} \left( 3 \phi^{t+1}_{n,\max} / 2 + 3 R_{\theta}^{4} \sigma_{t+1}^{-4} / 4 \right). \nonumber
\end{alignat}
We have
\begin{alignat}{1}
& \mathbb{E}_{\text{init}, s \thicksim \rho_{\pi^{\ast},n}} \left[ \left\Vert \log \pi^{\theta^{t+1}_{n}}(\cdot \vert s) / \pi^{\theta^{t}}(\cdot \vert s) \right\Vert_{\infty}^{2} \right] \le 2 M_{n}^{t+1}.
\end{alignat}
\end{lemma}
See Appendix \ref{proof:lemma:stepwise_energy_difference} for the proof. The difference is bounded due to the neural network parameterization and projection, which can be reduced by using a sufficiently large temperature $\tau$ or a small radius $R_{\theta}$.

\subsection{Aggregation Error \label{subsec:aggregation_error}}

In addition to the above errors introduced by local training, the federated setting has its unique challenge, i.e., the parameter aggregation in the central server may distort the policy and value function. To quantify the dissimilarity between MDPs, we introduce the following quantity on the level of heterogeneity. Similar definitions are widely used in the FL/FRL literature \cite{MLSYS2020_38af8613,li2020on,10038492,xie2023client,pmlr-v151-jin22a}.
\begin{definition}
\label{def:impacthete}(Level of heterogeneity). We define the level of heterogeneity as the pairwise dissimilarity between action-value functions: $\kappa = \mathbb{E}_{n \thicksim \mathcal{C}, s \thicksim \rho_{\pi^{\ast},n}, a \thicksim \pi^{\ast}} \left[ \sum_{i} q_{i} \left\vert Q^{w^{t}_{n}}(s,a) - Q^{w^{t}_{i}}(s,a) \right\vert \right]$.
\end{definition}
Similar to Lemma \ref{lemma:error_propagation}, we are interested in the error in probability space as shown in the following lemma, which is particularly useful for the convergence analysis.
\begin{lemma}
\label{lemma:omega}
(Aggregation Error in Log Probability). Suppose that Assumption \ref{assumption:concentrability_coefficient} holds. Define $\Omega^{t} = \left\vert \mathbb{E}_{\text{init}, n \thicksim \mathcal{C}, s \thicksim \rho_{\pi^{\ast},n}, a \thicksim \pi^{\ast}} \left[ \log \left( \pi^{\theta^{t+1}}(a \vert s) / \pi^{\theta^{t+1}_{n}}(a \vert s) \right) \right] \right\vert$, we have
\begin{alignat}{1}
\Omega^{t} &\le 2 \mathbb{E}_{n \thicksim \mathcal{C}} \left[ \sup \left\vert c_{t+1}^{-1}(s,a) \right\vert e^{t+1}_{n} \xi^{t+1}_{n} \right] + p^{\prime} \frac{R_{\theta}^{2}}{\sigma_{t+1}^{2}} + (1 - p^{\prime} ) \frac{\sqrt{6} R_{\theta}}{\tau_{t+1}} + p \beta_{t}^{-1} \kappa  \nonumber \\
& \quad + \mathcal{O} \left( \sup \left\vert c_{t+1}^{-1}(s,a) \right\vert R_{\theta}^{6/5} m^{-1/10} \hat{R}_{\theta}^{2/5} \right),
\end{alignat}
where $c_{t}(s,a) = \tau_{t}, e_{n}^{t+1} = \epsilon_{n}^{t+1},p^{\prime} = 0$ for Softmax policies, $c_{t}(s,a) = \upsilon_{t}(s,a), e_{n}^{t+1} = \tilde{\epsilon}_{n}^{t+1},p^{\prime} = 1$ for Gaussian policies, $p = 1$ for the baseline, and $p = 0$ for FedRAC.
\end{lemma}
See Appendix \ref{proof:lemma:omega} for the proof.
\begin{remark}
Lemma \ref{lemma:omega} shows that the aggregation step of FedRAC possesses a lower error bound free from the heterogeneity $\kappa$ of the system. Intuitively, estimating gradients with the same action-value function $Q^{w^{t}}$ stabilizes the learning and encourages clients to move in the same direction. However, they may still have different gradient directions because of varying state visitations.
\end{remark}

\subsection{Global Convergence \label{subsec:global_convergence}}

Lastly, we establish the global convergence rate by combining all the previous results.
\begin{theorem}
\label{theorem:global_convergence_rate}
(Global Convergence Rate). Define $\varepsilon^{t} = \sum_{n=1}^{N} q_{n} \varepsilon_{n}^{t+1}$, $M^{t} = \sum_{n=1}^{N} q_{n} M_{n}^{t+1}$, and suppose that Assumption \ref{assumption:concentrability_coefficient} holds. For the policy sequence attained by Algorithm \ref{alg:FedRAC}, we have
\begin{alignat}{1}
\min_{0 \le t \le T} \mathbb{E}_{\text{init}} \left[ \eta(\pi^{\ast}) - \eta(\pi^{\theta^{t}}) \right] &\le \frac{D_{KL}^{\ast,0} + \sum_{t=1}^{T} \left( \varepsilon^{t} + M^{t} + \Omega^{t} \right)}{\sum_{t=1}^{T}\beta_{t}^{-1}}.
\end{alignat}
\end{theorem}
See Appendix \ref{proof:theorem:global_convergence_rate} for the proof. As the width $m \to \infty$, the initial policy tends to a uniform policy, so $D_{KL}^{\ast,0}$ tends to $\log \left\vert \mathcal{A} \right\vert - \mathbb{E}_{n \thicksim \mathcal{C}, s \thicksim \rho_{\pi^{\ast},n}} \left[ H(\pi^{\ast}(\cdot \vert s)) \right]$, where $H(\cdot)$ denotes the entropy of a random variable. FedRAC achieves a tighter convergence upper bound than the baseline under both Softmax and Gaussian policies, with its performance independent of the heterogeneity level $\kappa$.


To gain insight into Theorem \ref{theorem:global_convergence_rate}, we examine the specific case with Softmax policy. A detailed derivation can be found in Appendix \ref{proof:eq:best_convergence_rate}. By setting $\beta_{t} = \beta \sqrt{T}$ and $\tau_{t+1} = \frac{\beta T^{2}}{t + 1}$, both algorithms converge at the rate of
$\frac{4 \sqrt{ \left( 5 R_{\theta}^{2} + 6 \right) \log \left\vert \mathcal{A} \right\vert}}{T} + \mathcal{O} \left( \frac{R_{\theta}^{6/5} m^{-1/10} \hat{R}_{\theta}^{2/5}}{\sqrt{T}}  + \frac{R_{\theta}}{\sqrt{T}}  \right) + p \kappa + \frac{\sum^{T}_{t=1} \sum^{N}_{n=1} q_{n} \epsilon_{n}^{t+1} \left( \psi^{t}_{n} + 2 \xi^{t}_{n} \right)}{T^{\frac{3}{2}}} + \frac{\sum^{T}_{t=1} \sum^{N}_{n=1} q_{n} \dot{\epsilon}^{t} \psi^{t}_{n}}{T} + \mathcal{O} \left( \frac{1}{T^{3}} \right)$
with the optimal penalty parameter $\beta_{t} = \beta \sqrt{T} = 2 \sqrt{\frac{ 5 R_{\theta}^{2} + 6 }{\log \left\vert \mathcal{A} \right\vert}}$, which is problem-specific and independent of $T$. This convergence rate consists of six terms: 1) the first $\mathcal{O} \left( 1/T \right)$ term same as that of centralized policy gradient with exact policy gradient and Softmax parameterization \cite{pmlr-v119-mei20b}; 2) the second $\mathcal{O} \left( 1/\sqrt{T} \right)$ term due to the impact of nonlinearity of neural network parameterization on aggregation; 3) the baseline may not converge to the globally optimal policy $\pi^{\ast}$ but its distance to $\pi^{\ast}$ is bounded by $\kappa$; 4) the accumulated policy improvement error; 5) the accumulated policy evaluation error which is changing slower than 4) w.r.t. $T$; and 6) a negligible term. With exact action-value functions, ideal policy update, and a homogeneous network, the first term dominates the rate. Thus, both algorithms converge to the optimal policy $\pi^{\ast}$ at the rate of $\mathcal{O} \left( \sqrt{\log \left\vert \mathcal{A} \right\vert} /T \right)$ when $m$ is sufficiently large. This rate is comparable to a recently developed rate for federated natural Actor-Critic \cite{NEURIPS2024_dbdea785}. 

\section{Experiments\label{sec:Experiments}}

In this section, we compare the proposed FedRAC with the conventional update order. All curves are averaged over five independent trials with different initialization seeds. Confidence intervals are reported by capturing randomness in model initialization, client selection, and environment dynamics. All experimental results are reproducible and can be accessed from the GitHub repository: https://github.com/ShiehShieh/FedRAC.

\subsection{Three Classical and One Autonomous Driving Environments \label{subsec:experiment_environments}}


We evaluate FedRAC on three federated environments based on OpenAI Gym \cite{DBLP:journals/corr/BrockmanCPSSTZ16}. We construct three heterogeneous networks by introducing heterogeneity into their transition probabilities. As a simple case, we vary actions by applying different constant shifts to a collection of 60 Mountain Car environments, called Mountain Cars. 
Actions are shifted by $\omega_{n}$ on all states before inputting to the $n$-th environment, where the constant $\omega_{n}$ is uniformly sampled from $[-1.5, 1.5]$. For Hopper and Half Cheetah, we vary the pole length so that the same action may lead to different outcomes for different clients \cite{pmlr-v151-jin22a}. Hoppers consist of 60 equally weighted Hopper-v3 environments, whose leg sizes are uniformly sampled from $[0.01, 0.10]$. Half Cheetahs is created in a similar way to Hoppers.

To further show the effectiveness of FedRAC, we construct a heterogeneous autonomous vehicle training system, called HongKongOSMs, by utilizing the Flow simulator \cite{pmlr-v87-vinitsky18a}, Eclipse SUMO (Simulation of Urban MObility), and OSM (OpenStreetMap) dataset \cite{OpenStreetMap, 9119753}. See Appendix \ref{sec:additionexperimentdetails} for details of the OSM datasets of ten areas. Each SUMO-based environment is based on one OSM dataset containing one RL-controlled vehicle and 20 IDM-controlled (Intelligent Driver Model) vehicles. We deploy two kinds of RL-controlled vehicles, differing in maximum acceleration/deceleration, so we have 20 environments. See Appendix \ref{sec:additionexperimentdetails} for more experimental and implementation details.

\subsection{Effectiveness of the Reversed Update Order \label{subsec:1}}

We implemented three base algorithms, including FedAvg, FedProx, and SCAFFOLD \cite{pmlr-v54-mcmahan17a,MLSYS2020_38af8613,pmlr-v119-karimireddy20a} with the conventional (the baseline) and the proposed update order (FedRAC) for actor and critic. As shown in Fig. \ref{fig:performance}, FedRAC constantly outperforms the baseline in all environments described in Section \ref{subsec:experiment_environments}. In some cases (Fig. \ref{fig:hopper} and \ref{fig:osm}), FedRAC’s performance remains invariant to the choice of the base algorithm, which empirically validates its capability to mitigate model divergence.

\subsection{Comparison of Different Parameterizations \label{subsec:2}}

We conducted a study on neural network parameterization, which was motivated by an observation from Section \ref{subsec:global_convergence}. In particular, the number of layers and neurons in the neural network influences the accuracy of value function estimation and policy update. These factors, in turn, directly impact the learning performance, as demonstrated in Fig. \ref{fig:parameterization}. While a deep network may achieve low training loss, it may also have a high generalization error. This is problematic in heterogeneous systems, as the local model may overfit its local environment, leading to diverging model updates. In Fig. \ref{fig:hopper-4layers} and \ref{fig:hopper-2layers-4layers}, FedRAC benefits from low training loss, as it helps to satisfy the constraints outlined in Lemma \ref{lemma:error_propagation}. In Fig. \ref{fig:hopper-2layers} and \ref{fig:hopper-4layers-2layers}, where networks are too narrow, FedRAC shows no advantage as it fails to meet those constraints. This coincides with our discussion at the end of Section \ref{subsec:global_convergence}.

\subsection{Comparison of Different Levels of Heterogeneity \label{subsec:3}}

In this section, we validate the robustness of FedRAC to different levels of heterogeneity. By varying the magnitude of action noise in Mountain Cars, we make three variants, including Low-Mountain Cars, Medium-Mountain Cars, and High-Mountain Cars. They sample the action noises $\omega_{n}$ from intervals $[-1.0, 1.0]$, $[-1.5, 1.5]$, and $[-2.0, 2.0]$, respectively. Intuitively, the heterogeneity level of the system increases as the action noise becomes more and more different among clients. As illustrated in Fig. \ref{fig:heterogeneity_level}, FedRAC can consistently outperform the baseline on all variants, except for FedRAC + SCAFFOLD vs. SCAFFOLD in Fig. \ref{fig:mcc_high_level_hete}. Nevertheless, both FedRAC + SCAFFOLD and SCAFFOLD achieve near-optimal performance in Fig. \ref{fig:mcc_high_level_hete}. However, we did not observe an increase in the performance gap between FedRAC and the baseline as the heterogeneity level increases. We conjecture that the attainable reward changes as the heterogeneity level changes, so the magnitudes of the performance gap between Fig. \ref{fig:mcc_low_level_hete}, \ref{fig:mcc_medium_level_hete}, and \ref{fig:mcc_high_level_hete} are not comparable.

\begin{figure}[!t]
\centering
\subfloat[]{\includegraphics[width=0.33\columnwidth]{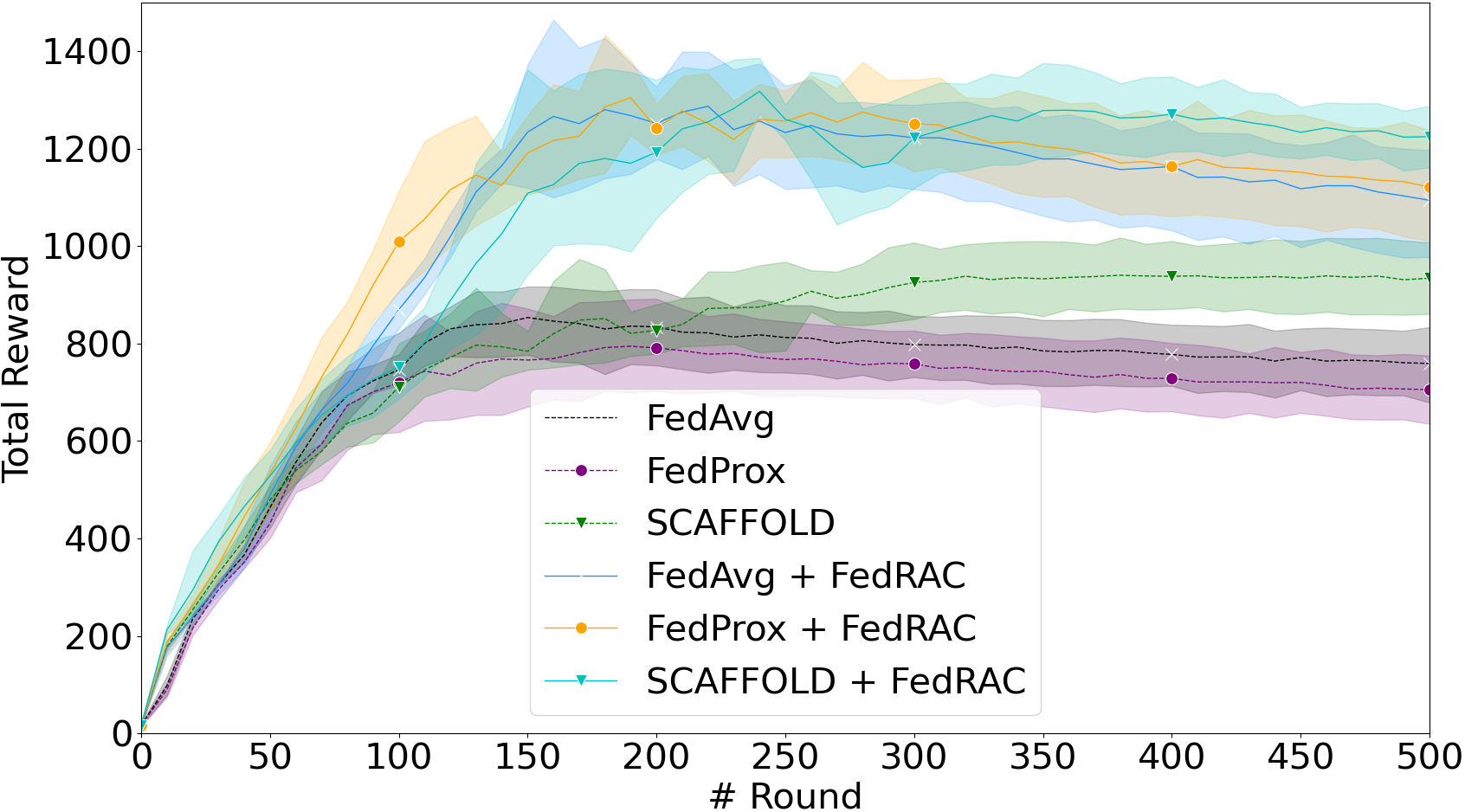}\label{fig:hopper}}
\hfil
\subfloat[]{\includegraphics[width=0.33\columnwidth]{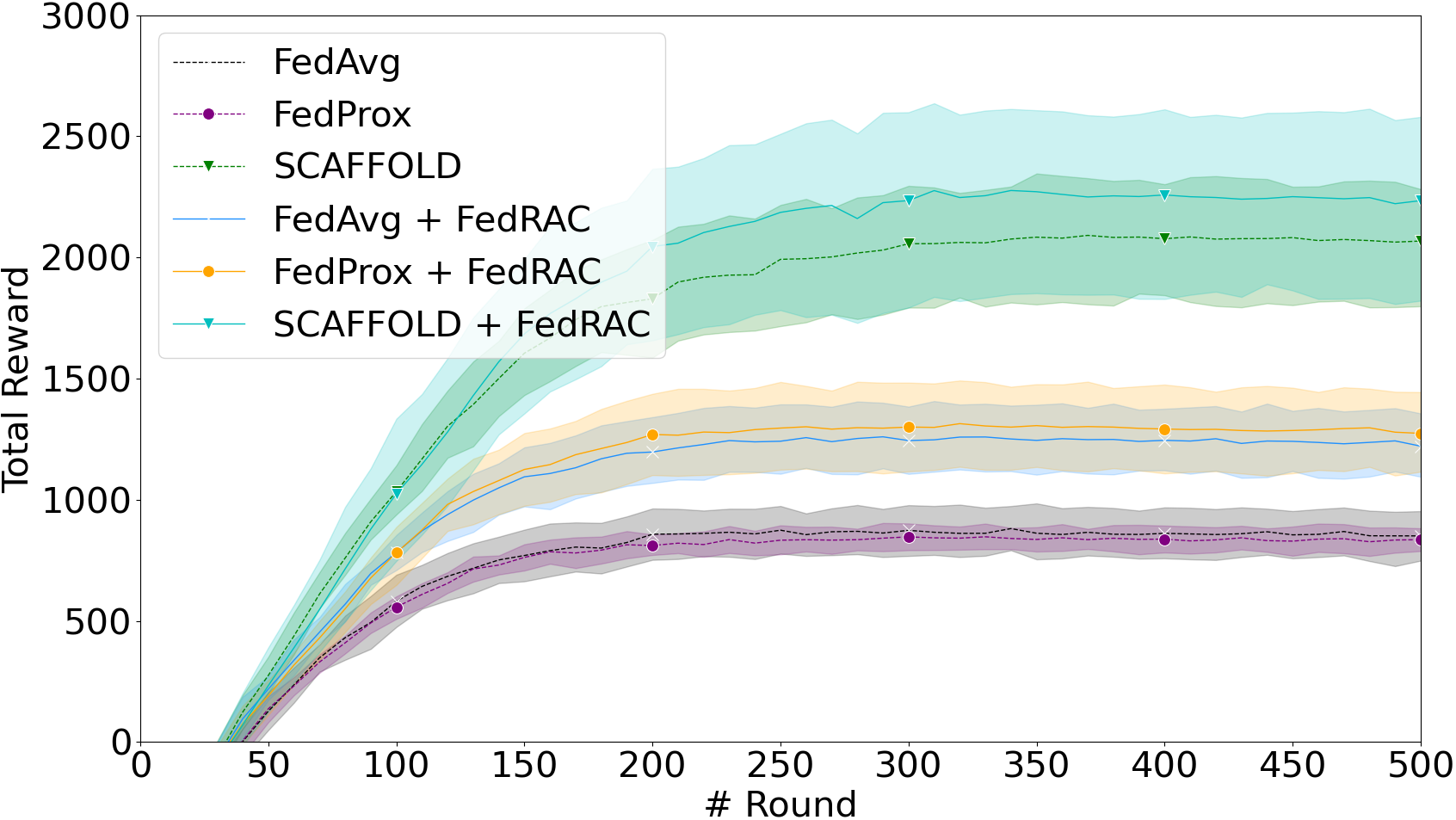}\label{fig:halfcheetah}}
\hfil
\subfloat[]{\includegraphics[width=0.33\columnwidth]{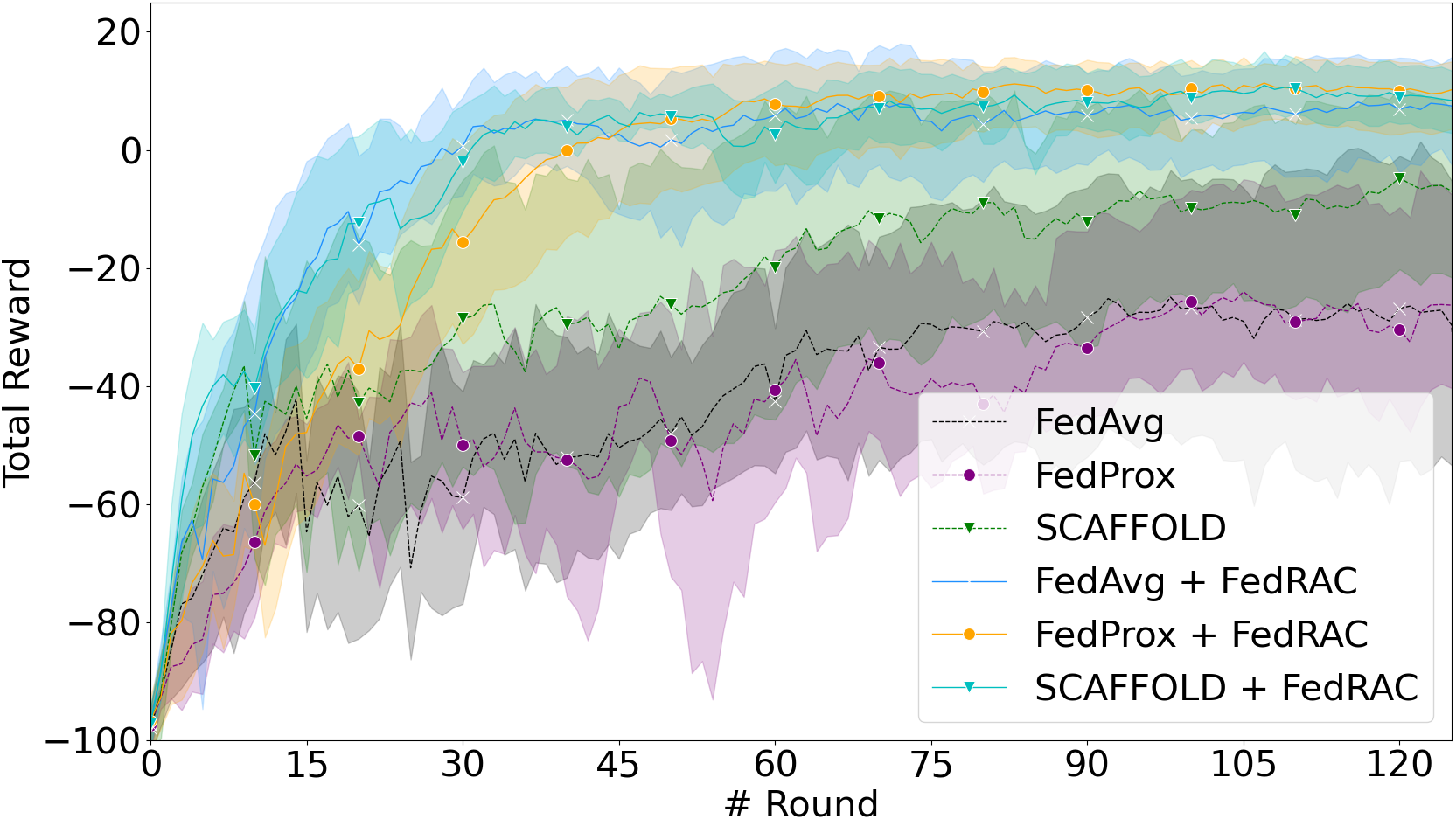}\label{fig:osm}}
\caption{Comparisons on different simulations. (a) Hoppers; (b) Half Cheetahs; (c) HongKongOSMs.}
\label{fig:performance}
\end{figure}

\begin{figure}[!t]
\centering
\subfloat[]{\includegraphics[width=0.24\columnwidth]{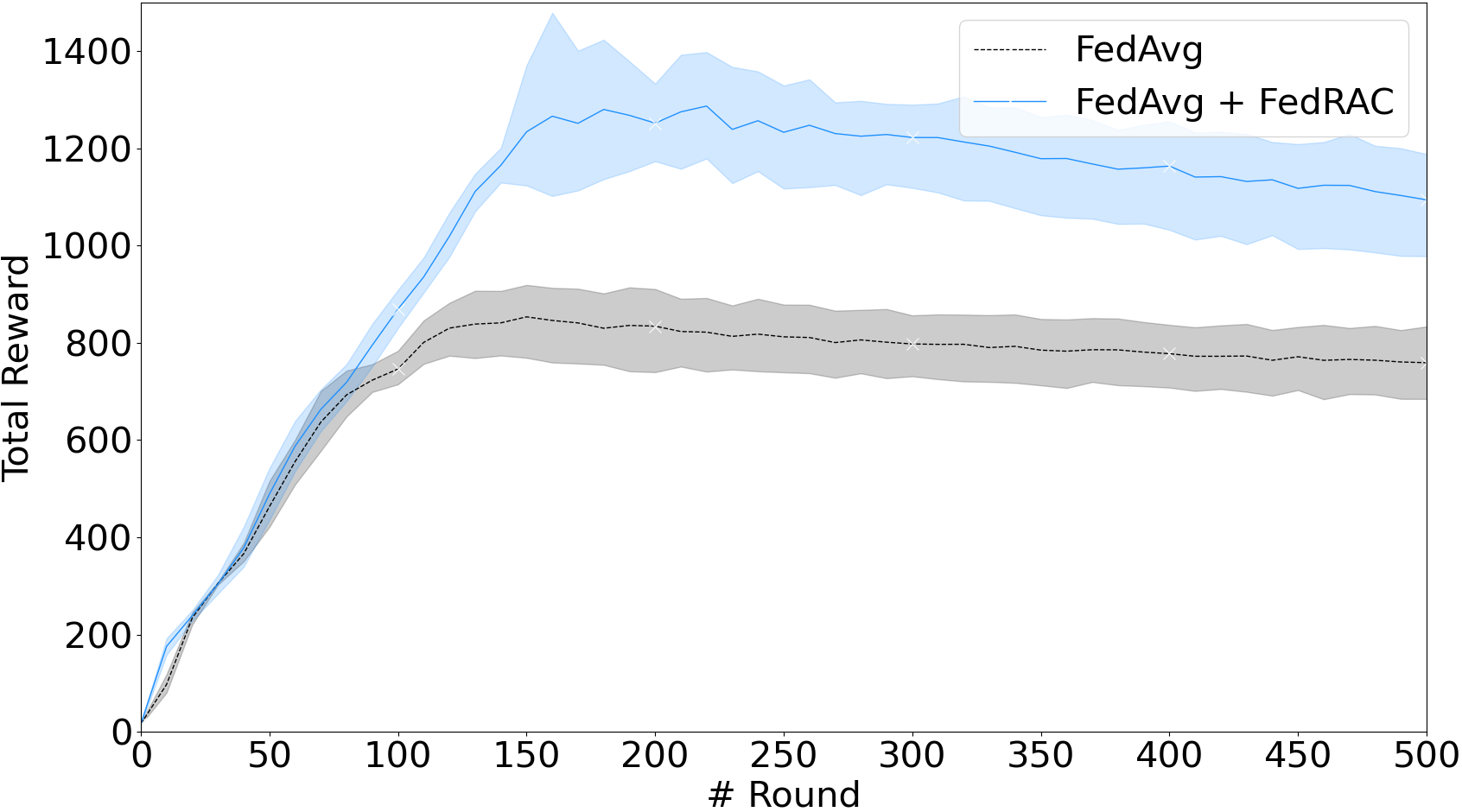}
\label{fig:hopper-4layers}}
\hfil
\subfloat[]{\includegraphics[width=0.24\columnwidth]{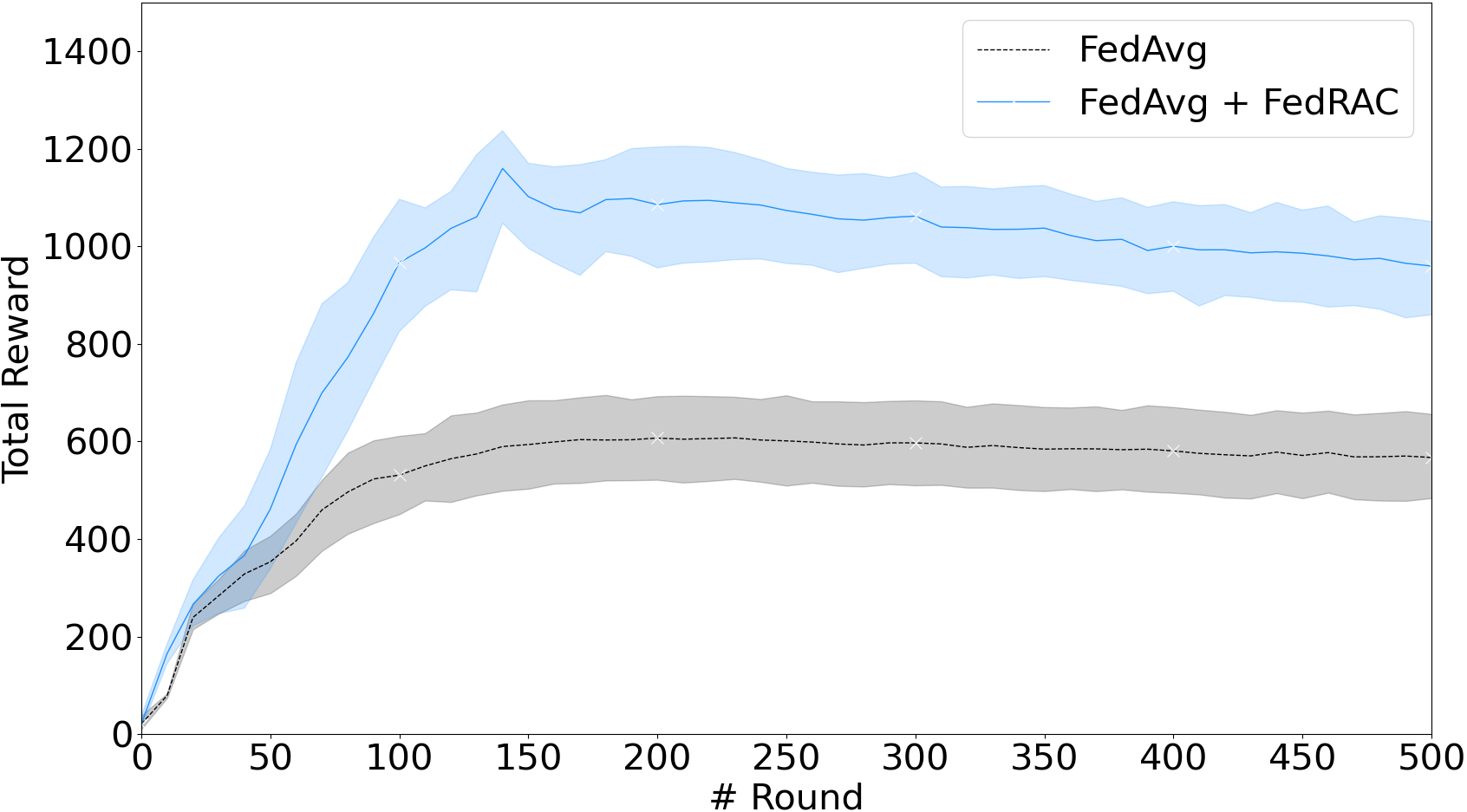}\label{fig:hopper-2layers-4layers}}
\hfil
\subfloat[]{\includegraphics[width=0.24\columnwidth]{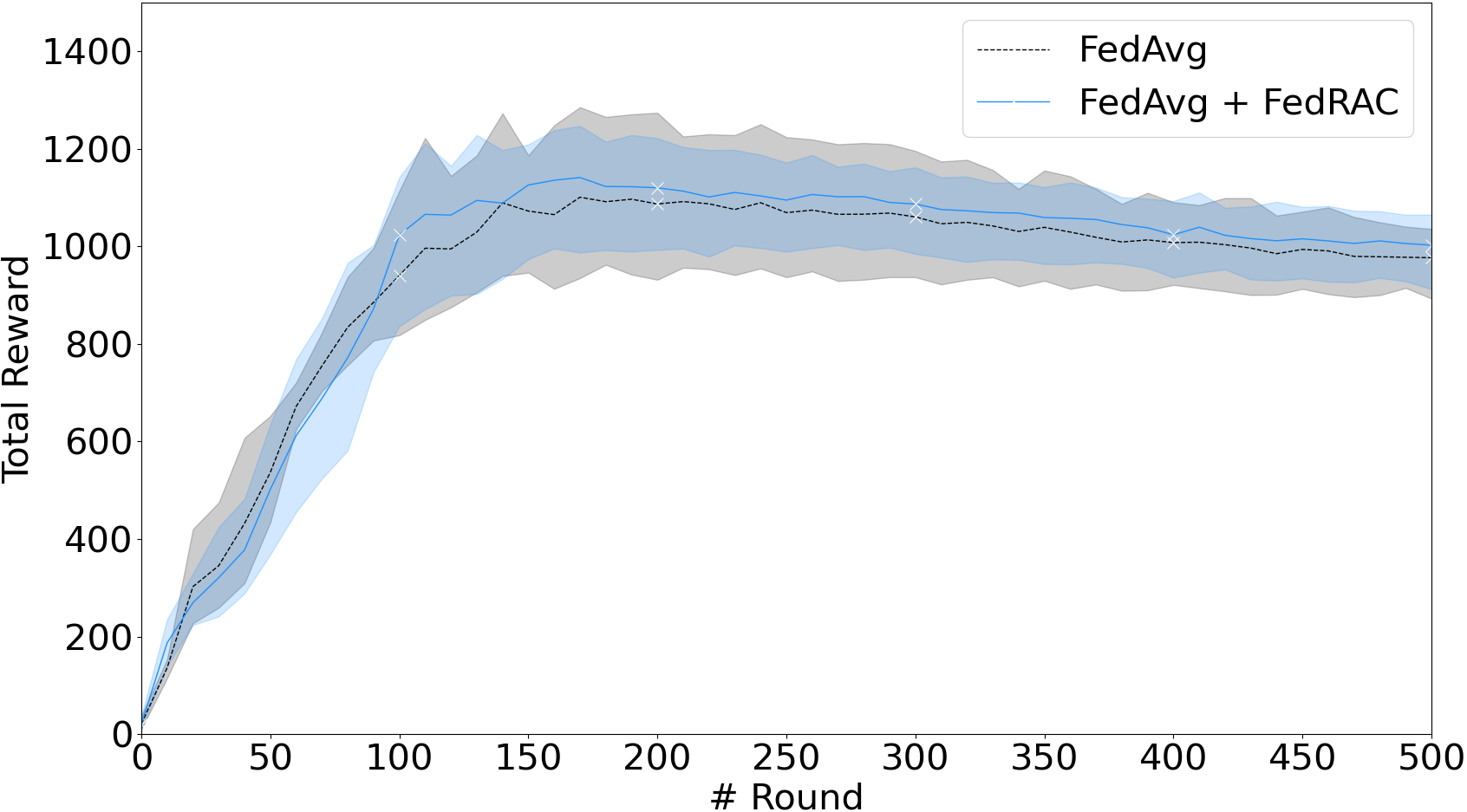}\label{fig:hopper-2layers}}
\hfil
\subfloat[]{\includegraphics[width=0.24\columnwidth]{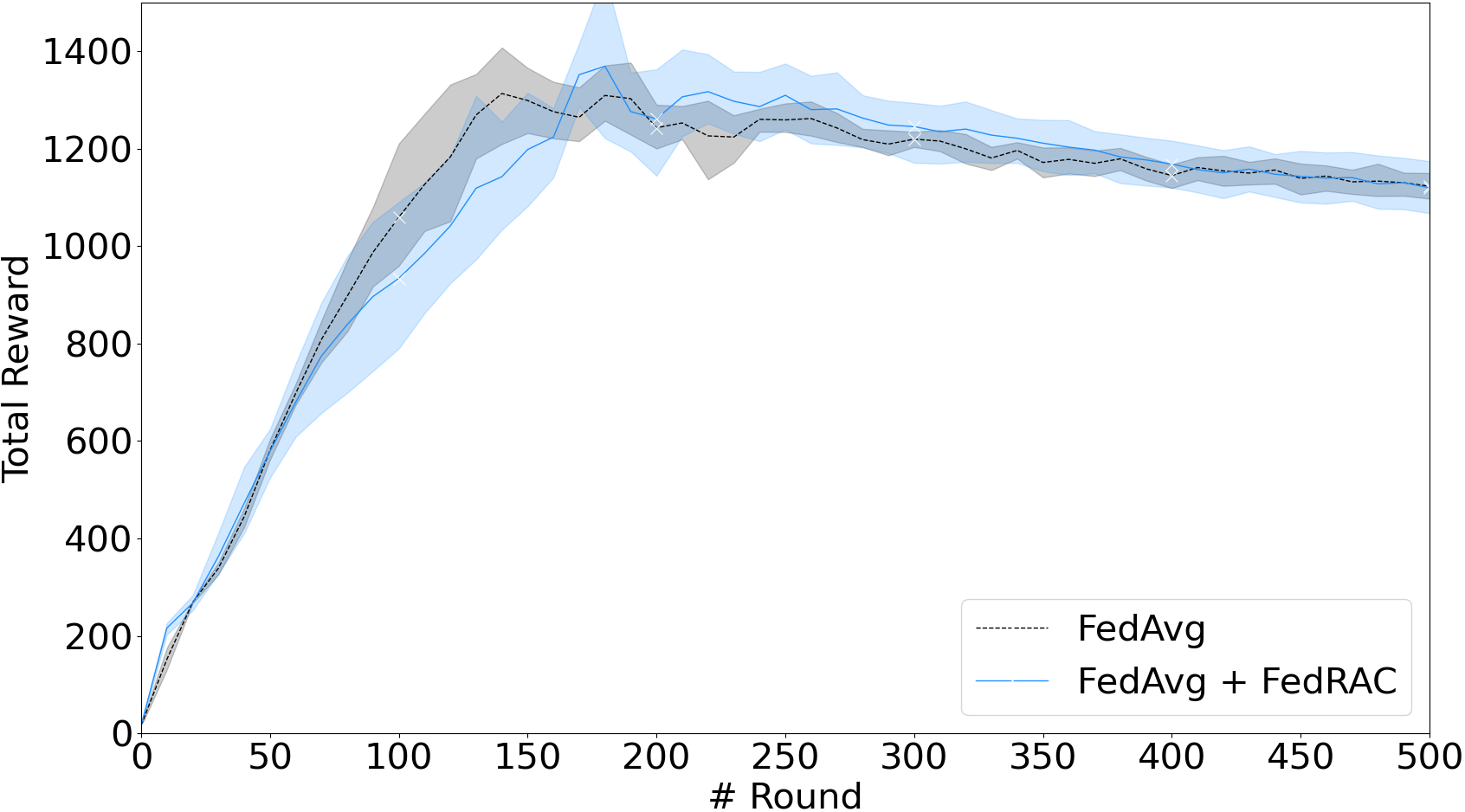}\label{fig:hopper-4layers-2layers}}
\caption{Learning curves on Hoppers with different neural network settings: (a) 4-layer MLPs for both; (b) 2-layer MLPs for actor and 4-layer MLPs for critic; (c) 2-layer MLPs for both; (d) 4-layer MLPs for actor and 2-layer MLPs for critic.}
\label{fig:parameterization}
\end{figure}

\begin{figure}[!t]
\centering
\subfloat[]{\includegraphics[width=0.24\columnwidth]{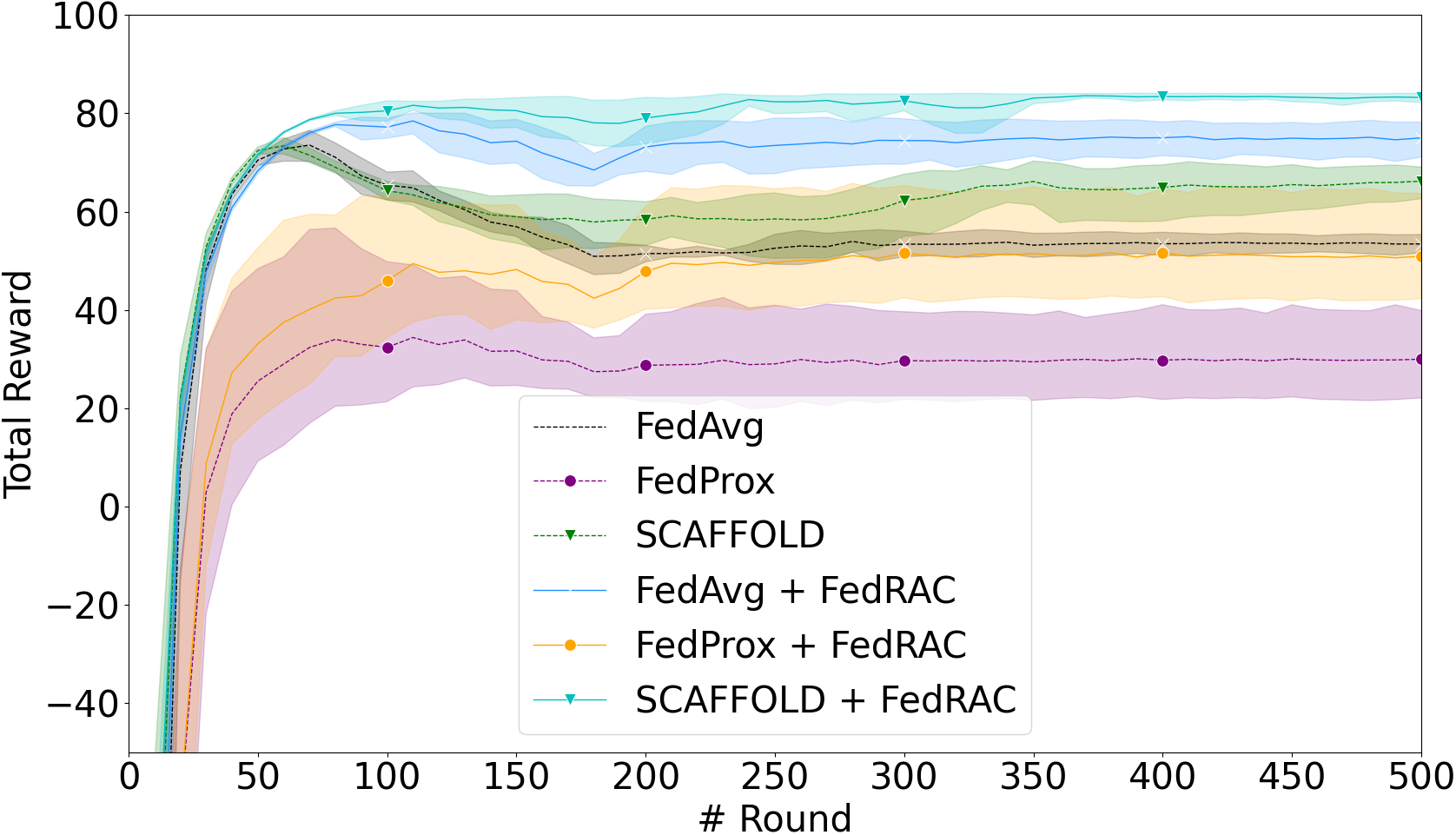}
\label{fig:mcc_low_level_hete}}
\hfil
\subfloat[]{\includegraphics[width=0.24\columnwidth]{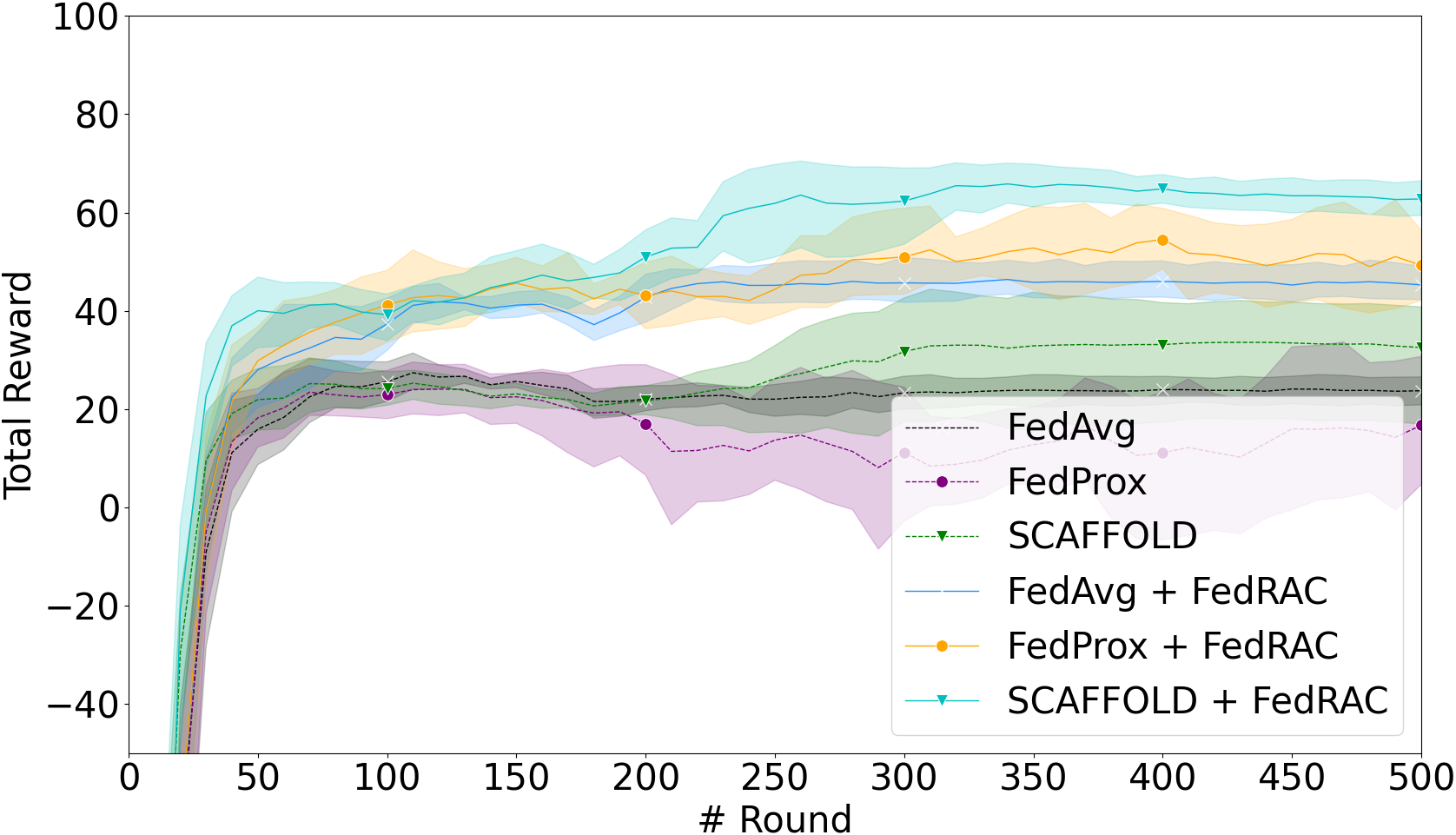}\label{fig:mcc_medium_level_hete}}
\hfil
\subfloat[]{\includegraphics[width=0.24\columnwidth]{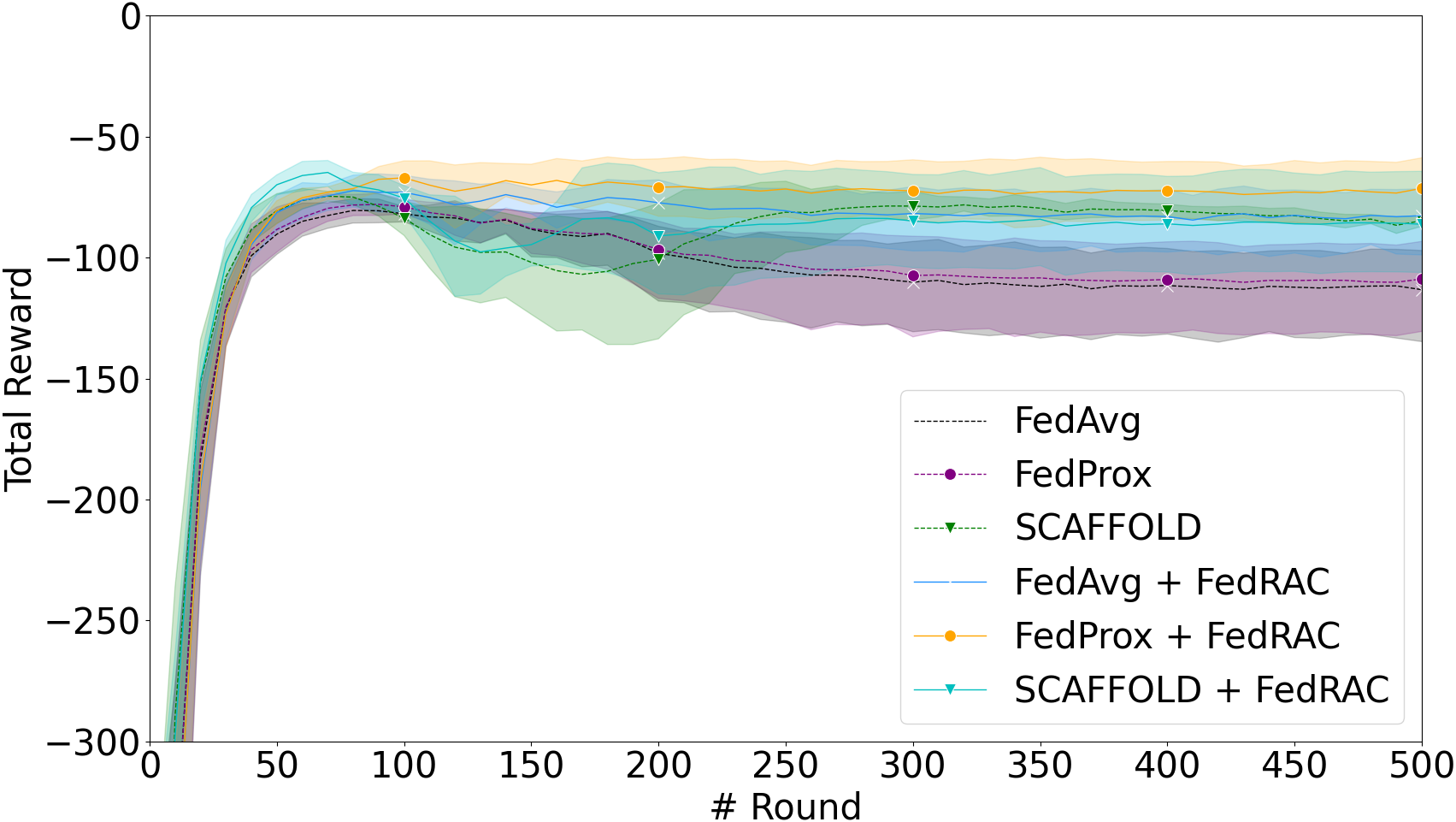}\label{fig:mcc_high_level_hete}}
\hfil
\subfloat[]{\includegraphics[width=0.24\columnwidth]{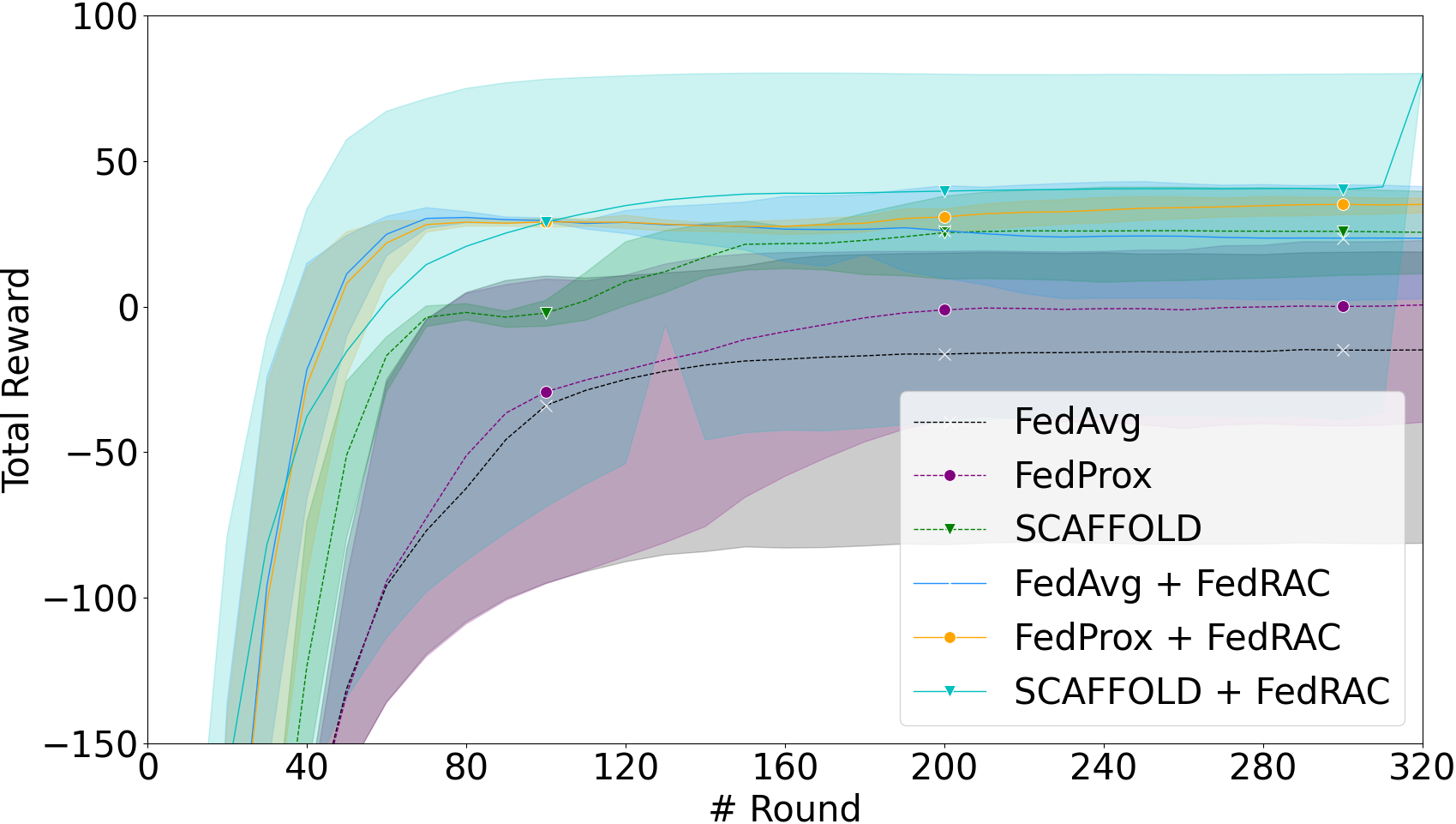}\label{fig:mcc_medium_level_hete_scalability}}
\caption{Learning curves on Mountain Cars with different levels of heterogeneity and a large-
scale scenario: (a) Low-Mountain Cars; (b) Medium-Mountain Cars; (c) High-Mountain Cars; (d)
Mountain Cars with 600 clients and a medium level of heterogeneity.}
\label{fig:heterogeneity_level}
\end{figure}


\subsection{Scalability \label{subsec:scalability}}

In Fig. \ref{fig:mcc_medium_level_hete_scalability}, we demonstrate the robustness of FedRAC for large-scale networks. We create a new Mountain Cars network with 600 clients and a medium level of heterogeneity, i.e., action shifts are uniformly sampled from $[-1.5, 1.5]$. FedRAC outperformed the baseline in this scenario, though SCAFFOLD demonstrated competitive performance even without reversing the update order.

%

\section{Conclusions and Limitations}

We explored the popular PPO algorithm in the FRL setting and proposed a novel yet simple approach (FedRAC) to mitigate the negative effect of data heterogeneity.
FedRAC only imposes lightweight modifications on the standard approach, promoting its potential application in real-world problems.
We provided the convergence analysis of FedRAC and the conventional approach, where the advantage of FedRAC is illustrated by a bound free from the data heterogeneity issue.
We have also investigated the impact of neural network/policy parameterization on convergence.
Experiment results showed that FedRAC outperforms the baseline on three heterogeneous networks based on classical RL environments. In addition, we evaluated FedRAC on an autonomous vehicles training problem to strengthen the connection between the theoretical setting and real-world applications.

Although the theory suggests that log-linear policy is desired, this finding is difficult to validate since there is a trade-off between approximation and aggregation error. In other words, although log-linear policies may minimize the error due to parameterization aggregation, they may suffer from a large approximation error due to locally under-fitted models. We leave this investigation to future work.

\bibliography{references}

\begin{thebibliography}{35}
\providecommand{\natexlab}[1]{#1}
\providecommand{\url}[1]{\texttt{#1}}
\expandafter\ifx\csname urlstyle\endcsname\relax
  \providecommand{\doi}[1]{doi: #1}\else
  \providecommand{\doi}{doi: \begingroup \urlstyle{rm}\Url}\fi

\bibitem[Antos et~al.(2007)Antos, Szepesv\'{a}ri, and Munos]{NIPS2007_da0d1111}
Antos, A., Szepesv\'{a}ri, C., and Munos, R.
\newblock Fitted q-iteration in continuous action-space mdps.
\newblock In Platt, J., Koller, D., Singer, Y., and Roweis, S. (eds.), \emph{Advances in Neural Information Processing Systems}, volume~20. Curran Associates, Inc., 2007.

\bibitem[Bhatnagar et~al.(2023)Bhatnagar, Borkar, and Guin]{10160088}
Bhatnagar, S., Borkar, V.~S., and Guin, S.
\newblock Actor–critic or critic–actor? a tale of two time scales.
\newblock \emph{IEEE Control Systems Letters}, 7:\penalty0 2671--2676, 2023.
\newblock \doi{10.1109/LCSYS.2023.3288931}.

\bibitem[Brockman et~al.(2016)Brockman, Cheung, Pettersson, Schneider, Schulman, Tang, and Zaremba]{DBLP:journals/corr/BrockmanCPSSTZ16}
Brockman, G., Cheung, V., Pettersson, L., Schneider, J., Schulman, J., Tang, J., and Zaremba, W.
\newblock Openai gym.
\newblock \emph{CoRR}, abs/1606.01540, 2016.

\bibitem[Cai et~al.(2019)Cai, Yang, Lee, and Wang]{NEURIPS2019_98baeb82}
Cai, Q., Yang, Z., Lee, J.~D., and Wang, Z.
\newblock Neural temporal-difference learning converges to global optima.
\newblock In Wallach, H., Larochelle, H., Beygelzimer, A., d\textquotesingle Alch\'{e}-Buc, F., Fox, E., and Garnett, R. (eds.), \emph{Advances in Neural Information Processing Systems}, volume~32. Curran Associates, Inc., 2019.

\bibitem[Cobbe et~al.(2021)Cobbe, Hilton, Klimov, and Schulman]{pmlr-v139-cobbe21a}
Cobbe, K.~W., Hilton, J., Klimov, O., and Schulman, J.
\newblock Phasic policy gradient.
\newblock In Meila, M. and Zhang, T. (eds.), \emph{Proceedings of the 38th International Conference on Machine Learning}, volume 139 of \emph{Proceedings of Machine Learning Research}, pp.\  2020--2027. PMLR, 18--24 Jul 2021.

\bibitem[Degris et~al.(2012)Degris, White, and Sutton]{Degris2012LinearOA}
Degris, T., White, M., and Sutton, R.
\newblock Linear off-policy actor-critic.
\newblock In \emph{ICML}, 2012.

\bibitem[Fan et~al.(2020)Fan, Wang, Xie, and Yang]{pmlr-v120-yang20a}
Fan, J., Wang, Z., Xie, Y., and Yang, Z.
\newblock A theoretical analysis of deep q-learning.
\newblock In Bayen, A.~M., Jadbabaie, A., Pappas, G., Parrilo, P.~A., Recht, B., Tomlin, C., and Zeilinger, M. (eds.), \emph{Proceedings of the 2nd Conference on Learning for Dynamics and Control}, volume 120 of \emph{Proceedings of Machine Learning Research}, pp.\  486--489. PMLR, 10--11 Jun 2020.

\bibitem[Fan et~al.(2021)Fan, Ma, Dai, Jing, Tan, and Low]{fan2021fault}
Fan, X., Ma, Y., Dai, Z., Jing, W., Tan, C., and Low, B. K.~H.
\newblock Fault-tolerant federated reinforcement learning with theoretical guarantee.
\newblock \emph{Advances in Neural Information Processing Systems}, 34, 2021.

\bibitem[Farahmand et~al.(2010)Farahmand, Szepesv\'{a}ri, and Munos]{NIPS2010_65cc2c82}
Farahmand, A.-m., Szepesv\'{a}ri, C., and Munos, R.
\newblock Error propagation for approximate policy and value iteration.
\newblock In Lafferty, J., Williams, C., Shawe-Taylor, J., Zemel, R., and Culotta, A. (eds.), \emph{Advances in Neural Information Processing Systems}, volume~23. Curran Associates, Inc., 2010.

\bibitem[Haarnoja et~al.(2017)Haarnoja, Tang, Abbeel, and Levine]{pmlr-v70-haarnoja17a}
Haarnoja, T., Tang, H., Abbeel, P., and Levine, S.
\newblock Reinforcement learning with deep energy-based policies.
\newblock In Precup, D. and Teh, Y.~W. (eds.), \emph{Proceedings of the 34th International Conference on Machine Learning}, volume~70 of \emph{Proceedings of Machine Learning Research}, pp.\  1352--1361. PMLR, 06--11 Aug 2017.

\bibitem[Jin et~al.(2022)Jin, Peng, Yang, Wang, and Zhang]{pmlr-v151-jin22a}
Jin, H., Peng, Y., Yang, W., Wang, S., and Zhang, Z.
\newblock Federated reinforcement learning with environment heterogeneity.
\newblock In Camps-Valls, G., Ruiz, F. J.~R., and Valera, I. (eds.), \emph{Proceedings of The 25th International Conference on Artificial Intelligence and Statistics}, volume 151 of \emph{Proceedings of Machine Learning Research}, pp.\  18--37. PMLR, 28--30 Mar 2022.

\bibitem[Karimireddy et~al.(2020)Karimireddy, Kale, Mohri, Reddi, Stich, and Suresh]{pmlr-v119-karimireddy20a}
Karimireddy, S.~P., Kale, S., Mohri, M., Reddi, S., Stich, S., and Suresh, A.~T.
\newblock {SCAFFOLD}: Stochastic controlled averaging for federated learning.
\newblock In III, H.~D. and Singh, A. (eds.), \emph{Proceedings of the 37th International Conference on Machine Learning}, volume 119 of \emph{Proceedings of Machine Learning Research}, pp.\  5132--5143. PMLR, 13--18 Jul 2020.

\bibitem[Konda \& Tsitsiklis(1999)Konda and Tsitsiklis]{NIPS1999_6449f44a}
Konda, V. and Tsitsiklis, J.
\newblock Actor-critic algorithms.
\newblock In Solla, S., Leen, T., and M\"{u}ller, K. (eds.), \emph{Advances in Neural Information Processing Systems}, volume~12. MIT Press, 1999.

\bibitem[Li et~al.(2020)Li, Sahu, Zaheer, Sanjabi, Talwalkar, and Smith]{MLSYS2020_38af8613}
Li, T., Sahu, A.~K., Zaheer, M., Sanjabi, M., Talwalkar, A., and Smith, V.
\newblock Federated optimization in heterogeneous networks.
\newblock In \emph{Proceedings of Machine Learning and Systems}, volume~2, pp.\  429--450, 2020.

\bibitem[{Li} et~al.(2020){Li}, {Huang}, {Yang}, {Wang}, and {Zhang}]{li2020on}
{Li}, X., {Huang}, K., {Yang}, W., {Wang}, S., and {Zhang}, Z.
\newblock On the convergence of fedavg on non-iid data.
\newblock In \emph{ICLR 2020 : Eighth International Conference on Learning Representations}, 2020.

\bibitem[Liu et~al.(2019)Liu, Cai, Yang, and Wang]{NEURIPS2019_227e072d}
Liu, B., Cai, Q., Yang, Z., and Wang, Z.
\newblock Neural trust region/proximal policy optimization attains globally optimal policy.
\newblock In Wallach, H., Larochelle, H., Beygelzimer, A., d\textquotesingle Alch\'{e}-Buc, F., Fox, E., and Garnett, R. (eds.), \emph{Advances in Neural Information Processing Systems}, volume~32. Curran Associates, Inc., 2019.

\bibitem[McMahan et~al.(2017)McMahan, Moore, Ramage, Hampson, and Arcas]{pmlr-v54-mcmahan17a}
McMahan, B., Moore, E., Ramage, D., Hampson, S., and Arcas, B. A.~y.
\newblock {Communication-Efficient Learning of Deep Networks from Decentralized Data}.
\newblock In \emph{Proceedings of the 20th International Conference on Artificial Intelligence and Statistics}, volume~54, pp.\  1273--1282, 20--22 Apr 2017.

\bibitem[Mei et~al.(2020)Mei, Xiao, Szepesvari, and Schuurmans]{pmlr-v119-mei20b}
Mei, J., Xiao, C., Szepesvari, C., and Schuurmans, D.
\newblock On the global convergence rates of softmax policy gradient methods.
\newblock In III, H.~D. and Singh, A. (eds.), \emph{Proceedings of the 37th International Conference on Machine Learning}, volume 119 of \emph{Proceedings of Machine Learning Research}, pp.\  6820--6829. PMLR, 13--18 Jul 2020.

\bibitem[Mnih et~al.(2015)Mnih, Kavukcuoglu, Silver, Rusu, Veness, Bellemare, Graves, Riedmiller, Fidjeland, Ostrovski, et~al.]{mnih2015human}
Mnih, V., Kavukcuoglu, K., Silver, D., Rusu, A.~A., Veness, J., Bellemare, M.~G., Graves, A., Riedmiller, M., Fidjeland, A.~K., Ostrovski, G., et~al.
\newblock Human-level control through deep reinforcement learning.
\newblock \emph{nature}, 518\penalty0 (7540):\penalty0 529--533, 2015.

\bibitem[Munos \& Szepesv{{\'a}}ri(2008)Munos and Szepesv{{\'a}}ri]{JMLR:v9:munos08a}
Munos, R. and Szepesv{{\'a}}ri, C.
\newblock Finite-time bounds for fitted value iteration.
\newblock \emph{Journal of Machine Learning Research}, 9\penalty0 (27):\penalty0 815--857, 2008.

\bibitem[{OpenStreetMap contributors}(2017)]{OpenStreetMap}
{OpenStreetMap contributors}.
\newblock {Planet dump retrieved from https://planet.osm.org }.
\newblock \url{ https://www.openstreetmap.org }, 2017.

\bibitem[Schulman et~al.(2015)Schulman, Levine, Moritz, Jordan, and Abbeel]{DBLP:journals/corr/SchulmanLMJA15}
Schulman, J., Levine, S., Moritz, P., Jordan, M.~I., and Abbeel, P.
\newblock Trust region policy optimization.
\newblock \emph{CoRR}, abs/1502.05477, 2015.

\bibitem[Schulman et~al.(2017)Schulman, Wolski, Dhariwal, Radford, and Klimov]{DBLP:journals/corr/SchulmanWDRK17}
Schulman, J., Wolski, F., Dhariwal, P., Radford, A., and Klimov, O.
\newblock Proximal policy optimization algorithms.
\newblock \emph{CoRR}, abs/1707.06347, 2017.

\bibitem[Sutton(1988)]{sutton1988learning}
Sutton, R.~S.
\newblock Learning to predict by the methods of temporal differences.
\newblock \emph{Machine learning}, 3\penalty0 (1):\penalty0 9--44, 1988.

\bibitem[Sutton et~al.(2000)Sutton, McAllester, Singh, and Mansour]{NIPS1999_464d828b}
Sutton, R.~S., McAllester, D., Singh, S., and Mansour, Y.
\newblock Policy gradient methods for reinforcement learning with function approximation.
\newblock In \emph{Advances in Neural Information Processing Systems}, volume~12, 2000.

\bibitem[Tosatto et~al.(2017)Tosatto, Pirotta, D'Eramo, and Restelli]{pmlr-v70-tosatto17a}
Tosatto, S., Pirotta, M., D'Eramo, C., and Restelli, M.
\newblock Boosted fitted q-iteration.
\newblock In Precup, D. and Teh, Y.~W. (eds.), \emph{Proceedings of the 34th International Conference on Machine Learning}, volume~70 of \emph{Proceedings of Machine Learning Research}, pp.\  3434--3443. PMLR, 06--11 Aug 2017.

\bibitem[Van~Hasselt et~al.(2016)Van~Hasselt, Guez, and Silver]{van2016deep}
Van~Hasselt, H., Guez, A., and Silver, D.
\newblock Deep reinforcement learning with double q-learning.
\newblock In \emph{Proceedings of the AAAI conference on artificial intelligence}, volume~30, 2016.

\bibitem[Vargas-Munoz et~al.(2021)Vargas-Munoz, Srivastava, Tuia, and Falcão]{9119753}
Vargas-Munoz, J.~E., Srivastava, S., Tuia, D., and Falcão, A.~X.
\newblock Openstreetmap: Challenges and opportunities in machine learning and remote sensing.
\newblock \emph{IEEE Geoscience and Remote Sensing Magazine}, 9\penalty0 (1):\penalty0 184--199, 2021.
\newblock \doi{10.1109/MGRS.2020.2994107}.

\bibitem[Vinitsky et~al.(2018)Vinitsky, Kreidieh, Flem, Kheterpal, Jang, Wu, Wu, Liaw, Liang, and Bayen]{pmlr-v87-vinitsky18a}
Vinitsky, E., Kreidieh, A., Flem, L.~L., Kheterpal, N., Jang, K., Wu, C., Wu, F., Liaw, R., Liang, E., and Bayen, A.~M.
\newblock Benchmarks for reinforcement learning in mixed-autonomy traffic.
\newblock In \emph{Proceedings of The 2nd Conference on Robot Learning}, volume~87, pp.\  399--409, 2018.

\bibitem[Wang et~al.(2019)Wang, Cai, Yang, and Wang]{wang2019neural}
Wang, L., Cai, Q., Yang, Z., and Wang, Z.
\newblock Neural policy gradient methods: Global optimality and rates of convergence.
\newblock \emph{arXiv preprint arXiv:1909.01150}, 2019.

\bibitem[Wang et~al.(2020)Wang, Foster, and Kakade]{wang2020statisticallimitsofflinerl}
Wang, R., Foster, D.~P., and Kakade, S.~M.
\newblock What are the statistical limits of offline rl with linear function approximation?, 2020.

\bibitem[Xie \& Song(2023{\natexlab{a}})Xie and Song]{10038492}
Xie, Z. and Song, S.
\newblock Fedkl: Tackling data heterogeneity in federated reinforcement learning by penalizing kl divergence.
\newblock \emph{IEEE Journal on Selected Areas in Communications}, 41\penalty0 (4):\penalty0 1227--1242, 2023{\natexlab{a}}.
\newblock \doi{10.1109/JSAC.2023.3242734}.

\bibitem[Xie \& Song(2023{\natexlab{b}})Xie and Song]{xie2023client}
Xie, Z. and Song, S.~H.
\newblock Client selection for federated policy optimization with environment heterogeneity.
\newblock \emph{arXiv preprint arXiv:1812.06210}, 2023{\natexlab{b}}.

\bibitem[Yang et~al.(2024)Yang, Cen, Wei, Chen, and Chi]{NEURIPS2024_dbdea785}
Yang, T., Cen, S., Wei, Y., Chen, Y., and Chi, Y.
\newblock Federated natural policy gradient and actor critic methods for multi-task reinforcement learning.
\newblock In Globerson, A., Mackey, L., Belgrave, D., Fan, A., Paquet, U., Tomczak, J., and Zhang, C. (eds.), \emph{Advances in Neural Information Processing Systems}, volume~37, pp.\  121304--121375. Curran Associates, Inc., 2024.

\bibitem[Zhao et~al.(2018)Zhao, Li, Lai, Suda, Civin, and Chandra]{DBLP:journals/corr/abs-1806-00582}
Zhao, Y., Li, M., Lai, L., Suda, N., Civin, D., and Chandra, V.
\newblock Federated learning with non-iid data.
\newblock \emph{CoRR}, abs/1806.00582, 2018.

\end{thebibliography}
\bibliographystyle{icml2024}


\newpage
\appendix
\section*{APPENDICES}

\section{Additional Implementation Details and Experiment Setting \label{sec:additionexperimentdetails}}

The hyperparameters for different experiments and the general FRL setting are shown in Table \ref{table:1} and Table \ref{table:4}, respectively. 

\begin{table}[ht!]
\caption{Hyperparameters for each environment.}
\label{table:1}
\centering
\begin{tabular}{c c c c c} 
 \hline
 Hyperparameter & Mountain Cars & Hoppers & Halfcheetahs & HongKongOSMs \\ [0.5ex] 
 \hline
 Learning Rate & 0.01 & 0.01 & 0.01 & 0.01 \\ 
 Learning Rate Decay & 0.98 & 0.98 & 0.98 & 0.98 \\
 Batch Size & 128 & 128 & 128 & 16 \\
 Timestep per Iteration (B) & 2048 & 2048 & 2048 & 256 \\
 Number of Epochs (E) & 10 & 10 & 10 & 10 \\
 Discount Factor ($\gamma$) & 0.99 & 0.99 & 0.99 & 0.99 \\
 Discount Factor for GAE & 0.95 & 0.95 & 0.95 & 0.95 \\
 KL Target & 0.001 & 0.01 & 0.01 & 1 \\
 $\mu$ for FedProx & 0.01 & 1 & 1 & 1\\ [1ex]
 \hline
\end{tabular}
\end{table}

\begin{table}[ht!]
\caption{General FRL setting. Refer to Section \ref{sec:problem_formulation} and Algorithm \ref{alg:FedRAC} for their definitions.}
\label{table:4}
\centering
\begin{tabular}{c c c c c} 
 \hline
 Environment & \#Client (N) & \#Participant (K) \\ [0.5ex] 
 \hline
 MountainCars & 60 & 6 \\
 Hoppers & 60 & 6 \\
 Halfcheetahs & 60 & 6 \\
 HongKongOSMs & 20 & 3 \\ [1ex]
 \hline
\end{tabular}
\end{table}

\begin{figure}[!b]
\centering
\subfloat[]{\includegraphics[width=0.14\columnwidth]{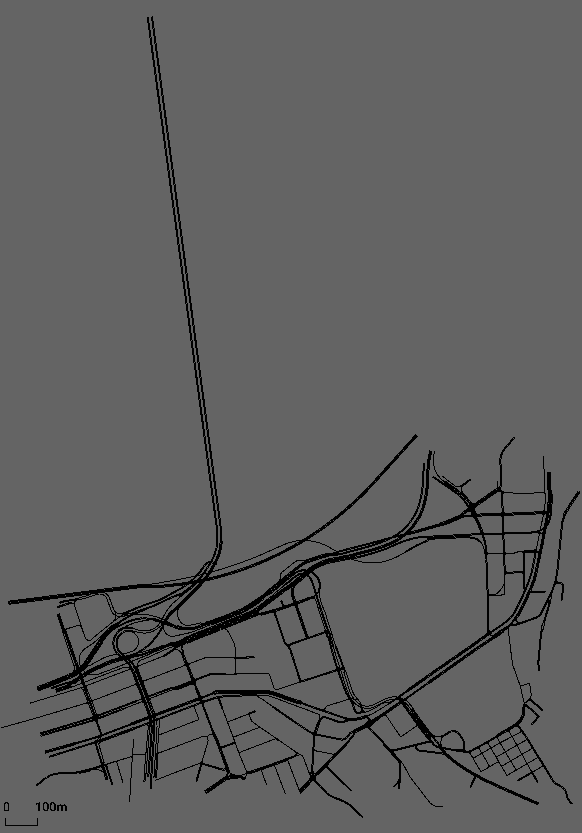}
\label{fig:causeway}}
\hfil
\subfloat[]{\includegraphics[width=0.14\columnwidth]{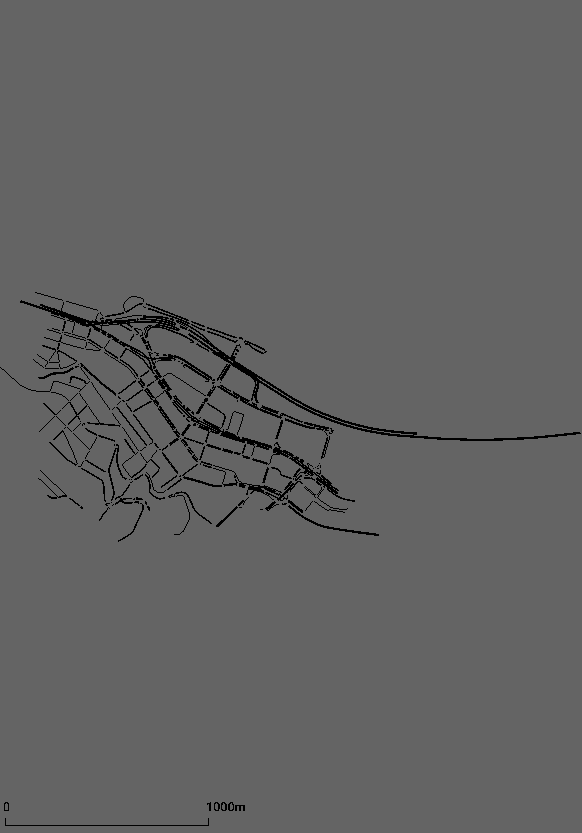}\label{fig:central}}
\hfil
\subfloat[]{\includegraphics[width=0.14\columnwidth]{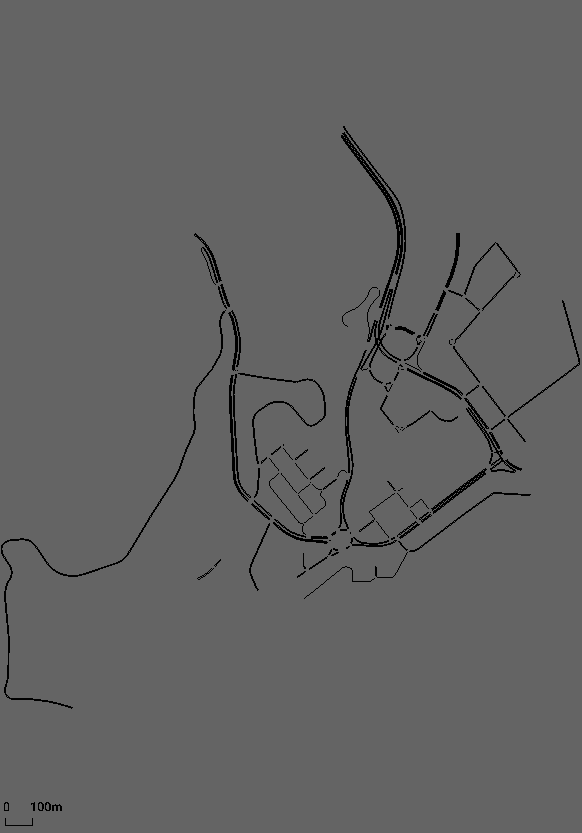}\label{fig:cw}}
\hfil
\subfloat[]{\includegraphics[width=0.14\columnwidth]{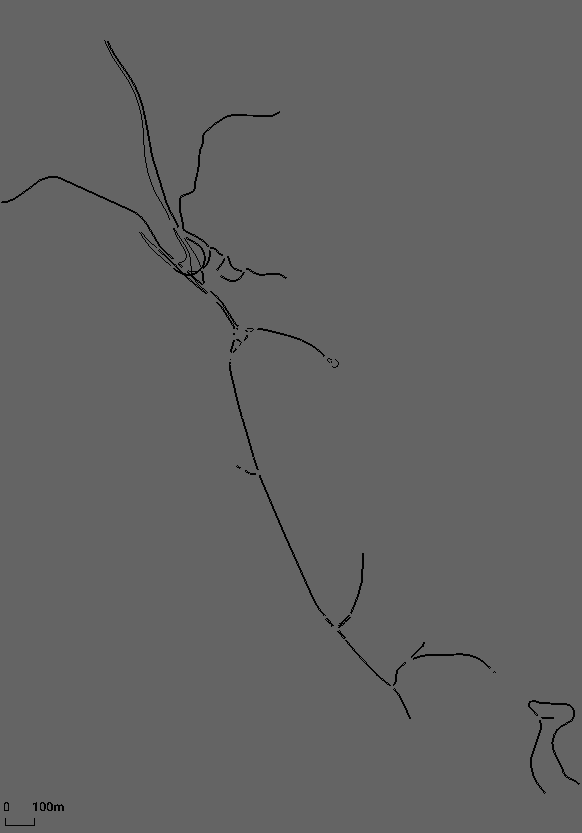}\label{fig:cwb}}
\hfil
\subfloat[]{\includegraphics[width=0.14\columnwidth]{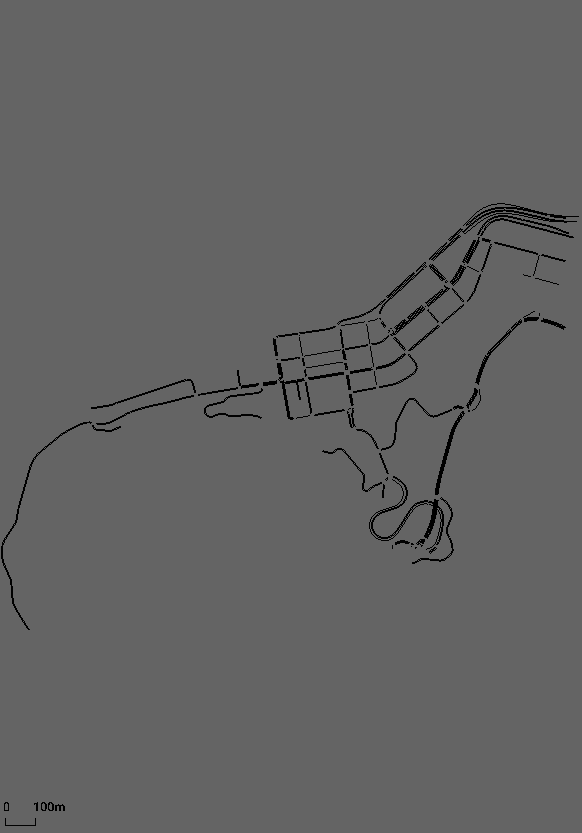}\label{fig:kennedy}}
\\
\subfloat[]{\includegraphics[width=0.14\columnwidth]{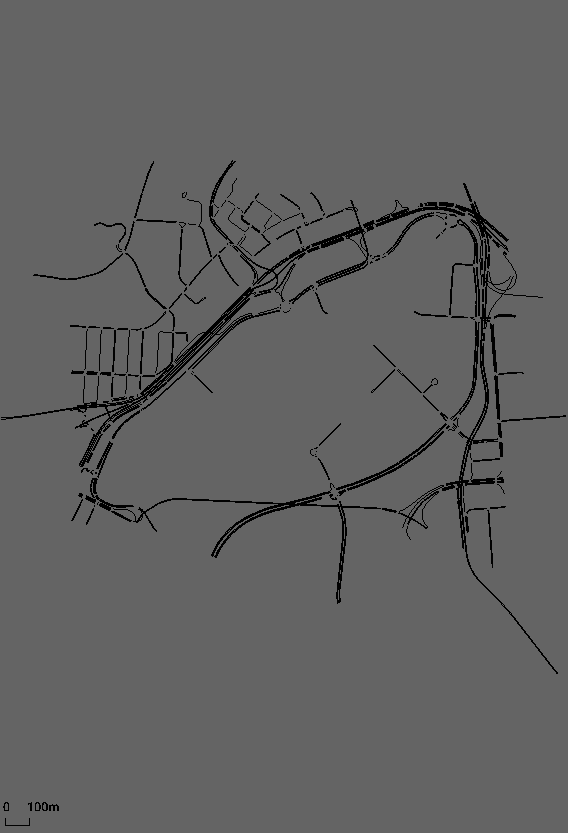}
\label{fig:kt}}
\hfil
\subfloat[]{\includegraphics[width=0.14\columnwidth]{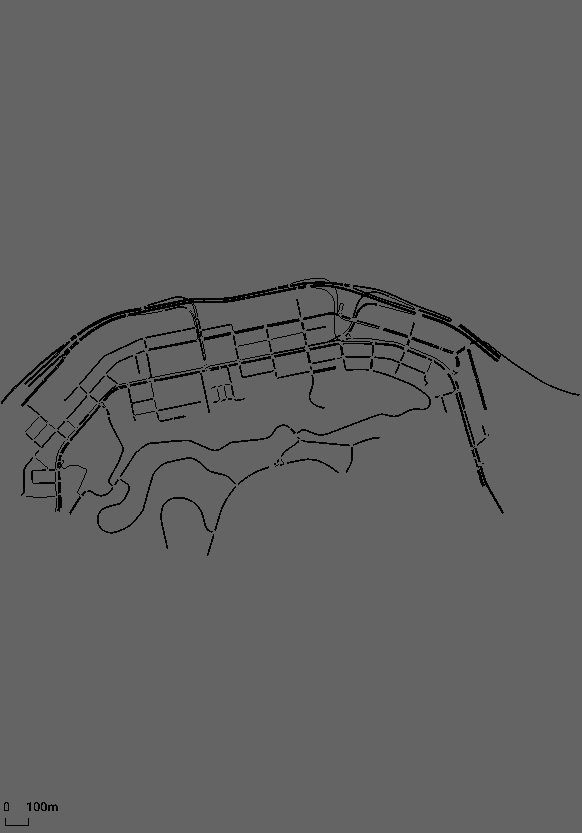}\label{fig:np}}
\hfil
\subfloat[]{\includegraphics[width=0.14\columnwidth]{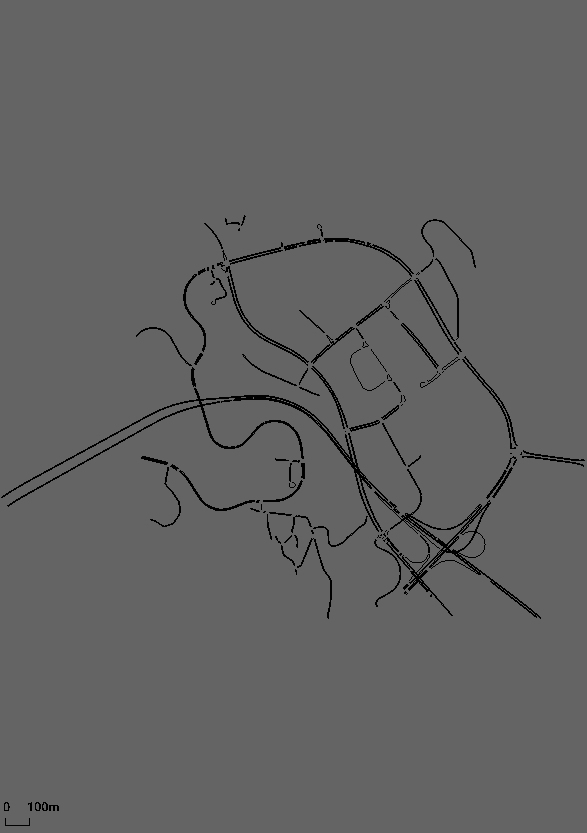}\label{fig:yt}}
\hfil
\subfloat[]{\includegraphics[width=0.14\columnwidth]{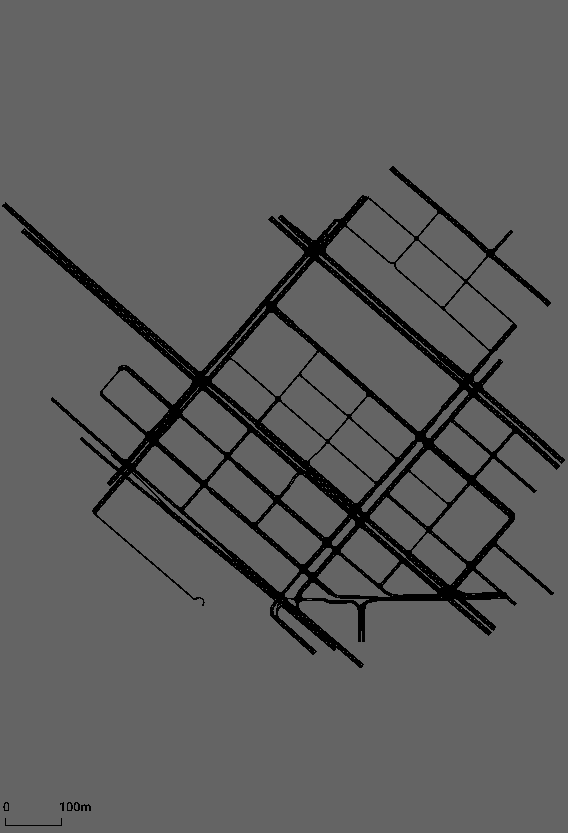}\label{fig:ssp}}
\hfil
\subfloat[]{\includegraphics[width=0.14\columnwidth]{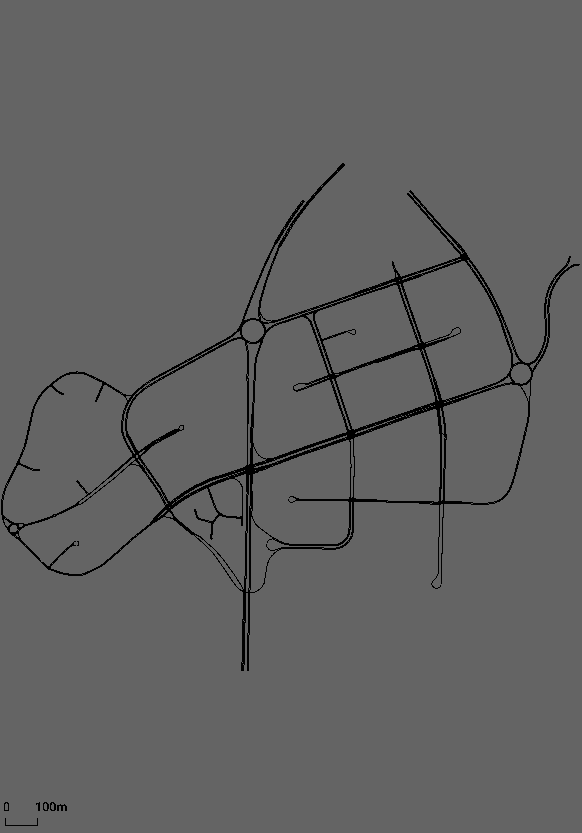}\label{fig:tko}}
\caption{OSM data of ten areas in Hong Kong. (a) Causeway; (b) Central; (c) Chai Wan; (d) Clear Water Bay; (e) Kennedy Town; (f) Kai Tak; (g) North Point; (h) Po Lam; (i) Sham Shui Po; (j) Tseung Kwan O.}
\label{fig:osms}
\end{figure}

\textbf{Implementation Details:} Since we are coping with continuous domains, we utilize Gaussian policies. Specifically, we use Multilayer Perceptrons (MLPs) with $\text{ReLU}$ non-linearity to output the mean of a Gaussian distribution, with a set of trainable standard deviations. We use two hidden layers (64, 64) for Mountain Cars and Half Cheetahs, and use four hidden layers (200, 100, 50, 25) for Hoppers and HongKongOSMs. 
We do not share parameters between the value function and policy, but they have the same neural network parameterization unless otherwise noted, e.g., Section \ref{subsec:2} and Fig. \ref{fig:parameterization}. We use the SGD optimizer with learning rate decay.

\textbf{Machines:} We simulate the federated learning experiments (1 server and N devices) on a commodity machine with 16 Intel(R) Xeon(R) Gold 6348 CPU @ 2.60GHZz. It took about 1, 5, and 30 minutes to finish one round of training for the Mountain Cars, Mujoco, and Sumo simulations, respectively.


\section{Proof of Lemma \ref{lemma:minimizer}\label{proof:lemma:minimizer}}

Neural PPO has proven the following policy update rule, which shows the relation between successive policies and suggests the subproblem of minimizing the MSE in (\ref{eq:policy_improvement}).
\begin{proposition}
\label{proposition:policy_update_relation}
\cite{NEURIPS2019_227e072d}. Given an arbitrary parameterized action-value function $Q^{w}$, the update
\begin{alignat}{1}
\pi^{t+1}_{n} &\leftarrow \arg\max_{\pi} \mathbb{E}_{s \thicksim \rho_{\pi^{\theta^{t}},n}} \left[ \left< Q^{w}(s,\cdot), \pi(\cdot \vert s) \right> - \beta_{t} \cdot D_{KL}(\pi(\cdot \vert s) \Vert \pi^{\theta^{t}}(\cdot \vert s)) \right] \nonumber
\end{alignat}
gives
\begin{alignat}{1}
\log \pi^{t+1}_{n} \propto \beta_{t} Q^{w} + \log \pi^{\theta^{t}}. \nonumber
\end{alignat}
\end{proposition}
The following proof is a duplication from \cite{NEURIPS2019_227e072d} which shows the relation between successive policies, but we do not restrict the proof to the Softmax policy.
\begin{proof}
The subproblem of policy improvement for solving $\pi^{t+1}_{n}$ takes the form
\begin{alignat}{1}
& \max_{\pi} \mathbb{E}_{s \thicksim \rho_{\pi^{\theta^{t}},n}} \left[ \left< \pi(\cdot \vert s), Q^{w}(s,\cdot) \right> - \beta_{t} \left< D_{KL}(\pi(\cdot \vert s) \Vert \pi^{\theta^{t}}(\cdot \vert s) \right> \right] \nonumber \\
& \text{subject to} \sum_{a \in \mathbf{A}} \pi(a \vert s) = 1, \forall s \in \mathbf{S}. \nonumber
\end{alignat}
The Lagrangian of the above maximization problem takes the form
\begin{align}
\int_{s \in \mathbf{S}} \left( \left< \pi(\cdot \vert s), Q^{w}(s,\cdot) \right> - \beta_{t} \left< D_{KL}(\pi(\cdot \vert s) \Vert \pi^{\theta^{t}}(\cdot \vert s) \right> \right) \rho_{\pi^{\theta^{t}},n} \mathrm{d}s + \int_{s \in \mathbf{S}} \left( \sum_{a \in \mathbf{A}} \pi(a \vert s) - 1 \right) \lambda \mathrm{d}s. \nonumber
\end{align}
The optimality condition is
\begin{align}
Q^{w}(s,a) + \beta_{t} \log \pi^{\theta^{t}} (a \vert s) - \beta_{t} \pi(a \vert s) - \beta_{t} + \frac{\lambda(s)}{\rho_{\pi^{\theta^{t}},n}(s)} = 0, \nonumber
\end{align}
which indicates
\begin{align}
\log \pi^{t+1}_{n} \propto \beta_{t} Q^{w} + \log \pi^{\theta^{t}}. \nonumber
\end{align}
\end{proof}

We extend the proof in \cite{NEURIPS2019_227e072d} to show that the subproblem in (\ref{eq:policy_improvement}) is a minimization of (\ref{eq:policy_improvement_obj}) for both Softmax and Gaussian policies. Lemma \ref{lemma:minimizer} is a direct consequence of Lemma \ref{lemma:minimizer_pre}.

\begin{lemma}
\label{lemma:minimizer_pre}
Suppose that the policy improvement error satisfies
\begin{alignat}{1}
& \mathbb{E}_{s \thicksim \rho_{\pi^{\theta^{t}},n}, a \thicksim \pi^{\theta^{t}}} \left[ \left( \tau_{t+1}^{-1} f_{\theta^{t+1}_{n}}(s,a) - \left( \beta_{t}^{-1} Q^{w}(s,a) + \tau_{t}^{-1} f_{\theta^{t}}(s,a) \right) \right)^{2} \right]^{1/2} \le \tau_{t+1}^{-1} \epsilon_{n}^{t+1}, \nonumber
\end{alignat}
and
\begin{alignat}{1}
& \mathbb{E}_{s \thicksim \rho_{\pi^{\theta^{t}},n}, a \thicksim \pi^{\theta^{t}}} \left[ \left( - \frac{(a - f_{\theta_{n}^{t+1}}(s,a))^{2}}{2 \sigma_{t+1}^{2}} - \beta_{t}^{-1} Q^{w}(s,a) + \frac{(a - f_{\theta^{t}}(s,a))^{2}}{2 \sigma_{t}^{2}} + \log \frac{\sigma_{t}}{\sigma_{t+1}} \right)^{2} \right] \nonumber \\
&\le \left( \sup \left\vert \upsilon_{t+1}^{-1}(s,a) \right\vert \epsilon_{n}^{t+1} \right)^{2}, \nonumber
\end{alignat}
for Softmax and Gaussian policies, respectively. We have
\begin{alignat}{1}
\mathbb{E}_{s \thicksim \rho_{\pi^{\theta^{t}},n}, a \thicksim \pi^{\theta^{t}}} \left[ \left( \pi^{\theta_{n}^{t+1}}(a \vert s) - \pi^{t+1}_{n} (a \vert s) \right)^{2} \right] \le 
\frac{\tau_{t+1}^{-2} \left( \epsilon_{n}^{t+1} \right)^{2}}{16} \nonumber
\end{alignat}
and
\begin{alignat}{1}
\mathbb{E}_{s \thicksim \rho_{\pi^{\theta^{t}},n}, a \thicksim \pi^{\theta^{t}}} \left[ \left( \pi^{\theta_{n}^{t+1}}(a \vert s) - \pi^{t+1}_{n}(a \vert s) \right)^{2} \right] < \left( \sup \left\vert \upsilon_{t+1}^{-1}(s,a) \right\vert \right)^{2} \left( \epsilon_{n}^{t+1} \right)^{2} \nonumber
\end{alignat}
for Softmax and Gaussian policies, respectively.
\end{lemma}
\begin{proof}
For Softmax policy, let $\tau^{-1}_{t+1} f_{n}^{t+1} = \beta_{t}^{-1} Q^{w^{t+1}_{n}} + \tau^{-1}_{t} f_{\theta^{t}}$. Since it is continuous w.r.t. $f$, by the mean value theorem, we have
\begin{alignat}{1}
\left\vert \pi^{\theta_{n}^{t+1}}(a \vert s) - \pi^{t+1}_{n}(a \vert s) \right\vert &= \left\vert \frac{\exp\left( \tau^{-1}_{n} f_{\theta_{n}^{t+1}}(s,a) \right)}{\sum_{a^{\prime} \in \mathcal{A}} \exp \left( \tau^{-1}_{n} f_{\theta_{n}^{t+1}}(s,a) \right)} - \frac{\exp\left( \tau^{-1}_{n} f_{n}^{t+1}(s,a) \right)}{\sum_{a^{\prime} \in \mathcal{A}} \exp \left( \tau^{-1}_{n} f_{n}^{t+1}(s,a) \right)} \right\vert \nonumber \\
&= \left\vert \frac{\partial}{\partial f(s,a)} \left( \frac{\exp\left( \tau^{-1}_{n} \tilde{f}(s,a) \right)}{\sum_{a^{\prime} \in \mathcal{A}} \exp \left( \tau^{-1}_{n} \tilde{f}(s,a) \right)} \right) \right\vert \cdot \left\vert f_{\theta_{n}^{t+1}}(s,a) - f_{n}^{t+1}(s,a) \right\vert, \nonumber
\end{alignat}
where $\tilde{f}$ is a function determined  by $f_{\theta_{n}^{t+1}}$ and $f_{n}^{t+1}$. Furthermore, we have
\begin{alignat}{1}
\left\vert \frac{\partial}{\partial f(s,a)} \left( \frac{\exp\left( \tau^{-1}_{n} \tilde{f}(s,a) \right)}{\sum_{a^{\prime} \in \mathcal{A}} \exp \left( \tau^{-1}_{n} \tilde{f}(s,a) \right)} \right) \right\vert = \tau_{t+1}^{-1} \cdot \pi(a \vert s) ( 1 - \pi(a \vert s)) \le \tau_{t+1}^{-1} / 4. \nonumber
\end{alignat}
Therefore, we obtain
\begin{alignat}{1}
& \mathbb{E}_{s \thicksim \rho_{\pi^{\theta^{t}},n}, a \thicksim \pi^{\theta^{t}}} \left[ \left( \pi^{\theta_{n}^{t+1}}(a \vert s) - \pi^{t+1}_{n}(a \vert s) \right)^{2} \right] \nonumber \\
&\le \mathbb{E}_{s \thicksim \rho_{\pi^{\theta^{t}},n}, a \thicksim \pi^{\theta^{t}}} \left[ \left( \tau_{t+1}^{-1} f_{\theta^{t+1}_{n}}(s,a) - \left( \beta_{t}^{-1} Q^{w}(s,a) + \tau_{t}^{-1} f_{\theta^{t}}(s,a) \right) \right)^{2} \right] / 16 \nonumber \\
&\le \frac{\tau_{t+1}^{-2} \left( \epsilon_{n}^{t+1} \right)^{2}}{16}. \nonumber
\end{alignat}

For Gaussian policy, let $-\frac{\left( a - f_{n}^{t+1}(s,a) \right)^{2}}{2\sigma
^{2}_{t+1}} - \log \sqrt{2\pi}\sigma_{t+1} = \beta_{t}^{-1} Q^{w^{t+1}_{n}} -\frac{\left( a - f_{\theta_{t}}(s,a) \right)^{2}}{2\sigma
^{2}_{t}} - \log \sqrt{2\pi}\sigma_{t}$. Since $e^{x}$ is continuous w.r.t. $x$, by the mean value theorem, we have
\begin{alignat}{1}
& \left\vert \pi^{\theta_{n}^{t+1}}(a \vert s) - \pi^{t+1}_{n}(a \vert s) \right\vert \nonumber \\
&= \left\vert \frac{1}{\sqrt{2\pi}\sigma_{t+1}}\exp\left( -\frac{\left( a - f_{\theta_{n}^{t+1}}(s,a) \right)^{2}}{2\sigma
^{2}_{t+1}} \right) - \frac{1}{\sqrt{2\pi}\sigma_{t+1}}\exp\left( -\frac{\left( a - f_{n}^{t+1}(s,a) \right)^{2}}{2\sigma
^{2}_{t+1}} \right) \right\vert \nonumber \\
&= \left\vert \frac{\partial \left( \frac{1}{\sqrt{2\pi}\sigma_{t+1}}\exp\left( -\frac{\left( a - \tilde{f}(s,a) \right)^{2}}{2\sigma^{2}_{t+1}} \right) \right)}{\partial -\frac{\left( a - f(s,a) \right)^{2}}{2\sigma^{2}_{t+1}}} \right\vert \cdot \left\vert \frac{\left( a - f_{n}^{t+1}(s,a) \right)^{2}}{2\sigma
^{2}_{t+1}} -\frac{\left( a - f_{\theta_{n}^{t+1}}(s,a) \right)^{2}}{2\sigma^{2}_{t+1}} \right\vert, \nonumber
\end{alignat}
where $\tilde{f}$ is a function determined  by $f_{\theta_{n}^{t+1}}$ and $f_{n}^{t+1}$. Furthermore, we have
\begin{alignat}{1}
\left\vert \frac{\partial}{\partial -\frac{\left( a - f(s,a) \right)^{2}}{2\sigma^{2}_{t+1}}} \left( \frac{1}{\sqrt{2\pi}\sigma_{t+1}}\exp\left(\frac{\left( a - f_{n}^{t+1}(s,a) \right)^{2}}{- 2\sigma
^{2}_{t+1}} \right) \right) \right\vert = \pi(a \vert s) - \frac{\sigma_{t+1}}{\sqrt{2 \pi} \left( a - f(s,a) \right)^{2}} < 1. \nonumber
\end{alignat}
Therefore, we obtain
\begin{alignat}{1}
& \mathbb{E}_{s \thicksim \rho_{\pi^{\theta^{t}},n}, a \thicksim \pi^{\theta^{t}}} \left[ \left( \pi^{\theta_{n}^{t+1}}(a \vert s) - \pi^{t+1}_{n}(a \vert s) \right)^{2} \right] \nonumber \\
&< \mathbb{E}_{s \thicksim \rho_{\pi^{\theta^{t}},n}, a \thicksim \pi^{\theta^{t}}} \left[ \left( - \frac{(a - f_{\theta_{n}^{t+1}}(s,a))^{2}}{2 \sigma_{t+1}^{2}} - \beta_{t}^{-1} Q^{w}(s,a) + \frac{(a - f_{\theta^{t}}(s,a))^{2}}{2 \sigma_{t}^{2}} + \log \frac{\sigma_{t}}{\sigma_{t+1}} \right)^{2} \right] \nonumber \\
&< \left( \sup \left\vert \upsilon_{t+1}^{-1}(s,a) \right\vert \right)^{2} \left( \epsilon_{n}^{t+1} \right)^{2}. \nonumber
\end{alignat}
\end{proof}

In the following, we present a few useful lemmas that the remaining proofs will frequently use. In particular, Lemma \ref{lemma:linearization_error} was proven by FedPOHCS \cite{xie2023client}, which measures the difference in output due to linearizing the two-layer neural network. With this error measurement, we can simplify the analysis of parameter aggregation by transforming two-layer neural networks to linear function approximators.
\begin{lemma}
\label{lemma:linearization_error}
Let $h_{\vartheta}$ be the placeholder for $f_{\theta}$ and $u_{w}$. Define the linearization of the two-layer neural network at its initialization point as:
\begin{alignat}{1}
h_{\vartheta^{t}}^{0}(s,a) = \frac{1}{\sqrt{m}} \sum_{i}^{m} b_{i} \cdot \mathbf{1} \left\{ \left(\vartheta^{0}_{i} \right)^{T}(s,a) > 0 \right\} \left( \vartheta^{t}_{i} \right)^{T}(s,a). \nonumber
\end{alignat}
Then, for all $\vartheta \in \mathcal{B}_{R_{\vartheta}}$, policy $\pi$, and client $n$, we have
\begin{alignat}{1}
\mathbb{E}_{\text{init},s \thicksim \rho_{\pi,n}, a \thicksim \pi} \left[ \left\vert h_{\vartheta}(s,a) - h_{\vartheta}^{0}(s,a) \right\vert \right] &= \mathcal{O} \left( R_{\vartheta}^{6/5} m^{-1/10} \hat{R}_{\vartheta}^{2/5} \right). \nonumber
\end{alignat}
\end{lemma}
\begin{proof}
Please refer to \cite{xie2023client} for the proof.
\end{proof}

\begin{lemma}
\label{lemma:bounded_by_r}
Let $h_{\vartheta}$ be the placeholder for $f_{\theta}$ and $u_{w}$. For all $\vartheta \in \mathcal{B}_{R_{\vartheta}}$, $t > 0$, $s \in \mathcal{S}$, and $a \in \mathcal{A}$, we have
\begin{alignat}{1}
\left( h_{\vartheta^{t}}(s,a) - h_{\vartheta^{0}}(s,a) \right)^{2} \le R_{\vartheta}^{2}, \label{eq:u_bounded_by_r} \\
\left( h_{\vartheta^{t}}^{0}(s,a) - h_{\vartheta^{0}}(s,a) \right)^{2} \le R_{\vartheta}^{2}. \label{eq:u0_bounded_by_r}
\end{alignat}
\end{lemma}

\begin{proof}
We first prove the 1-Lipschitz continuity of $h_{\vartheta^{t}}(s,a)$ and $h_{\vartheta^{t}}^{0}(s,a)$ and the results follows.

Note that the gradient of $h_{\vartheta^{t}}(s,a)$ w.r.t. to $\vartheta$ takes the form
\begin{alignat}{1}
\nabla_{\vartheta} h_{\vartheta^{t}}(s,a) & = \frac{1}{\sqrt{m}} \left( b_{1} \cdot \mathbf{1} \left\{ \vartheta_{1}^{T}(s,a) > 0 \right\} \cdot (s,a)^{T}, \dots, b_{m} \cdot \mathbf{1} \left\{ \vartheta_{m}^{T}(s,a) > 0 \right\} \cdot (s,a)^{T} \right)^{T}, \nonumber  
\end{alignat}
which yields
$\left\Vert \nabla_{\vartheta} h_{\vartheta}(s,a) \right\Vert^{2}_{2} = \frac{1}{m} \sum_{i=1}^{m} \mathbf{1} \left\{ \vartheta_{i}^{T}(s,a) > 0 \right\} \left\Vert (s,a) \right\Vert^{2}_{2} \le 1.$
Therefore, $h_{\vartheta}(s,a)$ is 1-Lipschitz continuous w.r.t. $\vartheta$. Similarly, it is easy to verify that $h_{\vartheta}^{0}(s,a)$ is 1-Lipschitz continuous w.r.t. $\vartheta$. By definition, we have
\begin{alignat}{1}
\left( h_{\vartheta^{t}}(s,a) - h_{\vartheta^{0}}(s,a) \right)^{2} \le \left\Vert \vartheta^{t} - \vartheta^{0} \right\Vert_{2}^{2} \le R_{\vartheta}^{2}. \nonumber
\end{alignat}
This completes the proof of (\ref{eq:u_bounded_by_r}), and the proof of (\ref{eq:u0_bounded_by_r}) is similar.
\end{proof}

\section{Proof of Lemma \ref{lemma:error_propagation}\label{proof:lemma:error_propagation}}

We start with a useful lemma.
\begin{lemma}
\label{lemma:taylor_expansion}
For all $t > 0, s \in \mathcal{S}, a \in \mathcal{A}$ and an arbitrary $X$, e.g., $X = \beta_{t}^{-1} Q^{w}(s,a)$, we have
\begin{alignat}{1}
& \left( - \frac{(a - f_{\theta_{n}^{t+1}}(s,a) )^{2}}{2 \sigma_{t+1}^{2}} - \log \sigma_{t+1} - X + \frac{(a - f_{\theta^{t}}(s,a) )^{2}}{2 \sigma_{t}^{2}} + \log \sigma_{t} \right)^{2} \nonumber \\
& \le 3 \left( \frac{f_{\theta^{t+1}_{n}}(s,a)}{\upsilon_{t+1}(s,a) } - \frac{f_{\theta^{t}}(s,a)}{\upsilon_{t}(s,a)} - X \right)^{2} + 3 \left( \log \frac{\sigma_{t}}{\sigma_{t+1}} + \frac{1}{2} \left( \frac{1}{\sigma_{t+1}^{2}} - \frac{1}{\sigma_{t}^{2}} \right) ( f_{\theta^{0}}^{2}(s,a) - a^{2} ) \right)^{2} \nonumber \\
& \quad + 3 \left( \frac{\left( f_{\theta^{t}}(s,a) - f_{\theta^{0}}(s,a) \right)^{2}}{2 \sigma_{t}^{2}} - \frac{\left( f_{\theta^{t+1}_{n}}(s,a) - f_{\theta^{0}}(s,a) \right)^{2}}{2 \sigma_{t+1}^{2}} \right)^{2}. \nonumber
\end{alignat}
\end{lemma}
\begin{proof}
By Taylor expansion, we have
\begin{alignat}{1}
& \left( - \frac{(a - f_{\theta_{n}^{t+1}}(s,a) )^{2}}{2 \sigma_{t+1}^{2}} - \log \sigma_{t+1} - X + \frac{(a - f_{\theta^{t}}(s,a) )^{2}}{2 \sigma_{t}^{2}} + \log \sigma_{t} \right)^{2} \nonumber \\
&= \left( \frac{(a - f_{\theta^{0}}(s,a) )^{2}}{2 \sigma_{t}^{2}} + \frac{( f_{\theta^{0}}(s,a) - a ) ( f_{\theta^{t}}(s,a) - f_{\theta^{0}}(s,a) )}{\sigma_{t}^{2}} + \frac{( f_{\theta^{t}}(s,a) - f_{\theta^{0}}(s,a) )^{2}}{2 \sigma_{t}^{2}} + \log \frac{\sigma_{t}}{\sigma_{t+1}} \right. \nonumber \\
& \quad \left. - \frac{(a - f_{\theta^{0}}(s,a) )^{2}}{2 \sigma_{t+1}^{2}} - \frac{( f_{\theta^{0}}(s,a) - a ) ( f_{\theta_{n}^{t+1}}(s,a) - f_{\theta^{0}}(s,a) )}{\sigma_{t+1}^{2}} - \frac{( f_{\theta^{t+1}_{n}}(s,a) - f_{\theta^{0}}(s,a) )^{2}}{2 \sigma_{t+1}^{2}} - X \right)^{2}. \nonumber
\end{alignat}
Rearranging these terms, we have
\begin{alignat}{1}
& \left( - \frac{(a - f_{\theta_{n}^{t+1}}(s,a) )^{2}}{2 \sigma_{t+1}^{2}} - \log \sigma_{t+1} - X + \frac{(a - f_{\theta^{t}}(s,a) )^{2}}{2 \sigma_{t}^{2}} + \log \sigma_{t} \right)^{2} \nonumber \\
&= \left( \frac{f_{\theta^{t+1}_{n}}(s,a)}{\upsilon_{t+1}(s,a)} - \frac{f_{\theta^{t}}(s,a)}{\upsilon_{t}(s,a)} + \frac{( f_{\theta^{t}}(s,a) - f_{\theta^{0}}(s,a) )^{2}}{2 \sigma_{t}^{2}} - \frac{( f_{\theta^{t+1}_{n}}(s,a) - f_{\theta^{0}}(s,a) )^{2}}{2 \sigma_{t+1}^{2}} + \log \frac{\sigma_{t}}{\sigma_{t+1}} \right. \nonumber \\
& \quad \left. + \left( \frac{f_{\theta^{0}}(s,a)}{\sigma_{t+1}^{2}} - \frac{f_{\theta^{0}}(s,a)}{\sigma_{t}^{2}} - \frac{f_{\theta^{0}}(s,a) - a}{2 \sigma
_{t+1}^{2}} + \frac{f_{\theta^{0}}(s,a) - a}{2 \sigma_{t}^{2}} \right) ( f_{\theta^{0}}(s,a) - a ) - X \right)^{2} \nonumber \\
&\le 3 \left( \frac{f_{\theta^{t+1}_{n}}(s,a)}{\upsilon_{t+1}(s,a) } - \frac{f_{\theta^{t}}(s,a)}{\upsilon_{t}(s,a)} - X \right)^{2} + 3 \left( \log \frac{\sigma_{t}}{\sigma_{t+1}} + \frac{1}{2} \left( \frac{1}{\sigma_{t+1}^{2}} - \frac{1}{\sigma_{t}^{2}} \right) ( f_{\theta^{0}}^{2}(s,a) - a^{2} ) \right)^{2} \nonumber \\
& \quad + 3 \left( \frac{\left( f_{\theta^{t}}(s,a) - f_{\theta^{0}}(s,a) \right)^{2}}{2 \sigma_{t}^{2}} - \frac{\left( f_{\theta^{t+1}_{n}}(s,a) - f_{\theta^{0}}(s,a) \right)^{2}}{2 \sigma_{t+1}^{2}} \right)^{2}, \nonumber
\end{alignat}
where the last inequality follows from $(a + b + c)^{2} \le 3a^{2} + 3b^{2} + 3c^{2}$.
\end{proof}

Next, we prove Lemma \ref{lemma:error_propagation}.
\begin{proof}
First, we focus on the policy improvement error of FedRAC. By Lemma \ref{lemma:taylor_expansion}, we have
\begin{alignat}{1}
& \mathbb{E}_{s \thicksim \rho_{\pi^{\theta^{t}},n}, a \thicksim \pi^{\theta^{t}}} \left[ \left( - \frac{ (a - f_{\theta_{n}^{t+1}}(s,a) )^{2}}{2 \sigma_{t+1}^{2}} - \beta_{t}^{-1} Q^{w}(s,a) + \frac{ (a - f_{\theta^{t}}(s,a) )^{2}}{2 \sigma_{t}^{2}} + \log \frac{\sigma_{t}}{\sigma_{t+1}} \right)^{2} \right] \nonumber \\
&\le \mathbb{E}_{s \thicksim \rho_{\pi^{\theta^{t}},n}, a \thicksim \pi^{\theta^{t}}} \left[ 3 \left( \upsilon_{t+1}^{-1}(s,a) f_{\theta^{t+1}_{n}}(s,a) - \upsilon_{t}^{-1}(s,a) f_{\theta^{t}} - \beta_{t}^{-1} Q^{w}(s,a) \right)^{2} \right. \nonumber \\
& \quad \left. + 3 \left( \frac{1}{2 \sigma_{t}^{2}} ( f_{\theta^{t}}(s,a) - f_{\theta^{0}}(s,a) )^{2} - \frac{1}{2 \sigma_{t+1}^{2}} ( f_{\theta^{t+1}_{n}}(s,a) - f_{\theta^{0}}(s,a) )^{2} \right)^{2} \right. \nonumber \\
& \quad \left. + 3 \left( \log \frac{\sigma_{t}}{\sigma_{t+1}} + \frac{1}{2} \left( \frac{1}{\sigma_{t+1}^{2}} - \frac{1}{\sigma_{t}^{2}} \right) ( f_{\theta^{0}}^{2}(s,a) - a^{2} ) \right)^{2} \right]. \label{eq:expected_taylor_expansion}
\end{alignat}
By Lemma \ref{lemma:bounded_by_r}, since the standard deviation $\sigma_{t}$ is monotonically decreasing, we have
\begin{alignat}{1}
3 \left( \frac{1}{2 \sigma_{t}^{2}} ( f_{\theta^{t}}(s,a) - f_{\theta^{0}}(s,a) )^{2} - \frac{1}{2 \sigma_{t+1}^{2}} ( f_{\theta^{t+1}_{n}}(s,a) - f_{\theta^{0}}(s,a) )^{2} \right)^{2} \le \frac{3 R_{\theta}^{4}}{4 \sigma_{t+1}^{4}}. \label{eq:square_sigma_r}
\end{alignat}
By substituting (\ref{eq:square_sigma_r}) into (\ref{eq:expected_taylor_expansion}), we can obtain
\begin{alignat}{1}
& \mathbb{E}_{s \thicksim \rho_{\pi^{\theta^{t}},n}, a \thicksim \pi^{\theta^{t}}} \left[ \left( - \frac{ (a - f_{\theta_{n}^{t+1}}(s,a) )^{2}}{2 \sigma_{t+1}^{2}} - \beta_{t}^{-1} Q^{w}(s,a) + \frac{ (a - f_{\theta^{t}}(s,a) )^{2}}{2 \sigma_{t}^{2}} + \log \frac{\sigma_{t}}{\sigma_{t+1}} \right)^{2} \right] \nonumber \\
&\le 3 \mathbb{E}_{s \thicksim \rho_{\pi^{\theta^{t}},n}, a \thicksim \pi^{\theta^{t}}} \left[ \left( \upsilon_{t+1}^{-1}(s,a) f_{\theta^{t+1}_{n}}(s,a) - \left( \beta_{t}^{-1} Q^{w}(s,a) + \upsilon_{t}^{-1}(s,a) f_{\theta^{t}}(s,a) \right) \right)^{2} \right] \nonumber \\
& \quad + 3 \mathbb{E}_{s \thicksim \rho_{\pi^{\theta^{t}},n}, a \thicksim \pi^{\theta^{t}}} \left[ \left( \log \frac{\sigma_{t}}{\sigma_{t+1}} + \frac{1}{2} \left( \frac{1}{\sigma_{t+1}^{2}} - \frac{1}{\sigma_{t}^{2}} \right) ( f_{\theta^{0}}^{2}(s,a) - a^{2} ) \right)^{2} \right] + \frac{3 R_{\theta}^{4}}{4 \sigma_{t+1}^{4}}. \nonumber
\end{alignat}
Therefore,
\begin{alignat}{1}
& \mathbb{E}_{s \thicksim \rho_{\pi^{\theta^{t}},n}, a \thicksim \pi^{\theta^{t}}} \left[ \left( \upsilon_{t+1}^{-1}(s,a) f_{\theta^{t+1}_{n}}(s,a) - \left( \beta_{t}^{-1} Q^{w}(s,a) + \upsilon_{t}^{-1}(s,a) f_{\theta^{t}}(s,a) \right) \right)^{2} \right] \nonumber \\
&\le \frac{\left( \sup \left\vert \upsilon_{t+1}^{-1}(s,a) \right\vert \epsilon_{n}^{t+1} \right)^{2}}{3} - \phi^{t+1}_{n} - \frac{R_{\theta}^{4}}{4 \sigma_{t+1}^{4}} \label{eq:proof:policy_improvement_error_fedrac0}
\end{alignat}
indicates
\begin{alignat}{1}
&\mathbb{E}_{s \thicksim \rho_{\pi^{\theta^{t}},n}, a \thicksim \pi^{\theta^{t}}} \left[ \left( - \frac{(a - f_{\theta_{n}^{t+1}}(s,a))^{2}}{2 \sigma_{t+1}^{2}} - \beta_{t}^{-1} Q^{w}(s,a) + \frac{(a - f_{\theta^{t}}(s,a))^{2}}{2 \sigma_{t}^{2}} + \log \frac{\sigma_{t}}{\sigma_{t+1}} \right)^{2} \right] \nonumber \\
&\le \left( \sup \left\vert \upsilon_{t+1}^{-1}(s,a) \right\vert \epsilon_{n}^{t+1} \right)^{2}. \label{eq:proof:policy_improvement_error_fedrac}
\end{alignat}
(\ref{eq:proof:policy_improvement_error_fedrac0}) and (\ref{eq:proof:policy_improvement_error_fedrac}) is crucial for linking FedRAC with the baseline, as it illustrates the accuracy of the policy improvement step required by FedRAC to achieve comparable accuracy with the baseline. Next, we turn to the policy evaluation error.
\begin{alignat}{1}
& \mathbb{E}_{s \thicksim \rho_{\pi^{\theta^{t}},n}, a \thicksim \pi^{\theta^{t}}} \left[ \left( Q^{w^{t}}(s,a) - \sum_{i} \frac{q_{i}\rho_{\pi^{\ast},i}(s)}{Z_{\pi^{\ast}}(s)} Q^{\pi^{\theta^{t}}}_{i}(s,a) \right)^{2} \right] \nonumber \\
&\le \mathbb{E}_{s \thicksim \rho_{\pi^{\theta^{t}},n}, a \thicksim \pi^{\theta^{t}}} \left[ 2 \left( Q^{w^{t}}(s,a) - \sum_{j} q_{j} Q^{w_{j}^{t}}(s,a) \right)^{2} \right] \nonumber \\
& \quad + \mathbb{E}_{s \thicksim \rho_{\pi^{\theta^{t}},n}, a \thicksim \pi^{\theta^{t}}} \left[ 2 \left( \sum_{j} q_{j} Q^{w_{j}^{t}}(s,a) - \sum_{i} \frac{q_{i}\rho_{\pi^{\ast},i}(s)}{Z_{\pi^{\ast}}(s)} Q^{\pi^{\theta^{t}}}_{i}(s,a) \right)^{2} \right], \label{eq:policy_evaluation_error_fedrac0}
\end{alignat}
where we have used the fact that $(a + b)^{2} \le 2a^{2} + 2b^{2}$. The first term in the RHS of (\ref{eq:policy_evaluation_error_fedrac0}) can be bounded by Lemma \ref{lemma:bounded_by_r} and Jensen's inequality as
\begin{alignat}{1}
& 2 \left( Q^{w^{t}}(s,a) - \sum_{j} q_{j} Q^{w_{j}^{t}}(s,a) \right)^{2} \nonumber \\
&\le 4 \left( Q^{w^{t}}(s,a) - Q^{w^{0}}(s,a) \right)^{2} + 4 \sum_{j} q_{j} \left( Q^{w^{0}}(s,a) - Q^{w_{j}^{t}}(s,a) \right)^{2} \nonumber \\
&\le 8 R_{w}^{2}. \label{eq:qet-qet0}
\end{alignat}
Therefore, it remains to bound the second term in the RHS of (\ref{eq:policy_evaluation_error_fedrac0}). By Jensen's inequality, we have
\begin{alignat}{1}
\left( \sum_{j} q_{j} Q^{w_{j}^{t}}(s,a) - \sum_{i} \frac{q_{i}\rho_{\pi^{\ast},i}(s)}{Z_{\pi^{\ast}}(s)} Q^{\pi^{\theta^{t}}}_{i}(s,a) \right)^{2} &\le \sum_{j} q_{j} \left( Q^{w_{j}^{t}}(s,a) - \frac{\rho_{\pi^{\ast},j}(s)}{Z_{\pi^{\ast}}(s)} Q^{\pi^{\theta^{t}}}_{j}(s,a) \right)^{2} \nonumber \\
&= \mathbb{E}_{i \thicksim \mathcal{C}} \left[ \left( Q^{w_{i}^{t}}(s,a) - \frac{\rho_{\pi^{\ast},i}(s)}{Z_{\pi^{\ast}}(s)} Q^{\pi^{\theta^{t}}}_{i}(s,a) \right)^{2} \right]. \nonumber
\end{alignat}
By the fact that $(a + b)^{2} \le 2a^{2} + 2b^{2}$ and $ab - cd = a(b-d) + d(a-c)$, we have
\begin{alignat}{1}
& \mathbb{E}_{s \thicksim \rho_{\pi^{\theta^{t}},n}, a \thicksim \pi^{\theta^{t}}} \left[ 2 \sum_{i} q_{i} \left( Q^{w_{i}^{t}}(s,a) - \frac{\rho_{\pi^{\ast},i}(s)}{Z_{\pi^{\ast}}(s)} Q^{\pi^{\theta^{t}}}_{i}(s,a) \right)^{2} \right] \nonumber \\
&\le 4 \mathbb{E}_{s \thicksim \rho_{\pi^{\theta^{t}},n}, a \thicksim \pi^{\theta^{t}}, i \thicksim \mathcal{C}} \left[ \left\vert Q^{\pi^{\theta^{t}}}_{i}(s,a) \right\vert^{2} \left\vert 1 - \frac{\rho_{\pi^{\ast},i}(s)}{Z_{\pi^{\ast}}(s)} \right\vert^{2} + \left\vert \frac{\rho_{\pi^{\ast},i}(s)}{Z_{\pi^{\ast}}(s)} \right\vert^{2} \left\vert Q^{w^{t}_{i}}(s,a) - Q_{i}^{\pi^{\theta^{t}}}(s,a) \right\vert^{2} \right]. \label{eq:policy_evaluation_error_fedrac1}
\end{alignat}
By Holder's inequality and (\ref{eq:policy_evaluation_error_fedrac1}), we have 
\begin{alignat}{1}
& \mathbb{E}_{s \thicksim \rho_{\pi^{\theta^{t}},n}, a \thicksim \pi^{\theta^{t}}} \left[ 2 \sum_{i} q_{i} \left( Q^{w_{i}^{t}}(s,a) - \frac{\rho_{\pi^{\ast},i}(s)}{Z_{\pi^{\ast}}(s)} Q^{\pi^{\theta^{t}}}_{i}(s,a) \right)^{2} \right] \nonumber \\
&\le 4 \mathbb{E}_{s \thicksim \rho_{\pi^{\theta^{t}},i}, a \thicksim \pi^{\theta^{t}}, i \thicksim \mathcal{C}} \left[ \left\vert \frac{\rho_{\pi^{\ast},n}(s)}{\rho_{\pi^{\ast},i}(s)} \right\vert^{2} \right]^{1/2} \cdot \left( \mathbb{E}_{s \thicksim \rho_{\pi^{\theta^{t}},i}, a \thicksim \pi^{\theta^{t}}, i \thicksim \mathcal{C}} \left[ \left\vert Q^{\pi^{\theta^{t}}}_{i}(s,a) \right\vert^{4} \left\vert 1 - \frac{\rho_{\pi^{\ast},i}(s)}{Z_{\pi^{\ast}}(s)} \right\vert^{4} \right]^{1/2} \right. \nonumber \\
& \quad \left. + \mathbb{E}_{s \thicksim \rho_{\pi^{\theta^{t}},i}, a \thicksim \pi^{\theta^{t}}, i \thicksim \mathcal{C}} \left[ \left\vert \frac{\rho_{\pi^{\ast},i}(s)}{Z_{\pi^{\ast}}(s)} \right\vert^{4} \left\vert Q^{w^{t}_{i}}(s,a) - Q_{i}^{\pi^{\theta^{t}}}(s,a) \right\vert^{4} \right]^{1/2} \right). \nonumber
\end{alignat}
By applying Holder's inequality to $\mathbb{E}_{s \thicksim \rho_{\pi^{\theta^{t}},i}, a \thicksim \pi^{\theta^{t}}, i \thicksim \mathcal{C}} \left[ \left\vert \frac{\rho_{\pi^{\ast},i}(s)}{Z_{\pi^{\ast}}(s)} \right\vert^{4} \left\vert Q^{w^{t}_{i}} - Q_{i}^{\pi^{\theta^{t}}} \right\vert^{4} \right]^{1/2}$, we have
\begin{alignat}{1}
& \mathbb{E}_{s \thicksim \rho_{\pi^{\theta^{t}},n}, a \thicksim \pi^{\theta^{t}}} \left[ 2 \sum_{i} q_{i} \left( Q^{w_{i}^{t}}(s,a) - \frac{\rho_{\pi^{\ast},i}(s)}{Z_{\pi^{\ast}}(s)} Q^{\pi^{\theta^{t}}}_{i}(s,a) \right)^{2} \right] \nonumber \\
&\le 4 \mathbb{E}_{s \thicksim \rho_{\pi^{\theta^{t}},i}, a \thicksim \pi^{\theta^{t}}, i \thicksim \mathcal{C}} \left[ \left\vert Q^{\pi^{\theta^{t}}}_{i}(s,a) \right\vert^{4} \left\vert 1 - \frac{\rho_{\pi^{\ast},i}(s)}{Z_{\pi^{\ast}}(s)} \right\vert^{4} \right]^{1/2} \cdot \mathbb{E}_{s \thicksim \rho_{\pi^{\theta^{t}},i}, a \thicksim \pi^{\theta^{t}}, i \thicksim \mathcal{C}} \left[ \left\vert \frac{\rho_{\pi^{\ast},n}(s)}{\rho_{\pi^{\ast},i}(s)} \right\vert^{2} \right]^{1/2} \nonumber \\
& \quad + 4 \mathbb{E}_{s \thicksim \rho_{\pi^{\theta^{t}},i}, a \thicksim \pi^{\theta^{t}}, i \thicksim \mathcal{C}} \left[ \left\vert Q^{w^{t}_{i}}(s,a) - Q_{i}^{\pi^{\theta^{t}}}(s,a) \right\vert^{8} \right]^{1/4} \cdot \mathbb{E}_{s \thicksim \rho_{\pi^{\theta^{t}},i}, a \thicksim \pi^{\theta^{t}}, i \thicksim \mathcal{C}} \left[ \left\vert \frac{\rho_{\pi^{\ast},i}(s)}{Z_{\pi^{\ast}}(s)} \right\vert^{8} \right]^{1/4} \nonumber \\
& \quad \cdot \mathbb{E}_{s \thicksim \rho_{\pi^{\theta^{t}},i}, a \thicksim \pi^{\theta^{t}}, i \thicksim \mathcal{C}} \left[ \left\vert \frac{\rho_{\pi^{\ast},n}(s)}{\rho_{\pi^{\ast},i}(s)} \right\vert^{2} \right]^{1/2}. \label{eq:policy_evaluation_error_fedrac2}
\end{alignat}
By substituting (\ref{eq:qet-qet0}) and (\ref{eq:policy_evaluation_error_fedrac2}) into (\ref{eq:policy_evaluation_error_fedrac0}), we have
\begin{alignat}{1}
& \mathbb{E}_{s \thicksim \rho_{\pi^{\theta^{t}},n}, a \thicksim \pi^{\theta^{t}}} \left[ \left( Q^{w^{t}}(s,a) - \sum_{i} \frac{q_{i}\rho_{\pi^{\ast},i}(s)}{Z_{\pi^{\ast}}(s)} Q^{\pi^{\theta^{t}}}_{i}(s,a) \right)^{2} \right] \nonumber \\
&\le 4 \mathbb{E}_{s \thicksim \rho_{\pi^{\theta^{t}},i}, a \thicksim \pi^{\theta^{t}}, i \thicksim \mathcal{C}} \left[ \left\vert Q^{\pi^{\theta^{t}}}_{i}(s,a) \right\vert^{4} \left\vert 1 - \frac{\rho_{\pi^{\ast},i}(s)}{Z_{\pi^{\ast}}(s)} \right\vert^{4} \right]^{1/2} \cdot \mathbb{E}_{s \thicksim \rho_{\pi^{\theta^{t}},i}, a \thicksim \pi^{\theta^{t}}, i \thicksim \mathcal{C}} \left[ \left\vert \frac{\rho_{\pi^{\ast},n}(s)}{\rho_{\pi^{\ast},i}(s)} \right\vert^{2} \right]^{1/2} \nonumber \\
& \quad + 4 \mathbb{E}_{s \thicksim \rho_{\pi^{\theta^{t}},i}, a \thicksim \pi^{\theta^{t}}, i \thicksim \mathcal{C}} \left[ \left\vert Q^{w^{t}_{i}}(s,a) - Q_{i}^{\pi^{\theta^{t}}}(s,a) \right\vert^{8} \right]^{1/4} \cdot \mathbb{E}_{s \thicksim \rho_{\pi^{\theta^{t}},i}, a \thicksim \pi^{\theta^{t}}, i \thicksim \mathcal{C}} \left[ \left\vert \frac{\rho_{\pi^{\ast},i}(s)}{Z_{\pi^{\ast}}(s)} \right\vert^{8} \right]^{1/4} \nonumber \\
& \quad \cdot \mathbb{E}_{s \thicksim \rho_{\pi^{\theta^{t}},i}, a \thicksim \pi^{\theta^{t}}, i \thicksim \mathcal{C}} \left[ \left\vert \frac{\rho_{\pi^{\ast},n}(s)}{\rho_{\pi^{\ast},i}(s)} \right\vert^{2} \right]^{1/2} + 8 R_{w}^{2}. \nonumber
\end{alignat}
Therefore,
\begin{alignat}{1}
& \mathbb{E}_{s \thicksim \rho_{\pi^{\theta^{t}},i}, a \thicksim \pi^{\theta^{t}}} \left[ \left( Q^{w^{t}_{i}}(s,a) - Q^{\pi^{\theta_{t}}}_{i}(s,a) \right)^{8} \right]^{1/4} \nonumber \\
&\le \frac{(\dot{\epsilon}_{n}^{t+1})^{2} - 8 R_{w}^{2}}{\mathbb{E}_{i \thicksim \mathcal{C},s \thicksim \rho_{\pi^{\theta^{t}},i}} \left[ \left\vert \frac{\rho_{\pi^{\ast},n}(s)}{\rho_{\pi^{\ast},i}(s)} \right\vert^{2} \right]^{1/2} 4 \mathbb{E}_{i \thicksim \mathcal{C},s \thicksim \rho_{\pi^{\theta^{t}},i}} \left[ \left\vert \frac{\rho_{\pi^{\ast},i}(s)}{Z_{\pi^{\ast}}(s)} \right\vert^{8} \right]^{1/4} } \nonumber \\
& \quad - \frac{\mathbb{E}_{i \thicksim \mathcal{C},s \thicksim \rho_{\pi^{\theta^{t}},i},a \thicksim \pi^{\theta^{t}}} \left[ \left\vert Q^{\pi^{\theta^{t}}}_{i}(s,a) \right\vert^{4} \left\vert 1 -\frac{\rho_{\pi^{\ast},i}(s)}{Z_{\pi^{\ast}}(s)} \right\vert^{4} \right]^{1/2} }{ \mathbb{E}_{i \thicksim \mathcal{C},s \thicksim \rho_{\pi^{\theta^{t}},i}} \left[ \left\vert \frac{\rho_{\pi^{\ast},i}(s)}{Z_{\pi^{\ast}}(s)} \right\vert^{8} \right]^{1/4} }, \forall i \in [N] \nonumber
\end{alignat}
indicates
\begin{alignat}{1}
\mathbb{E}_{s \thicksim \rho_{\pi^{\theta^{t}},n}, a \thicksim \pi^{\theta^{t}}} \left[ \left( Q^{w^{t}}(s,a) - \sum_{i} \frac{q_{i}\rho_{\pi^{\ast},i}(s)}{Z_{\pi^{\ast}}(s)} Q^{\pi^{\theta^{t}}}_{i}(s,a) \right)^{2} \right] \le \left( \dot{\epsilon}_{n}^{t+1} \right)^{2}. \nonumber
\end{alignat}
Now we can prove (\ref{eq:lemma:error_propagation}) on a case-by-case basis. We first prove the case for the baseline (\ref{eq:policy_evaluation}), and the result for FedRAC (\ref{eq:policy_evaluation_fedrac}) is similar. For concrete representation, we denote $\log \pi^{\theta^{t+1}_{n}}(\cdot \vert s) - \left( \beta_{t}^{-1} Q^{w^{t}}(s,\cdot) + \log \pi^{\theta^{t}}(\cdot \vert s) \right)$ by $y_{1}(n,t,s)$, denote $\log \pi^{\theta^{t+1}_{n}}(a \vert s) - \left( \beta_{t}^{-1} Q^{w^{t}}(s,a) + \log \pi^{\theta^{t}}(a \vert s) \right)$ by $y_{2}(n,t,s,a)$, denote $\beta_{t}^{-1} Q^{\pi^{\theta^{t}}}(s,\cdot) - \beta_{t}^{-1} Q^{w^{t}}(s,\cdot)$ by $y_{3}(t,s)$, and denote $\beta_{t}^{-1} Q^{\pi^{\theta^{t}}}(s,a) - \beta_{t}^{-1} Q^{w^{t}}(s,a)$ by $y_{4}(t,s,a)$. By definition and the triangle inequality, we have
\begin{alignat}{1}
& \left\vert \mathbb{E}_{s \thicksim \rho_{\pi^{\ast},n}} \left[ \left< \log \left( \pi^{\theta^{t+1}_{n}}(\cdot \vert s) / \pi^{t+1}_{n}(\cdot \vert s) \right), \pi^{\ast}(\cdot \vert s) - \pi^{\theta^{t}}(\cdot \vert s) \right> \right] \right\vert \nonumber \\
&= \left\vert \mathbb{E}_{s \thicksim \rho_{\pi^{\ast},n}} \left[ \left< \log \pi^{\theta^{t+1}_{n}}(\cdot \vert s) - \left( \beta_{t}^{-1} Q^{\pi^{\theta^{t}}}(s,\cdot) + \log \pi^{\theta^{t}}(\cdot \vert s) \right), \pi^{\ast}(\cdot \vert s) - \pi^{\theta^{t}}(\cdot \vert s) \right> \right] \right\vert \nonumber \\
&\le \left\vert \mathbb{E}_{s \thicksim \rho_{\pi^{\ast},n}} \left[ \left< y_{1}(n,t,s), \pi^{\ast}(\cdot \vert s) - \pi^{\theta^{t}}(\cdot \vert s) \right> \right] \right\vert + \left\vert \mathbb{E}_{s \thicksim \rho_{\pi^{\ast},n}} \left[ \left< y_{3}(t,s), \pi^{\ast}(\cdot \vert s) - \pi^{\theta^{t}}(\cdot \vert s) \right> \right] \right\vert. \label{eq:policy_evaluation_error_fedrac3}
\end{alignat}
First, we bound the first term on the RHS of (\ref{eq:policy_evaluation_error_fedrac3}). By the Cauchy-Schwarz inequality, we have
\begin{alignat}{1}
& \left\vert \mathbb{E}_{s \thicksim \rho_{\pi^{\ast},n}} \left[ \left< y_{1}(n,t,s), \pi^{\ast}(\cdot \vert s) - \pi^{\theta^{t}}(\cdot \vert s) \right> \right] \right\vert \nonumber \\
&= \left\vert \int_{\mathcal{S}} \left< y_{1}(n,t,s), \pi^{\theta^{t}}(\cdot \vert s) \rho_{\pi^{\theta^{t}},n}(s) \left( \frac{\pi^{\ast}(\cdot \vert s)}{\pi^{\theta^{t}}(\cdot \vert s)} - \frac{\pi^{\theta^{t}}(\cdot \vert s)}{\pi^{\theta^{t}}(\cdot \vert s)} \right) \right> \cdot \frac{\rho_{\pi^{\ast},n}(s)}{\rho_{\pi^{\theta^{t}},n}(s)} \mathrm{d} s \right\vert \nonumber \\
&= \left\vert \int_{\mathcal{S} \times \mathcal{A}} y_{2}(n,t,s,a) \cdot \left( \frac{\rho_{\pi^{\ast},n}(s) \pi^{\ast}(a \vert s)}{\rho_{\pi^{\theta^{t}},n}(s) \pi^{\theta^{t}}(a \vert s)} - \frac{\rho_{\pi^{\ast},n}(s)}{\rho_{\pi^{\theta^{t}},n}(s)} \right) \mathrm{d} \rho_{\pi^{\theta^{t}}}(s) \pi^{\theta^{t}}(a \vert s) \right\vert \nonumber \\
&\le \mathbb{E}_{s \thicksim \rho_{\pi^{\theta^{t},n}, a \thicksim \pi^{\theta^{t}}}} \left[ y_{2}^{2}(n,t,s,a) \right]^{1/2} \cdot \mathbb{E}_{s \thicksim \rho_{\pi^{\theta^{t},n}, a \thicksim \pi^{\theta^{t}}}} \left[ \left( \frac{\rho_{\pi^{\ast},n}(s) \pi^{\ast}(a \vert s)}{\rho_{\pi^{\theta^{t}},n}(s) \pi^{\theta^{t}}(a \vert s)} - \frac{\rho_{\pi^{\ast},n}(s)}{\rho_{\pi^{\theta^{t}},n}(s)} \right)^{2} \right]^{1/2}. \label{eq:policy_evaluation_error_fedrac4}
\end{alignat}
Next, we bound the second term on the RHS of (\ref{eq:policy_evaluation_error_fedrac3}). By the Cauchy-Schwarz inequality, we have
\begin{alignat}{1}
& \left\vert \mathbb{E}_{s \thicksim \rho_{\pi^{\ast},n}} \left[ \left< y_{3}(t,s), \pi^{\ast}(\cdot \vert s) - \pi^{\theta^{t}}(\cdot \vert s) \right> \right] \right\vert \nonumber \\
&= \left\vert \int_{\mathcal{S}} \left< y_{3}(t,s), \pi^{\theta^{t}}(\cdot \vert s) \rho_{\pi^{\theta^{t}},n}(s) \left( \frac{\pi^{\ast}(\cdot \vert s)}{\pi^{\theta^{t}}(\cdot \vert s)} - \frac{\pi^{\theta^{t}}(\cdot \vert s)}{\pi^{\theta^{t}}(\cdot \vert s)} \right) \right> \cdot \frac{\rho_{\pi^{\ast},n}(s)}{\rho_{\pi^{\theta^{t}},n}(s)} \mathrm{d} s \right\vert \nonumber \\
&= \left\vert \int_{\mathcal{S} \times \mathcal{A}} y_{4}(t,s,a) \cdot \left( \frac{\rho_{\pi^{\ast},n}(s) \pi^{\ast}(\cdot \vert s)}{\rho_{\pi^{\theta^{t}},n}(s) \pi^{\theta^{t}}(\cdot \vert s)} - \frac{\rho_{\pi^{\ast},n}(s)}{\rho_{\pi^{\theta^{t}},n}(s)} \right) \mathrm{d} \rho_{\pi^{\theta^{t}}}(s) \pi^{\theta^{t}}(a \vert s) \right\vert \nonumber \\
&\le \mathbb{E}_{s \thicksim \rho_{\pi^{\theta^{t},n}, a \thicksim \pi^{\theta^{t}}}} \left[ y_{4}^{2}(t,s,a) \right]^{1/2} \cdot \mathbb{E}_{s \thicksim \rho_{\pi^{\theta^{t},n}, a \thicksim \pi^{\theta^{t}}}} \left[ \left( \frac{\rho_{\pi^{\ast},n}(s) \pi^{\ast}(a \vert s)}{\rho_{\pi^{\theta^{t}},n}(s) \pi^{\theta^{t}}(a \vert s)} - \frac{\rho_{\pi^{\ast},n}(s)}{\rho_{\pi^{\theta^{t}},n}(s)} \right)^{2} \right]^{1/2}. \label{eq:policy_evaluation_error_fedrac5}
\end{alignat}
By substituting (\ref{eq:policy_evaluation_error_fedrac4}) and (\ref{eq:policy_evaluation_error_fedrac5}) into (\ref{eq:policy_evaluation_error_fedrac3}), we can obtain
\begin{alignat}{1}
& \left\vert \mathbb{E}_{s \thicksim \rho_{\pi^{\ast},n}} \left[ \left< \log \left( \pi^{\theta^{t+1}_{n}}(\cdot \vert s) / \pi^{t+1}_{n}(\cdot \vert s) \right), \pi^{\ast}(\cdot \vert s) - \pi^{\theta^{t}}(\cdot \vert s) \right> \right] \right\vert \nonumber \\
&\le \sup \left\vert c_{t+1}^{-1}(s,a) \right\vert \epsilon^{t+1}_{n} \psi^{t}_{n} + \beta_{t+1}^{-1} \dot{\epsilon}^{t+1}_{n} \psi^{t}_{n}, \nonumber
\end{alignat}
where $c_{t}(s,a) = \tau_{t}$ for Softmax policy and $c_{t}(s,a) = \upsilon_{t}(s,a)$ for Gaussian policy. The proof for FedRAC (\ref{eq:policy_evaluation_fedrac}) can be obtained by replacing $Q^{\pi^{\theta^{t}}}$ with $\sum_{i} \frac{q_{i}\rho_{\pi^{\ast},i}(s)}{Z_{\pi^{\ast}}(s)} Q^{\pi^{\theta^{t}}}_{i}$ in (\ref{eq:policy_evaluation_error_fedrac3}).
\end{proof}

\section{Proof of Lemma \ref{lemma:stepwise_energy_difference}\label{proof:lemma:stepwise_energy_difference}}

\begin{proof}
We first prove the case with the Softmax policy. For all $n \in [N]$, we have
\begin{alignat}{1}
& \mathbb{E}_{\text{init}, s \thicksim \rho_{\pi^{\ast},n}} \left[ \left\Vert \log \pi^{\theta^{t+1}_{n}}(\cdot \vert s) / \pi^{\theta^{t}}(\cdot \vert s) \right\Vert_{\infty}^{2} \right] \nonumber \\
&= \mathbb{E}_{\text{init}, s \thicksim \rho_{\pi^{\ast},n}} \left[ \max_{a \in \mathcal{A}} \left( \tau_{t+1}^{-1} f_{\theta^{t+1}_{n}}(s,a) - \tau_{t}^{-1} f_{\theta^{t}}(s,a) + \log \frac{\sum_{a^{\prime}} \exp \left( \tau^{-1}_{t+1} f_{\theta_{n}^{t+1}}(s,a^{\prime}) \right)}{\sum_{a^{\prime}} \exp \left( \tau^{-1}_{t} f_{\theta^{t}}(s,a^{\prime}) \right)} \right)^{2} \right] \nonumber \\
&\le 2 \mathbb{E}_{\text{init}, s \thicksim \rho_{\pi^{\ast},n}} \left[ \max_{a \in \mathcal{A}} \left( \tau_{t+1}^{-1} f_{\theta^{t+1}_{n}}(s,a) - \tau_{t}^{-1} f_{\theta^{t}}(s,a) \right)^{2} \right] \nonumber \\
& \quad + 2 \mathbb{E}_{\text{init}, s \thicksim \rho_{\pi^{\ast},n}} \left[ \left( \log \sum_{a^{\prime}} \exp \left( \tau^{-1}_{t+1} f_{\theta_{n}^{t+1}}(s,a^{\prime}) \right) - \log \sum_{a^{\prime}} \exp \left( \tau^{-1}_{t} f_{\theta^{t}}(s,a^{\prime}) \right) \right)^{2} \right], \label{eq:proof:lemma:stepwise_energy_difference0}
\end{alignat}
due to the fact that $(a + b)^{2} \le 2a^{2} + 2b^{2}$. The first term on the RHS of (\ref{eq:proof:lemma:stepwise_energy_difference0}) can be bounded as
\begin{alignat}{1}
& 2 \mathbb{E}_{\text{init}, s \thicksim \rho_{\pi^{\ast},n}} \left[ \max_{a \in \mathcal{A}} \left( \tau_{t+1}^{-1} f_{\theta^{t+1}_{n}}(s,a) - \tau_{t}^{-1} f_{\theta^{t}}(s,a) \right)^{2} \right] \nonumber \\
&\le 2 \mathbb{E}_{\text{init}, s \thicksim \rho_{\pi^{\ast},n}} \left[ \max_{a \in \mathcal{A}} \left( \vert \tau_{t+1}^{-1} \vert \vert f_{\theta^{t+1}_{n}}(s,a) - f_{\theta^{t}}(s,a) \vert + \vert f_{\theta^{t}}(s,a) \vert \vert \tau_{t+1}^{-1} - \tau_{t}^{-1} \vert \right)^{2} \right] \nonumber \\
&\le 2 \mathbb{E}_{\text{init}, s \thicksim \rho_{\pi^{\ast},n}} \left[ \max_{a \in \mathcal{A}} \left( 2 \vert \tau_{t+1}^{-1} \vert^{2} \vert f_{\theta^{t+1}_{n}}(s,a) - f_{\theta^{t}}(s,a) \vert^{2} + 2 \vert f_{\theta^{t}}(s,a) \vert^{2} \vert \tau_{t+1}^{-1} - \tau_{t}^{-1} \vert^{2} \right) \right] \nonumber \\
&\le 2 \mathbb{E}_{\text{init}, s \thicksim \rho_{\pi^{\ast},n}} \left[ 8 ( \tau_{t+1}^{-1} )^{2} R_{\theta}^{2} + 4 \left( \max_{a \in \mathcal{A}} f_{\theta^{0}}^{2}(s,a) + R_{\theta}^{2} \right) ( \tau_{t+1}^{-1} - \tau_{t}^{-1} )^{2} \right], \nonumber
\end{alignat}
where $a_{\max}(s) = \arg\max_{a \in \mathcal{A}} \tau_{t+1}^{-1} f_{\theta^{t+1}_{n}}(s,a) - \tau_{t}^{-1} f_{\theta^{t}}(s,a)$, the first and the second inequality follow from the fact that $ab - cd = a(b-d) + d(a-c)$ and $(a + b)^{2} \le 2a^{2} + 2b^{2}$, respectively. The second inequality follows from Lemma \ref{lemma:bounded_by_r} and the fact that $f^{2}_{\theta^{t}}(s,a) \le 2 f^{2}_{\theta^{0}}(s,a) +2 R_{\theta}^{2}$, which is due to the 1-Lipschitz continuity of $f_{\theta}$ w.r.t. $\theta$.

It remains to bound the second term on the RHS of (\ref{eq:proof:lemma:stepwise_energy_difference0}). We consider states $s$ with $\log \sum_{a^{\prime}} \exp \left( \tau^{-1}_{t+1} f_{\theta_{n}^{t+1}}(s,a^{\prime}) \right) \ge \log \sum_{a^{\prime}} \exp \left( \tau^{-1}_{t} f_{\theta^{t}}(s,a^{\prime}) \right)$ and the case for the other side is similar. By the log-sum inequality, the log-sum-exp trick, and the fact that $ab - cd = a(b-d) + d(a-c)$, we can obtain
\begin{alignat}{1}
& 2 \mathbb{E}_{\text{init}} \left[ \left( \log \sum_{a^{\prime}} \exp \left( \tau^{-1}_{t+1} f_{\theta_{n}^{t+1}}(s,a^{\prime}) \right) - \log \sum_{a^{\prime}} \exp \left( \tau^{-1}_{t} f_{\theta^{t}}(s,a^{\prime}) \right) \right)^{2} \right] \nonumber \\
&\le 2 \mathbb{E}_{\text{init}} \left[ \left( \max_{a \in \mathcal{A}} \tau^{-1}_{t+1} f_{\theta_{n}^{t+1}}(s,a) + \log \left\vert \mathcal{A} \right\vert - \frac{1}{\left\vert \mathcal{A} \right\vert} \sum_{a^{\prime}} \tau^{-1}_{t} f_{\theta^{t}}(s,a^{\prime}) - \log \left\vert \mathcal{A} \right\vert \right)^{2} \right] \nonumber \\ 
&\le 2 \mathbb{E}_{\text{init}} \left[ \left( \tau^{-1}_{t+1} \left( \max_{a \in \mathcal{A}} f_{\theta_{n}^{t+1}}(s,a) - \frac{1}{\left\vert \mathcal{A} \right\vert} \sum_{a^{\prime}} f_{\theta^{t}}(s,a^{\prime}) \right) + \frac{1}{\left\vert \mathcal{A} \right\vert} \sum_{a^{\prime}} f_{\theta^{t}}(s,a^{\prime}) \left( \tau^{-1}_{t+1} - \tau^{-1}_{t} \right) \right)^{2} \right] \nonumber \\
&\le 4 \mathbb{E}_{\text{init}} \left[ \left( \tau^{-1}_{t+1} \left( \max_{a \in \mathcal{A}} f_{\theta_{n}^{t+1}}(s,a) - \sum_{a^{\prime}} \frac{f_{\theta^{t}}(s,a^{\prime})}{\left\vert \mathcal{A} \right\vert} \right) \right)^{2} + \left( \sum_{a^{\prime}} \frac{f_{\theta^{t}}(s,a^{\prime})}{\left\vert \mathcal{A} \right\vert} \left( \tau^{-1}_{t+1} - \tau^{-1}_{t} \right) \right)^{2} \right]. \nonumber
\end{alignat}
By Lemma \ref{lemma:changing_action_error} and the fact that $f^{2}_{\theta^{t}}(s,a) \le 2 f^{2}_{\theta^{0}}(s,a) +2 R_{\theta}^{2}$, we can obtain
\begin{alignat}{1}
& 2 \mathbb{E}_{\text{init}} \left[ \left( \log \sum_{a^{\prime}} \exp \left( \tau^{-1}_{t+1} f_{\theta_{n}^{t+1}}(s,a^{\prime}) \right) - \log \sum_{a^{\prime}} \exp \left( \tau^{-1}_{t} f_{\theta^{t}}(s,a^{\prime}) \right) \right)^{2} \right] \nonumber \\
&\le 4 \mathbb{E}_{\text{init}} \left[ \max_{a^{\prime} \in \mathcal{A}} \left( \tau^{-1}_{t+1} \left( \max_{a \in \mathcal{A}} f_{\theta_{n}^{t+1}}(s,a) - f_{\theta^{t}}(s,a^{\prime}) \right) \right)^{2} + \left( \max_{a^{\prime} \in \mathcal{A}} f_{\theta^{t}}(s,a^{\prime}) \left( \tau^{-1}_{t+1} - \tau^{-1}_{t} \right) \right)^{2} \right] \nonumber \\
&\le 24 ( \tau_{t+1}^{-1} )^{2} R_{\theta}^{2} + 48 ( \tau_{t+1}^{-1} )^{2} + \mathbb{E}_{\text{init}} \left[ 8 \left( \max_{a \in \mathcal{A}} f_{\theta^{0}}^{2}(s,a) + R_{\theta}^{2} \right) ( \tau_{t+1}^{-1} - \tau_{t}^{-1} )^{2} \right]. \label{eq:proof:lemma:stepwise_energy_difference1}
\end{alignat}
The upper bound for states $s$ with $\log \sum_{a^{\prime}} \exp \left( \tau^{-1}_{t+1} f_{\theta_{n}^{t+1}}(s,a^{\prime}) \right) \le \log \sum_{a^{\prime}} \exp \left( \tau^{-1}_{t} f_{\theta^{t}}(s,a^{\prime}) \right)$ is the same. Therefore, we have
\begin{alignat}{1}
&\mathbb{E}_{\text{init}, s \thicksim \rho_{\pi^{\ast},n}} \left[ \left\Vert \log \pi^{\theta^{t+1}_{n}}(\cdot \vert s) / \pi^{\theta^{t}}(\cdot \vert s) \right\Vert_{\infty}^{2} \right] \nonumber \\
&\le \mathbb{E}_{\text{init}, s \thicksim \rho_{\pi^{\ast},n}} \left[ ( \tau_{t+1}^{-1} )^{2} \left( 40 R_{\theta}^{2} + 48 \right) + 16 \left( \max_{a \in \mathcal{A}} f_{\theta^{0}}^{2}(s,a) + R_{\theta}^{2} \right) ( \tau_{t+1}^{-1} - \tau_{t}^{-1} )^{2} \right]. \nonumber
\end{alignat}

Next, we prove the case with the Gaussian policy. By Lemma \ref{lemma:taylor_expansion}, we have
\begin{alignat}{1}
& \mathbb{E}_{\text{init}, s \thicksim \rho_{\pi^{\ast},n}} \left[ \max_{a \in \mathcal{A}} \left( - \frac{(a - f_{\theta_{n}^{t+1}}(s,a))^{2}}{2 \sigma_{t+1}^{2}} - \log \sigma_{t+1} + \frac{(a - f_{\theta^{t}}(s,a))^{2}}{2 \sigma_{t}^{2}} + \log \sigma_{t} \right)^{2} \right] \nonumber \\
&\le \mathbb{E}_{\text{init}, s \thicksim \rho_{\pi^{\ast},n}} \left[ \max_{a \in \mathcal{A}} \left( 3 \left( \upsilon_{t+1}^{-1}(s,a) f_{\theta^{t+1}_{n}}(s,a) - \upsilon_{t}^{-1}(s,a) f_{\theta^{t}}(s,a) \right)^{2} \right. \right. \nonumber \\
& \quad \left. \left. + 3 \left( \frac{1}{2 \sigma_{t}^{2}} ( f_{\theta^{t}}(s,a) - f_{\theta^{0}}(s,a) )^{2} - \frac{1}{2 \sigma_{t+1}^{2}} ( f_{\theta^{t+1}_{n}}(s,a) - f_{\theta^{0}}(s,a) )^{2} \right)^{2} \right. \right. \nonumber \\
& \quad \left. \left. + 3 \left( \log \frac{\sigma_{t}}{\sigma_{t+1}} + \frac{1}{2} \left( \frac{1}{\sigma_{t+1}^{2}} - \frac{1}{\sigma_{t}^{2}} \right) ( f_{\theta^{0}}^{2}(s,a) - a^{2} ) \right)^{2} \right) \right]. \nonumber
\end{alignat}
By Lemma \ref{lemma:bounded_by_r} and the discussion of (\ref{eq:proof:lemma:stepwise_energy_difference0}), we have
\begin{alignat}{1}
& \mathbb{E}_{\text{init}, s \thicksim \rho_{\pi^{\ast},n}} \left[ \max_{a \in \mathcal{A}} \left( - \frac{(a - f_{\theta_{n}^{t+1}}(s,a))^{2}}{2 \sigma_{t+1}^{2}} - \log \sigma_{t+1} + \frac{(a - f_{\theta^{t}}(s,a))^{2}}{2 \sigma_{t}^{2}} + \log \sigma_{t} \right)^{2} \right] \nonumber \\
&\le 3 \mathbb{E}_{\text{init}, s \thicksim \rho_{\pi^{\ast},n}} \left[ \max_{a \in \mathcal{A}} \left( \upsilon_{t+1}^{-1}(s,a) f_{\theta^{t+1}_{n}}(s,a) - \upsilon_{t}^{-1}(s,a) f_{\theta^{t}}(s,a) \right)^{2} \right] \nonumber \\
& \quad + 3 \mathbb{E}_{\text{init}, s \thicksim \rho_{\pi^{\ast},n}} \left[ \max_{a \in \mathcal{A}} \left( \log \frac{\sigma_{t}}{\sigma_{t+1}} + \frac{1}{2} \left( \frac{1}{\sigma_{t+1}^{2}} - \frac{1}{\sigma_{t}^{2}} \right) ( f_{\theta^{0}}^{2}(s,a) - a^{2} ) \right)^{2} \right] + \frac{3 R_{\theta}^{4}}{4 \sigma_{t+1}^{4}} \nonumber \\
&\le \mathbb{E}_{\text{init}, s \thicksim \rho_{\pi^{\ast},n}} \left[ \max_{a \in \mathcal{A}} \left( 24 ( \upsilon_{t+1}^{-1}(s, a) )^{2} R_{\theta}^{2} + 12 \left( f_{\theta^{0}}^{2}(s,a) + R_{\theta}^{2} \right) ( \upsilon_{t+1}^{-1}(s, a) - \upsilon_{t}^{-1}(s, a) )^{2} \right) \right] \nonumber \\
& \quad + 3 \mathbb{E}_{\text{init}, s \thicksim \rho_{\pi^{\ast},n}} \left[ \max_{a \in \mathcal{A}} \left( \log \frac{\sigma_{t}}{\sigma_{t+1}} + \frac{1}{2} \left( \frac{1}{\sigma_{t+1}^{2}} - \frac{1}{\sigma_{t}^{2}} \right) ( f_{\theta^{0}}^{2}(s,a) - a^{2} ) \right)^{2} \right] + \frac{3 R_{\theta}^{4}}{4 \sigma_{t+1}^{4}}. \nonumber
\end{alignat}
\end{proof}

\section{Proof of Lemma \ref{lemma:aggregation_error}: Aggregation Error in Neural Network\label{proof:lemma:aggregation_error}}

The following lemma characterizes the aggregation error in the network output.
\begin{lemma}
\label{lemma:aggregation_error}
(Aggregation Error in Neural Network). Suppose that Assumption \ref{assumption:concentrability_coefficient} holds,
we have
\begin{alignat}{1}
& \mathbb{E}_{\text{init}, n \thicksim \mathcal{C}, s \thicksim \rho_{\pi^{\ast},n}, a \thicksim \pi^{\ast}} \left[ \left\vert c_{t+1}^{-1}(s,a) f_{\theta_{n}^{t+1}}(s,a) - c_{t+1}^{-1}(s,a) f_{\theta^{t+1}}(s,a) \right\vert \right] \nonumber \\
&\le p \beta_{t}^{-1} \kappa + 2 \mathbb{E}_{n \thicksim \mathcal{C}} \left[ \sup \left\vert c_{t+1}^{-1}(s,a) \right\vert e^{t+1}_{n} \xi^{t+1}_{n} \right] + \mathcal{O} \left( \sup \left\vert c_{t+1}^{-1}(s,a) \right\vert R_{\theta}^{6/5} m^{-1/10} \hat{R}_{\theta}^{2/5} \right), \nonumber
\end{alignat}
where $c_{t+1}(s,a) = \tau_{t+1}, e_{n}^{t+1} = \epsilon_{n}^{t+1}$ for Softmax policies, $c_{t+1}(s,a) = \upsilon_{t+1}(s,a), e_{n}^{t+1} = \tilde{\epsilon}_{n}^{t+1}$ for Gaussian policies, $p = 1$ for the baseline, and $p = 0$ for FedRAC.
\end{lemma}
\begin{proof}
We begin by deriving the common component for the Softmax policy and subsequently extend this derivation to the baseline and FedRAC individually. By the triangle inequality and Lemma \ref{lemma:linearization_error}, we have
\begin{alignat}{1}
& \mathbb{E}_{\text{init}, n \thicksim \mathcal{C}, s \thicksim \rho_{\pi^{\ast},n}, a \thicksim \pi^{\ast}} \left[ \left\vert \tau_{t+1}^{-1} f_{\theta_{n}^{t+1}}(s,a) - \tau_{t+1}^{-1} f_{\theta^{t+1}}(s,a) \right\vert \right] \nonumber \\
&\le \mathbb{E}_{\text{init}, n \thicksim \mathcal{C}, s \thicksim \rho_{\pi^{\ast},n}, a \thicksim \pi^{\ast}} \left[ \left\vert \tau_{t+1}^{-1} f_{\theta_{n}^{t+1}}(s,a) - \sum_{i} q_{i} \tau_{t+1}^{-1} f_{\theta^{t+1}_{i}}(s,a) \right\vert \right. 
\nonumber \\
& \quad \left. + \left\vert \sum_{i} q_{i} \tau_{t+1}^{-1} f_{\theta^{t+1}_{i}}(s,a) - \sum_{i} q_{i} \tau_{t+1}^{-1} f_{\theta^{t+1}_{i}}^{0}(s,a) \right\vert + \left\vert \tau_{t+1}^{-1} f_{\theta^{t+1}}^{0}(s,a) - \tau_{t+1}^{-1} f_{\theta^{t+1}}(s,a) \right\vert \right] \nonumber \\
&\le \mathbb{E}_{\text{init}, n \thicksim \mathcal{C}, s \thicksim \rho_{\pi^{\ast},n}, a \thicksim \pi^{\ast}} \left[ \left\vert \tau_{t+1}^{-1} f_{\theta_{n}^{t+1}}(s,a) - \sum_{i} q_{i} \tau_{t+1}^{-1} f_{\theta^{t+1}_{i}}(s,a) \right\vert \right] + \mathcal{O} \left( \tau_{t+1}^{-1} R_{\theta}^{6/5} m^{-1/10} \hat{R}_{\theta}^{2/5} \right), \nonumber
\end{alignat}
where we have decomposed the global parameter into local components by linearizing $f_{\theta}(s,a)$. Next, we split the discussion into two directions, one for the baseline with action-value function $Q^{w^{t}_{n}}$ and the other for FedRAC with action-value function $Q^{w^{t}}$. For the baseline, we have
\begin{alignat}{1}
& \mathbb{E}_{\text{init}, n \thicksim \mathcal{C}, s \thicksim \rho_{\pi^{\ast},n}, a \thicksim \pi^{\ast}} \left[ \left\vert \tau_{t+1}^{-1} f_{\theta_{n}^{t+1}}(s,a) - \tau_{t+1}^{-1} f_{\theta^{t+1}}(s,a) \right\vert \right] \nonumber \\
&\le \mathbb{E}_{\text{init}, n \thicksim \mathcal{C}, s \thicksim \rho_{\pi^{\ast},n}, a \thicksim \pi^{\ast}} \left[ \left\vert \tau_{t+1}^{-1} f_{\theta_{n}^{t+1}}(s,a) - \left( \beta_{t}^{-1} Q^{w^{t}_{n}}(s,a) + \tau_{t}^{-1} f_{\theta^{t}}(s,a) \right) \right\vert \right. \nonumber \\
& \quad \left. + \sum_{i} q_{i} \left\vert \left( \beta_{t}^{-1} Q^{w^{t}_{n}}(s,a) + \tau_{t}^{-1} f_{\theta^{t}}(s,a) \right) - \left( \beta_{t}^{-1} Q^{w^{t}_{i}}(s,a) + \tau_{t}^{-1} f_{\theta^{t}}(s,a) \right) \right\vert \right. \nonumber \\
& \quad \left. + \sum_{i} q_{i} \left\vert \tau_{t+1}^{-1} f_{\theta_{i}^{t+1}}(s,a) - \left( \beta_{t}^{-1} Q^{w^{t}_{i}}(s,a) + \tau_{t}^{-1} f_{\theta^{t}}(s,a) \right) \right\vert \right] + \mathcal{O} \left( \tau_{t+1}^{-1} R_{\theta}^{6/5} m^{-1/10} \hat{R}_{\theta}^{2/5} \right) \nonumber \\
&\le 2 \mathbb{E}_{n \thicksim \mathcal{C}} \left[ \tau_{t+1}^{-1} \epsilon^{t+1}_{n} \xi^{t+1}_{n} \right] + \beta_{t}^{-1} \mathbb{E}_{n \thicksim \mathcal{C}, s \thicksim \rho_{\pi^{\ast},n}, a \thicksim \pi^{\ast}} \left[ \sum_{i} q_{i} \left\vert Q^{w^{t}_{n}}(s,a) - Q^{w^{t}_{i}}(s,a) \right\vert \right] \nonumber \\
& \quad + \mathcal{O} \left( \tau_{t+1}^{-1} R_{\theta}^{6/5} m^{-1/10} \hat{R}_{\theta}^{2/5} \right). \nonumber
\end{alignat}
For FedRAC, we have
\begin{alignat}{1}
& \mathbb{E}_{\text{init}, n \thicksim \mathcal{C}, s \thicksim \rho_{\pi^{\ast},n}, a \thicksim \pi^{\ast}} \left[ \left\vert \tau_{t+1}^{-1} f_{\theta_{n}^{t+1}}(s,a) - \tau_{t+1}^{-1} f_{\theta^{t+1}}(s,a) \right\vert \right] \nonumber \\
&\le \mathbb{E}_{\text{init}, n \thicksim \mathcal{C}, s \thicksim \rho_{\pi^{\ast},n}, a \thicksim \pi^{\ast}} \left[ \left\vert \tau_{t+1}^{-1} f_{\theta_{n}^{t+1}}(s,a) - \left( \beta_{t}^{-1} Q^{w^{t}}(s,a) + \tau_{t}^{-1} f_{\theta^{t}}(s,a) \right) \right\vert \right. \nonumber \\
& \quad \left. + \sum_{i} q_{i} \left\vert \left( \beta_{t}^{-1} Q^{w^{t}}(s,a) + \tau_{t}^{-1} f_{\theta^{t}}(s,a) \right) - \left( \beta_{t}^{-1} Q^{w^{t}}(s,a) + \tau_{t}^{-1} f_{\theta^{t}}(s,a) \right) \right\vert \right. \nonumber \\
& \quad \left. + \sum_{i} q_{i} \left\vert \tau_{t+1}^{-1} f_{\theta_{i}^{t+1}}(s,a) - \left( \beta_{t}^{-1} Q^{w^{t}}(s,a) + \tau_{t}^{-1} f_{\theta^{t}}(s,a) \right) \right\vert \right] + \mathcal{O} \left( \tau_{t+1}^{-1} R_{\theta}^{6/5} m^{-1/10} \hat{R}_{\theta}^{2/5} \right) \nonumber \\
&\le 2 \mathbb{E}_{n \thicksim \mathcal{C}} \left[ \tau_{t+1}^{-1} \epsilon^{t+1}_{n} \xi^{t+1}_{n} \right] + \mathcal{O} \left( \tau_{t+1}^{-1} R_{\theta}^{6/5} m^{-1/10} \hat{R}_{\theta}^{2/5} \right). \nonumber
\end{alignat}
Similarly, we begin by deriving the common component for the Gaussian policy and subsequently extend this derivation to the baseline and FedRAC individually. By the triangle inequality and Lemma \ref{lemma:linearization_error}, we have
\begin{alignat}{1}
& \mathbb{E}_{\text{init}, n \thicksim \mathcal{C}, s \thicksim \rho_{\pi^{\ast},n}, a \thicksim \pi^{\ast}} \left[ \left\vert \upsilon_{t+1}^{-1}(s,a) f_{\theta_{n}^{t+1}}(s,a) - \upsilon_{t+1}^{-1}(s,a) f_{\theta^{t+1}}(s,a) \right\vert \right] \nonumber \\
&\le \mathbb{E}_{\text{init}, n \thicksim \mathcal{C}, s \thicksim \rho_{\pi^{\ast},n}, a \thicksim \pi^{\ast}} \left[ \left\vert \frac{f_{\theta_{n}^{t+1}}(s,a)}{\upsilon_{t+1}(s,a)} - \sum_{i} q_{i} \frac{f_{\theta^{t+1}_{i}}(s,a)}{\upsilon_{t+1}(s,a)} \right\vert \right. \nonumber \\
& \quad + \left. \left\vert \sum_{i} q_{i} \frac{f_{\theta^{t+1}_{i}}(s,a)}{\upsilon_{t+1}(s,a)} - \sum_{i} q_{i} \frac{f_{\theta^{t+1}_{i}}^{0}(s,a)}{\upsilon_{t+1}(s,a)} \right\vert + \left\vert \frac{f_{\theta^{t+1}}^{0}(s,a)}{\upsilon_{t+1}(s,a)} - \frac{f_{\theta^{t+1}}(s,a)}{\upsilon_{t+1}(s,a)} \right\vert \right] \nonumber \\
&\le \mathbb{E}_{\text{init}, n \thicksim \mathcal{C}, s \thicksim \rho_{\pi^{\ast},n}, a \thicksim \pi^{\ast}} \left[ \left\vert \frac{f_{\theta_{n}^{t+1}}(s,a)}{\upsilon_{t+1}(s,a)} - \sum_{i} q_{i} \frac{f_{\theta^{t+1}_{i}}(s,a)}{\upsilon_{t+1}(s,a)} \right\vert \right] + \mathcal{O} \left( \sup \left\vert \upsilon_{t+1}^{-1}(s,a) \right\vert R_{\theta}^{6/5} m^{-1/10} \hat{R}_{\theta}^{2/5} \right), \nonumber
\end{alignat}
where we have decomposed the global parameter into local components by linearizing $f_{\theta}(s,a)$ as well. Next, we split the discussion into two directions, one for the baseline with action-value function $Q^{w^{t}_{n}}$ and the other for FedRAC with action-value function $Q^{w^{t}}$. For the baseline, we have
\begin{alignat}{1}
& \mathbb{E}_{\text{init}, n \thicksim \mathcal{C}, s \thicksim \rho_{\pi^{\ast},n}, a \thicksim \pi^{\ast}} \left[ \left\vert \upsilon_{t+1}^{-1}(s,a) f_{\theta_{n}^{t+1}}(s,a) - \upsilon_{t+1}^{-1}(s,a) f_{\theta^{t+1}}(s,a) \right\vert \right] \nonumber \\
&\le \mathbb{E}_{\text{init}, n \thicksim \mathcal{C}, s \thicksim \rho_{\pi^{\ast},n}, a \thicksim \pi^{\ast}} \left[ \left\vert \upsilon_{t+1}^{-1}(s,a) f_{\theta_{n}^{t+1}}(s,a) - \left( \beta_{t}^{-1} Q^{w^{t}_{n}}(s,a) + \upsilon_{t}(s,a)^{-1} f_{\theta^{t}}(s,a) \right) \right\vert \right. \nonumber \\
& \quad \left. + \sum_{i} q_{i} \left\vert \left( \beta_{t}^{-1} Q^{w^{t}_{n}}(s,a) + \upsilon_{t}^{-1}(s,a) f_{\theta^{t}}(s,a) \right) - \left( \beta_{t}^{-1} Q^{w^{t}_{i}}(s,a) + \upsilon_{t}^{-1}(s,a) f_{\theta^{t}}(s,a) \right) \right\vert \right. \nonumber \\
& \quad \left. + \sum_{i} q_{i} \left\vert \frac{f_{\theta_{i}^{t+1}}(s,a)}{\upsilon_{t+1}(s,a)} - \left( \frac{Q^{w^{t}_{i}}(s,a)}{\beta_{t}} + 
\frac{f_{\theta^{t}}(s,a)}{\upsilon_{t}(s,a)} \right) \right\vert \right] + \mathcal{O} \left( \sup \left\vert \upsilon_{t+1}^{-1}(s,a) \right\vert R_{\theta}^{6/5} m^{-1/10} \hat{R}_{\theta}^{2/5} \right) \nonumber \\
&\le 2 \mathbb{E}_{n \thicksim \mathcal{C}} \left[ \sup \left\vert \upsilon_{t+1}^{-1}(s,a) \right\vert \tilde{\epsilon}_{n}^{t+1} \xi^{t+1}_{n} \right] + \beta_{t}^{-1} \mathbb{E}_{n \thicksim \mathcal{C}, s \thicksim \rho_{\pi^{\ast},n}, a \thicksim \pi^{\ast}} \left[ \sum_{i} q_{i} \left\vert Q^{w^{t}_{n}}(s,a) - Q^{w^{t}_{i}}(s,a) \right\vert \right] \nonumber \\
& \quad + \mathcal{O} \left( \sup \left\vert \upsilon_{t+1}^{-1}(s,a) \right\vert R_{\theta}^{6/5} m^{-1/10} \hat{R}_{\theta}^{2/5}\right). \nonumber
\end{alignat}
For FedRAC, we have
\begin{alignat}{1}
& \mathbb{E}_{\text{init}, n \thicksim \mathcal{C}, s \thicksim \rho_{\pi^{\ast},n}, a \thicksim \pi^{\ast}} \left[ \left\vert \upsilon_{t+1}^{-1}(s,a) f_{\theta_{n}^{t+1}}(s,a) - \upsilon_{t+1}^{-1}(s,a) f_{\theta^{t+1}}(s,a) \right\vert \right] \nonumber \\
&\le \mathbb{E}_{\text{init}, n \thicksim \mathcal{C}, s \thicksim \rho_{\pi^{\ast},n}, a \thicksim \pi^{\ast}} \left[ \left\vert \upsilon_{t+1}^{-1}(s,a) f_{\theta_{n}^{t+1}}(s,a) - \left( \beta_{t}^{-1} Q^{w^{t}}(s,a) + \upsilon_{t}^{-1}(s,a) f_{\theta^{t}}(s,a) \right) \right\vert \right. \nonumber \\
& \quad \left. + \sum_{i} q_{i} \left\vert \frac{f_{\theta_{i}^{t+1}}(s,a)}{\upsilon_{t+1}(s,a)} - \left( \frac{Q^{w^{t}}(s,a)}{\beta_{t}} + \frac{f_{\theta^{t}}(s,a)}{\upsilon_{t}(s,a)} \right) \right\vert \right] + \mathcal{O} \left( \sup \left\vert \upsilon_{t+1}^{-1}(s,a) \right\vert R_{\theta}^{6/5} m^{-1/10} \hat{R}_{\theta}^{2/5} \right) \nonumber \\
&\le 2 \mathbb{E}_{n \thicksim \mathcal{C}} \left[ \sup \left\vert \upsilon_{t+1}^{-1}(s,a) \right\vert \tilde{\epsilon}_{n}^{t+1} \xi^{t+1}_{n} \right] + \mathcal{O} \left( \sup \left\vert \upsilon_{t+1}^{-1}(s,a) \right\vert R_{\theta}^{6/5} m^{-1/10} \hat{R}_{\theta}^{2/5} \right). \nonumber
\end{alignat}
\end{proof}

\begin{remark}
\label{remark:log-linear}
Lemma \ref{lemma:aggregation_error} reflects the impacts of the non-linearity of parameterization and over-parameterization via $\mathcal{O} \left( R_{\theta}^{6/5} m^{-1/10} \hat{R}_{\theta}^{2/5} \right)$. This impact is twofold: First, as $m \to \infty$, this term will eventually disappear, and the aggregation error will be independent of the network structure. Second, as suggested by the proof in Appendix \ref{proof:lemma:aggregation_error}, this term is due to the non-linearity of the network structure, and a log-linear policy does not suffer from this issue.
\end{remark}

\section{Proof of Lemma \ref{lemma:omega}\label{proof:lemma:omega}}

We need the following lemma for proving Lemma \ref{lemma:omega}.
\begin{lemma}
\label{lemma:changing_action_error}
Let $h_{\vartheta}$ be the placeholder for $f_{\theta}$ and $h_{w}$. For parameters $\vartheta, \vartheta^{\prime} \in \mathcal{B}_{R_{\vartheta}}$, state $s \in \mathcal{S}$, and any pair of action $a$ and $a^{\prime}$, we have
\begin{alignat}{1}
\mathbb{E}_{\text{init}} \left[ \left\vert h_{\vartheta}(s,a) - h_{\vartheta^{\prime}}(s,a^{\prime}) \right\vert \right] &\le \sqrt{6 R_{\vartheta}^{2} + 12} \le \sqrt{6} R_{\vartheta} + \mathcal{O} \left( 1 \right). \label{eq:lemma:changing_action_error_0}
\end{alignat}
\end{lemma}
\begin{proof}
By Jensen's inequality, we have
\begin{alignat}{1}
& \mathbb{E}_{\text{init}} \left[ \left\vert h_{\vartheta}(s,a) - h_{\vartheta^{\prime}}(s,a^{\prime}) \right\vert \right]^{2} \nonumber \\
&= \mathbb{E}_{\text{init}} \left[ \left\vert h_{\vartheta}(s,a) - h_{\vartheta^{0}}(s,a) + h_{\vartheta^{0}}(s,a) - h_{\vartheta^{0}}(s,a^{\prime}) + h_{\vartheta^{0}}(s,a^{\prime}) - h_{\vartheta^{\prime}}(s,a^{\prime}) \right\vert \right]^{2} \nonumber \\
&\le \mathbb{E}_{\text{init}} \left[ \left\vert h_{\vartheta}(s,a) - h_{\vartheta^{0}}(s,a) + h_{\vartheta^{0}}(s,a) - h_{\vartheta^{0}}(s,a^{\prime}) + h_{\vartheta^{0}}(s,a^{\prime}) - h_{\vartheta^{\prime}}(s,a^{\prime}) \right\vert^{2} \right]. \nonumber
\end{alignat}
By the fact that $(a + b + c)^{2} \le 3 a^{2} + 3 b^{2} + 3 c^{2}$ and Lemma \ref{lemma:bounded_by_r}, we have
\begin{alignat}{1}
& \mathbb{E}_{\text{init}} \left\vert \left[ h_{\vartheta}(s,a) - h_{\vartheta^{\prime}}(s,a^{\prime}) \right\vert \right]^{2} \nonumber \\
&\le \mathbb{E}_{\text{init}} \left[ 3 \left\vert h_{\vartheta}(s,a) - h_{\vartheta^{0}}(s,a) \right\vert^{2} + 3 \left\vert h_{\vartheta^{0}}(s,a) - h_{\vartheta^{0}}(s,a^{\prime}) \right\vert^{2} + 3 \left\vert h_{\vartheta^{0}}(s,a^{\prime}) - h_{\vartheta^{\prime}}(s,a^{\prime}) \right\vert^{2} \right] \nonumber \\
&\le 6 R_{\vartheta}^{2} + \mathbb{E}_{\text{init}} \left[ 3 \left\vert h_{\vartheta^{0}}(s,a) - h_{\vartheta^{0}}(s,a^{\prime}) \right\vert^{2} \right]. \label{proof:eq:lemma:changing_action_error_0}
\end{alignat}
It remains to bound the second term on the RHS of (\ref{proof:eq:lemma:changing_action_error_0}). By the fact that $(a + b)^{2} \le 2 a^{2} + 2 b^{2}$, we have
\begin{alignat}{1}
3 \mathbb{E}_{\text{init}} \left[ \left\vert h_{\vartheta^{0}}(s,a) - h_{\vartheta^{0}}(s,a^{\prime}) \right\vert^{2} \right] &\le 12 \mathbb{E}_{\text{init}} \left[ \left\vert h_{\vartheta^{0}}(s,a) \right\vert^{2} \right] \nonumber \\
&= 12 \mathbb{E}_{\text{init}} \left[ \frac{1}{m} \left\vert \sum_{i}^{m} b_{i} \cdot \mathbf{1} \left\{ \left( \vartheta^{0}_{i} \right)^{T}(s,a) > 0 \right\} \left( \vartheta^{0}_{i} \right)^{T} (s,a) \right\vert^{2} \right]. \nonumber \\
&\le 12 \mathbb{E}_{\text{init}} \left[ \frac{1}{m} \left( \sum_{i}^{m} \left\vert \left( \vartheta^{0}_{i} \right)^{T} (s,a) \right\vert \right)^{2} \right]. \nonumber
\end{alignat}
By Cauchy-Schwarz inequality, Jensen's inequality, and $\left\Vert (s,a) \right\Vert_{2} \le 1$, we have
\begin{alignat}{1}
3 \mathbb{E}_{\text{init}} \left[ \left\vert h_{\vartheta^{0}}(s,a) - h_{\vartheta^{0}}(s,a^{\prime}) \right\vert^{2} \right] &\le \mathbb{E}_{\text{init}} \left[ \frac{12}{m} \left( \sum_{i}^{m} \left\Vert \vartheta^{0}_{i} \right\Vert \right)^{2} \right] \le 12 \mathbb{E}_{\text{init}} \left[ \sum_{i}^{m} \left\Vert \vartheta^{0}_{i} \right\Vert_{2}^{2} \right] \le 12. \label{proof:eq:lemma:changing_action_error_1}
\end{alignat}
We complete the proof by substituting (\ref{proof:eq:lemma:changing_action_error_1}) into (\ref{proof:eq:lemma:changing_action_error_0}) and taking the square root of both sides.
\end{proof}

Next, we prove Lemma \ref{lemma:omega}.
\begin{proof}
For the Softmax policy, we have
\begin{alignat}{1}
\Omega^{t} &= \left\vert \mathbb{E}_{\text{init}} \left[ \sum_{s} \sum_{n} q_{n} \rho_{\pi^{\ast},n}(s) \left< \log \left( \pi^{\theta^{t+1}}(\cdot \vert s) / \pi^{\theta^{t+1}_{n}}(\cdot \vert s) \right), \pi^{\ast} \right> \right] \right\vert \nonumber \\
&= \left\vert \mathbb{E}_{\text{init}, n \thicksim \mathcal{C}, s \thicksim \rho_{\pi^{\ast},n}, a \thicksim \pi^{\ast}} \left[ \log \pi^{\theta^{t+1}}(s,a) - \log \pi^{\theta_{n}^{t+1}}(s,a) \right] \right\vert \nonumber \\
&= \left\vert \mathbb{E}_{\text{init}, n \thicksim \mathcal{C}, s \thicksim \rho_{\pi^{\ast},n}, a \thicksim \pi^{\ast}} \left[ \tau^{-1}_{t+1} f_{\theta^{t+1}}(s,a) - \tau^{-1}_{t+1} f_{\theta_{n}^{t+1}}(s,a) + \log \sum_{a^{\prime}} \exp \left( \tau^{-1}_{t+1} f_{\theta_{n}^{t+1}}(s,a^{\prime}) \right) \right. \right. \nonumber \\
& \left. \left. \quad - \log \sum_{a^{\prime}} \exp \left( \tau^{-1}_{t+1} f_{\theta^{t+1}}(s,a^{\prime}) \right) \right] \right\vert. \label{eq:proof:lemma:omega0}
\end{alignat}
The last two log-sum-exp terms on the RHS of (\ref{eq:proof:lemma:omega0}) can be bounded as follows. We consider states $s$ with $\log \sum_{a^{\prime}} \exp \left( \tau^{-1}_{t+1} f_{\theta_{n}^{t+1}}(s,a^{\prime}) \right) \ge \log \sum_{a^{\prime}} \exp \left( \tau^{-1}_{t + 1} f_{\theta^{t + 1}}(s,a^{\prime}) \right)$ and the case for the other side is similar. By the log-sum inequality, we can obtain
\begin{alignat}{1}
\log \sum_{a^{\prime}} \exp \left( \tau^{-1}_{t+1} f_{\theta_{n}^{t+1}}(s,a^{\prime}) \right) &\le \max_{a^{\prime} \in \mathcal{A}} \tau^{-1}_{t+1} f_{\theta_{n}^{t+1}}(s,a^{\prime}) + \log \left\vert \mathcal{A} \right\vert. \label{eq:log-sum}
\end{alignat}
By the log-sum-exp trick, we can obtain
\begin{alignat}{1}
- \log \sum_{a^{\prime}} \exp \left( \tau^{-1}_{t+1} f_{\theta^{t+1}}(s,a^{\prime}) \right) &\le - \frac{1}{\left\vert \mathcal{A} \right\vert} \sum_{a^{\prime}} \tau^{-1}_{t+1} f_{\theta^{t+1}}(s,a^{\prime}) - \log \left\vert \mathcal{A} \right\vert. \label{eq:log-sum-exp}
\end{alignat}
For these states $s$, we have the following bound by Lemma \ref{lemma:changing_action_error}. It is easy to verify the same bound for the other states by switching terms, so we omit the related discussion.
\begin{alignat}{1}
\mathbb{E}_{\text{init}} \left[ \left\vert \log \frac{\sum_{a^{\prime}} \exp \left( \tau^{-1}_{t+1} f_{\theta_{n}^{t+1}}(s,a^{\prime}) \right)}{\sum_{a^{\prime}} \exp \left( \tau^{-1}_{t+1} f_{\theta^{t+1}}(s,a^{\prime}) \right)} \right\vert \right] &\le \left\vert \max_{a \in \mathcal{A}} \frac{f_{\theta_{n}^{t+1}}(s,a)}{\tau_{t+1}} - \frac{\sum_{a^{\prime}} f_{\theta^{t+1}}(s,a^{\prime})}{\tau_{t+1} \left\vert \mathcal{A} \right\vert} \right\vert \nonumber \\
&\le \sqrt{6} \tau_{t+1}^{-1} R_{\vartheta} + \mathcal{O} \left( \tau_{t+1}^{-1} \right). \label{eq:proof:lemma:omega1}
\end{alignat}
By Lemma \ref{lemma:aggregation_error} and substituting (\ref{eq:proof:lemma:omega1}) into (\ref{eq:proof:lemma:omega0}), we can obtain
\begin{alignat}{1}
\Omega^{t} &\le \mathbb{E}_{\text{init}, n \thicksim \mathcal{C}, s \thicksim \rho_{\pi^{\ast},n}, a \thicksim \pi^{\ast}} \left[ \left\vert \frac{f_{\theta^{t+1}}(s,a)}{\tau_{t+1}} - \frac{f_{\theta_{n}^{t+1}}(s,a)}{\tau_{t+1}} \right\vert + \left\vert \log \frac{\sum_{a^{\prime}} \exp \left( \tau^{-1}_{t+1} f_{\theta_{n}^{t+1}}(s,a^{\prime}) \right)}{\sum_{a^{\prime}} \exp \left( \tau^{-1}_{t+1} f_{\theta^{t+1}}(s,a^{\prime}) \right)} \right\vert \right] \nonumber \\
&\le 2 \mathbb{E}_{n \thicksim \mathcal{C}} \left[ \tau_{t+1}^{-1} \epsilon^{t+1}_{n} \xi^{t+1}_{n} \right] + \beta_{t}^{-1} \mathbb{E}_{n \thicksim \mathcal{C}, s \thicksim \rho_{\pi^{\ast},n}, a \thicksim \pi^{\ast}} \left[ \sum_{i} q_{i} \left\vert Q^{w^{t}_{n}}(s,a) - Q^{w^{t}_{i}}(s,a) \right\vert \right] \nonumber \\
& \quad + \sqrt{6} \tau_{t+1}^{-1} R_{\theta} + \mathcal{O} \left( \tau_{t+1}^{-1} R_{\theta}^{6/5} m^{-1/10} \hat{R}_{\theta}^{2/5} \right) \nonumber
\end{alignat}
for the baseline and
\begin{alignat}{1}
\Omega^{t+1} &\le 2 \mathbb{E}_{n \thicksim \mathcal{C}} \left[ \tau_{t+1}^{-1} \epsilon^{t+1}_{n} \xi^{t+1}_{n} \right] + \sqrt{6} \tau_{t+1}^{-1} R_{\theta} + \mathcal{O} \left( \tau_{t+1}^{-1} R_{\theta}^{6/5} m^{-1/10} \hat{R}_{\theta}^{2/5} \right) \nonumber
\end{alignat}
for FedRAC. For the Gaussian policy, we have
\begin{alignat}{1}
\Omega^{t} &= \left\vert \mathbb{E}_{\text{init}} \left[ \sum_{s} \sum_{n} q_{n} \rho_{\pi^{\ast},n}(s) \left< \log \left( \pi^{\theta^{t+1}}(\cdot \vert s) / \pi^{\theta^{t+1}_{n}}(\cdot \vert s) \right), \pi^{\ast} \right> \right] \right\vert \nonumber \\
&= \left\vert \mathbb{E}_{\text{init}, n \thicksim \mathcal{C}, s \thicksim \rho_{\pi^{\ast},n}, a \thicksim \pi^{\ast}} \left[ \log \pi^{\theta_{n}^{t+1}} - \log \pi^{\theta^{t+1}} \right] \right\vert \nonumber \\
&= \frac{1}{2 \sigma_{t+1}^{2}} \left\vert \mathbb{E}_{n \thicksim \mathcal{C}, s \thicksim \rho_{\pi^{\ast},n}, a \thicksim \pi^{\ast}} \left[ \left( a - f_{\theta^{t+1}}(s,a) \right)^2 - \left( a - f_{\theta^{t+1}_{n}}(s,a) \right)^2 \right] \right\vert. \nonumber
\end{alignat}
By utilizing Taylor expansion, we can obtain
\begin{alignat}{1}
\Omega^{t} &= \frac{1}{2 \sigma_{t+1}^{2}} \left\vert \mathbb{E}_{\text{init}, n \thicksim \mathcal{C}, s \thicksim \rho_{\pi^{\ast},n}, a \thicksim \pi^{\ast}} \left[ \left( a - f_{\theta^{0}}(s,a) \right)^2 + 2 \left( f_{\theta^{0}}(s,a) - a \right) \left( f_{\theta^{t+1}}(s,a) - f_{\theta^{0}}(s,a) \right) \right. \right. \nonumber \\
& \quad + 2 \left( f_{\theta^{t+1}}(s,a) - f_{\theta^{0}}(s,a) \right)^{2} - \left( a - f_{\theta^{0}}(s,a) \right)^2 - 2 \left( f_{\theta^{0}}(s,a) - a \right) \left( f_{\theta^{t+1}_{n}}(s,a) - f_{\theta^{0}}(s,a) \right) \nonumber \\
& \quad \left. \left. - 2 \left( f_{\theta^{t+1}_{n}}(s,a) - f_{\theta^{0}}(s,a) \right)^2 \right] \right\vert. \nonumber
\end{alignat}
Rearranging these terms, we have
\begin{alignat}{1}
\Omega^{t} &= \frac{1}{2 \sigma_{t+1}^{2}} \left\vert \mathbb{E}_{\text{init}, n \thicksim \mathcal{C}, s \thicksim \rho_{\pi^{\ast},n}, a \thicksim \pi^{\ast}} \left[ 2 \left( f_{\theta^{0}}(s,a) - a \right) \left( f_{\theta^{t+1}}(s,a) - f_{\theta^{t+1}_{n}}(s,a) \right) \right. \right. \nonumber \\
& \quad \left. \left. + 2 \left( f_{\theta^{t+1}}(s,a) - f_{\theta^{0}}(s,a) \right)^{2} - 2 \left( f_{\theta^{t+1}_{n}}(s,a) - f_{\theta^{0}}(s,a) \right)^2 \right] \right\vert \nonumber \\
&\le \left\vert \mathbb{E}_{\text{init}, n \thicksim \mathcal{C}, s \thicksim \rho_{\pi^{\ast},n}, a \thicksim \pi^{\ast}} \left[ \frac{f_{\theta^{0}}(s,a) - a}{\sigma_{t+1}^{2}} \left( f_{\theta^{t+1}}(s,a) - f_{\theta^{t+1}_{n}}(s,a) \right) \right] \right\vert + \frac{R_{\theta}^{2}}{\sigma_{t+1}^{2}} \nonumber \\
&\le \mathbb{E}_{\text{init}, n \thicksim \mathcal{C}, s \thicksim \rho_{\pi^{\ast},n}, a \thicksim \pi^{\ast}} \left[ \left\vert \upsilon_{t+1}^{-1}(s,a) f_{\theta^{t+1}}(s,a) - \upsilon_{t+1}^{-1}(s,a) f_{\theta^{t+1}_{n}}(s,a) \right\vert \right] + \frac{R_{\theta}^{2}}{\sigma_{t+1}^{2}}. \nonumber
\end{alignat}
By Lemma \ref{lemma:aggregation_error}, we can obtain
\begin{alignat}{1}
\Omega^{t} &\le 2 \mathbb{E}_{n \thicksim \mathcal{C}} \left[ \sup \left\vert \upsilon_{t+1}^{-1}(s,a) \right\vert \tilde{\epsilon}_{n}^{t+1} \xi^{t+1}_{n} \right] + \mathcal{O} \left( \sup \left\vert \upsilon_{t+1}^{-1}(s,a) \right\vert R_{\theta}^{6/5} m^{-1/10} \hat{R}_{\theta}^{2/5} \right) \nonumber \\
& \quad + \beta_{t}^{-1} \mathbb{E}_{n \thicksim \mathcal{C}, s \thicksim \rho_{\pi^{\ast},n}, a \thicksim \pi^{\ast}} \left[ \sum_{i} q_{i} \left\vert Q^{w^{t}_{n}}(s,a) - Q^{w^{t}_{i}}(s,a) \right\vert \right] + \frac{R_{\theta}^{2}}{\sigma_{t+1}^{2}} \nonumber
\end{alignat}
for the baseline and
\begin{alignat}{1}
\Omega^{t} &\le 2 \mathbb{E}_{n \thicksim \mathcal{C}} \left[ \sup \left\vert \upsilon_{t+1}^{-1}(s,a) \right\vert \tilde{\epsilon}_{n}^{t+1} \xi^{t+1}_{n} \right] + \mathcal{O} \left( \sup \left\vert \upsilon_{t+1}^{-1}(s,a) \right\vert R_{\theta}^{6/5} m^{-1/10} \hat{R}_{\theta}^{2/5} \right) + \frac{R_{\theta}^{2}}{\sigma_{t+1}^{2}} \nonumber
\end{alignat}
for FedRAC.
\end{proof}

\section{Proof of Theorem \ref{theorem:global_convergence_rate}\label{proof:theorem:global_convergence_rate}}

\begin{proof}
For the baseline with action-value function $Q^{w^{t}_{n}}$, by adding and subtracting $D_{KL}(\pi^{\ast}(\cdot \vert s) \Vert \pi^{\theta^{t+1}_{n}}(\cdot \vert s))$, we have
\begin{alignat}{1}
& D_{KL}(\pi^{\ast}(\cdot \vert s) \Vert \pi^{\theta^{t}}(\cdot \vert s)) - D_{KL}(\pi^{\ast}(\cdot \vert s) \Vert \pi^{\theta^{t+1}}(\cdot \vert s)) \nonumber \\
&= \left< \log \left( \pi^{\theta^{t+1}}(\cdot \vert s) / \pi^{\theta^{t+1}_{n}}(\cdot \vert s) \right), \pi^{\ast}(\cdot \vert s) \right> + \beta_{t}^{-1} \left< Q_{n}^{\pi^{\theta^{t}}}(s,\cdot), \pi^{\ast}(\cdot \vert s) - \pi^{\theta^{t}}(\cdot \vert s) \right> \nonumber \\
& \quad + D_{KL}(\pi^{\theta_{n}^{t+1}}(\cdot \vert s) \Vert \pi^{\theta^{t}}(\cdot \vert s)) + \left< \log \left( \pi^{\theta^{t+1}_{n}}(\cdot \vert s) / \pi^{\theta^{t}}(\cdot \vert s) \right), \pi^{\theta_{t}}(\cdot \vert s) - \pi^{\theta^{t+1}_{n}}(\cdot \vert s) \right> \nonumber \\
& \quad + \left< \log \left( \pi^{\theta^{t+1}_{n}}(\cdot \vert s) / \pi^{\theta^{t}}(\cdot \vert s) \right) - \beta_{t}^{-1} Q_{n}^{\pi^{\theta^{t}}}(s,\cdot), \pi^{\ast}(\cdot \vert s) - \pi^{\theta^{t}}(\cdot \vert s) \right>. \nonumber
\end{alignat}
By taking expectation with respect to initialization, $n \thicksim \mathcal{C}$, and $s \thicksim \rho_{\pi^{\ast},n}$, we have
\begin{alignat}{1}
& \mathbb{E}_{\text{init}, n \thicksim \mathcal{C}, s \thicksim \rho_{\pi^{\ast},n}} \left[ D_{KL}(\pi^{\ast}(\cdot \vert s) \Vert \pi^{\theta^{t+1}}(\cdot \vert s)) - D_{KL}(\pi^{\ast}(\cdot \vert s) \Vert \pi^{\theta^{t}}(\cdot \vert s)) \right] \nonumber \\
&= \mathbb{E}_{\text{init}, n \thicksim \mathcal{C}, s \thicksim \rho_{\pi^{\ast},n}} \left[ - \left< \log \left( \pi^{\theta^{t+1}_{n}}(\cdot \vert s) / \pi^{\theta^{t}}(\cdot \vert s) \right) - \beta_{t}^{-1} Q_{n}^{\pi^{\theta^{t}}}(s,\cdot), \pi^{\ast}(\cdot \vert s) - \pi^{\theta^{t}}(\cdot \vert s) \right> \right. \nonumber \\
& \quad - \beta_{t}^{-1} \left< Q_{n}^{\pi^{\theta^{t}}}(s,\cdot), \pi^{\ast}(\cdot \vert s) - \pi^{\theta^{t}}(\cdot \vert s) \right> - \left< \log \left( \pi^{\theta^{t+1}}(\cdot \vert s) / \pi^{\theta^{t+1}_{n}}(\cdot \vert s) \right), \pi^{\ast}(\cdot \vert s) \right> \nonumber \\
& \quad \left. - D_{KL}(\pi^{\theta_{n}^{t+1}}(\cdot \vert s) \Vert \pi^{\theta^{t}}(\cdot \vert s)) - \left< \log \left( \pi^{\theta^{t+1}_{n}}(\cdot \vert s) / \pi^{\theta^{t}}(\cdot \vert s) \right), \pi^{\theta_{t}}(\cdot \vert s) - \pi^{\theta^{t+1}_{n}}(\cdot \vert s) \right> \right]. \label{eq:proof:theorem:global_convergence_rate0}
\end{alignat}

For FedRAC with action-value function $Q^{w^{t}}$, we have
\begin{alignat}{1}
& D_{KL}(\pi^{\ast}(\cdot \vert s) \Vert \pi^{\theta^{t}}(\cdot \vert s)) - D_{KL}(\pi^{\ast}(\cdot \vert s) \Vert \pi^{\theta^{t+1}}(\cdot \vert s)) \nonumber \\
&= \left< \log \left( \pi^{\theta^{t+1}_{n}}(\cdot \vert s) / \pi^{\theta^{t}}(\cdot \vert s) \right) - \beta_{t}^{-1} \sum_{i} \frac{q_{i} \rho_{\pi^{\ast},i}(s)}{Z_{\pi^{\ast}}(s)} Q_{i}^{\pi^{\theta^{t}}}(s,\cdot), \pi^{\ast}(\cdot \vert s) - \pi^{\theta^{t}}(\cdot \vert s) \right> \nonumber \\
& \quad + \beta_{t}^{-1} \left< \sum_{i} \frac{q_{i} \rho_{i}(s)}{Z_{\pi^{\ast}}(s)} Q_{i}^{\pi^{\theta^{t}}}(s,\cdot), \pi^{\ast}(\cdot \vert s) - \pi^{\theta^{t}}(\cdot \vert s) \right> + \left< \log \left( \pi^{\theta^{t+1}}(\cdot \vert s) / \pi^{\theta^{t+1}_{n}}(\cdot \vert s) \right), \pi^{\ast}(\cdot \vert s) \right> \nonumber \\
& \quad + D_{KL}(\pi^{\theta_{n}^{t+1}}(\cdot \vert s) \Vert \pi^{\theta^{t}}(\cdot \vert s)) + \left< \log \left( \pi^{\theta^{t+1}_{n}}(\cdot \vert s) / \pi^{\theta^{t}}(\cdot \vert s) \right), \pi^{\theta_{t}}(\cdot \vert s) - \pi^{\theta^{t+1}_{n}}(\cdot \vert s) \right>. \nonumber
\end{alignat}
By taking expectation with respect to initialization, $n \thicksim \mathcal{C}$, and $s \thicksim \rho_{\pi^{\ast},n}$, we have
\begin{alignat}{1}
& \mathbb{E}_{\text{init}, n \thicksim \mathcal{C}, s \thicksim \rho_{\pi^{\ast},n}} \left[ D_{KL}(\pi^{\ast}(\cdot \vert s) \Vert \pi^{\theta^{t+1}}(\cdot \vert s)) - D_{KL}(\pi^{\ast}(\cdot \vert s) \Vert \pi^{\theta^{t}}(\cdot \vert s)) \right] \nonumber \\
&= \mathbb{E}_{\text{init}, n \thicksim \mathcal{C}, s \thicksim \rho_{\pi^{\ast},n}} \left[ - D_{KL}(\pi^{\theta_{n}^{t+1}}(\cdot \vert s) \Vert \pi^{\theta^{t}}(\cdot \vert s)) - \left< \log \left( \pi^{\theta^{t+1}_{n}}(\cdot \vert s) / \pi^{\theta^{t}}(\cdot \vert s) \right), \pi^{\theta_{t}}(\cdot \vert s) - \pi^{\theta^{t+1}_{n}}(\cdot \vert s) \right> \right. \nonumber \\
& \quad - \beta_{t}^{-1} \left< \sum_{i} \frac{q_{i} \rho_{\pi^{\ast},i}(s)}{Z_{\pi^{\ast}}(s)} Q_{i}^{\pi^{\theta^{t}}}(s,\cdot), \pi^{\ast}(\cdot \vert s) - \pi^{\theta^{t}}(\cdot \vert s) \right> - \left< \log \left( \pi^{\theta^{t+1}}(\cdot \vert s) / \pi^{\theta^{t+1}_{n}}(\cdot \vert s) \right), \pi^{\ast}(\cdot \vert s) \right> \nonumber \\
& \quad \left. - \left< \log \left( \pi^{\theta^{t+1}_{n}}(\cdot \vert s) / \pi^{\theta^{t}}(\cdot \vert s) \right) - \beta_{t}^{-1} \sum_{n} \frac{q_{i} \rho_{\pi^{\ast},i}(s)}{Z_{\pi^{\ast}}(s)} Q_{i}^{\pi^{\theta^{t}}}(s,\cdot), \pi^{\ast}(\cdot \vert s) - \pi^{\theta^{t}}(\cdot \vert s) \right> \right]. \label{eq:proof:theorem:global_convergence_rate1}
\end{alignat}
By the Pinsker's inequality and rearranging (\ref{eq:proof:theorem:global_convergence_rate0}) and (\ref{eq:proof:theorem:global_convergence_rate1}), we can obtain
\begin{alignat}{1}
& \mathbb{E}_{\text{init}, n \thicksim \mathcal{C}, s \thicksim \rho_{\pi^{\ast},n}} \left[ D_{KL}(\pi^{\ast}(\cdot \vert s) \Vert \pi^{\theta^{t+1}}(\cdot \vert s)) - D_{KL}(\pi^{\ast}(\cdot \vert s) \Vert \pi^{\theta^{t}}(\cdot \vert s)) \right] \nonumber \\
&\le - \beta_{t}^{-1} \mathbb{E}_{\text{init}} \left[ (\eta(\pi^{\ast}) - \eta(\pi^{\theta^{t}})) \right] + \varepsilon^{t} + \Omega^{t} - \frac{1}{2} \mathbb{E}_{\text{init}} \left\Vert \pi^{\theta_{n}^{t+1}}(\cdot \vert s) - \pi^{\theta^{t}}(\cdot \vert s) \right\Vert_{1}^{2} \nonumber \\
& \quad + \mathbb{E}_{\text{init}} \left[ \left\Vert \log \left( \pi^{\theta^{t+1}_{n}}(\cdot \vert s) \right) - \log \left( \pi^{\theta^{t}}(\cdot \vert s) \right) \right\Vert_{\infty} \cdot \left\Vert \pi^{\theta_{t}}(\cdot \vert s) - \pi^{\theta^{t+1}_{n}}(\cdot \vert s) \right\Vert_{1} \right] \nonumber \\
&\le - \beta_{t}^{-1} \mathbb{E}_{\text{init}} \left[ (\eta(\pi^{\ast}) - \eta(\pi^{\theta^{t}})) \right] + \varepsilon^{t} + \Omega^{t} + \frac{1}{2} \mathbb{E}_{\text{init}} \left\Vert \log \left( \pi^{\theta^{t+1}_{n}}(\cdot \vert s) / \pi^{\theta^{t}}(\cdot \vert s) \right) \right\Vert_{\infty}^{2}, \label{eq:proof:theorem:global_convergence_rate2}
\end{alignat}
where we have used the fact that $2xy - y^{2} \le x^{2}$. Rearranging and telescoping (\ref{eq:proof:theorem:global_convergence_rate2}), we finally obtain
\begin{alignat}{1}
\min_{0 \le t \le T} \mathbb{E}_{\text{init}} \left[ \eta(\pi^{\ast}) - \eta(\pi^{\theta^{t}}) \right] \le \frac{D_{KL}^{\ast,0} + \sum_{t=1}^{T}(\varepsilon^{t} + M^{t} + \Omega^{t})}{\sum_{t=1}^{T}\beta_{t}^{-1}}. \nonumber
\end{alignat}
\end{proof}

\section{Derivation of the Optimal Convergence Rate\label{proof:eq:best_convergence_rate}}

\begin{proof}
We consider the case with a Softmax policy for ease of illustration. By Lemmas \ref{lemma:error_propagation}, \ref{lemma:stepwise_energy_difference}, \ref{lemma:omega}, and Theorem \ref{theorem:global_convergence_rate}, we have
\begin{alignat}{1}
& \min_{0 \le t \le T} \mathbb{E}_{\text{init}} \left[ \eta(\pi^{\ast}) - \eta(\pi^{\theta^{t}}) \right] \nonumber \\
&\le \frac{D_{KL}^{\ast,0} + \sum_{t=1}^{T} \left( \varepsilon^{t} + M^{t} + \Omega^{t} \right)}{\sum_{t=1}^{T}\beta_{t}^{-1}} \nonumber \\
&= \frac{\log \left\vert \mathcal{A} \right\vert + \sum_{t=1}^{T} \frac{20 R_{\theta}^{2} + 24}{\tau_{t+1}^{2}} + \sum_{t=1}^{T} \frac{\mathcal{O} \left( R_{\theta}^{6/5} m^{-1/10} \hat{R}_{\theta}^{2/5} \right) }{\tau_{t+1}} + \sum_{t=1}^{T} \frac{\epsilon^{t} \psi^{t}}{\tau_{t+1}} + \sum_{t=1}^{T} \beta^{-1}_{t} \dot{\epsilon}^{t} \psi^{t}}{\sum_{t=1}^{T}\beta_{t}^{-1}} \nonumber \\
& \quad + \frac{\sum_{t=1}^{T} \frac{2 \epsilon^{t} \xi^{t}}{\tau_{t+1}} + \sum_{t=1}^{T} p \beta_{t}^{-1} \kappa + \sum_{t=1}^{T} 8 \left( f_{\theta^{0}}^{2} + R_{\theta}^{2} \right) \left( \tau_{t+1}^{-1} - \tau_{t}^{-1} \right)^{2} + \sum_{t=1}^{T} \frac{\sqrt{6} R_{\theta}}{\tau_{t+1}}}{\sum_{t=1}^{T}\beta_{t}^{-1}}. \nonumber
\end{alignat}
By setting $\beta_{t} = \beta \sqrt
T$ and $\tau_{t+1} = \frac{\beta T^{2}}{t + 1}$, we have
\begin{alignat}{1}
& \min_{0 \le t \le T} \mathbb{E}_{\text{init}} \left[ \eta(\pi^{\ast}) - \eta(\pi^{\theta^{t}}) \right] \nonumber \\
&\le \frac{\log \left\vert \mathcal{A} \right\vert + \sum_{t=1}^{T} \frac{20 R_{\theta}^{2} + 24}{\tau_{t+1}^{2}} + \sum_{t=1}^{T} \frac{\mathcal{O} \left( R_{\theta}^{6/5} m^{-1/10} \hat{R}_{\theta}^{2/5} \right) }{\tau_{t+1}} + \sum_{t=1}^{T} \frac{\epsilon^{t} \psi^{t}}{\tau_{t+1}} + \sum_{t=1}^{T} \beta^{-1}_{t} \dot{\epsilon}^{t} \psi^{t}}{\beta^{-1}\sqrt{T}} \nonumber \\
& \quad + \frac{\sum_{t=1}^{T} \frac{2 \epsilon^{t} \xi^{t}}{\tau_{t+1}} + \sum_{t=1}^{T} p \beta_{t}^{-1} \kappa + \sum_{t=1}^{T} 8 \left( f_{\theta^{0}}^{2} + R_{\theta}^{2} \right) \beta^{-2} T^{-4} + \sum_{t=1}^{T} \frac{\sqrt{6} R_{\theta}}{\tau_{t+1}}}{\beta^{-1}\sqrt{T}} \nonumber \\
&= \frac{\log \left\vert \mathcal{A} \right\vert + \left( 20 R_{\theta}^{2} + 24 \right) \beta^{-2} \sum_{t=1}^{T} \frac{(t+1)^2}{T^4}}{\beta^{-1}\sqrt{T}} + \mathcal{O} \left( \frac{R_{\theta}^{6/5} m^{-1/10} \hat{R}_{\theta}^{2/5}}{\sqrt{T}} \right) \nonumber \\
& \quad + \frac{\sum^{T}_{t=1} \sum^{N}_{n=1} q_{n} \epsilon_{n}^{t+1} \left( \psi^{t}_{n} + 2 \xi^{t}_{n} \right)}{T^{\frac{3}{2}}} + \frac{\sum^{T}_{t=1} \sum^{N}_{n=1} q_{n} \dot{\epsilon}_{n}^{t} \psi^{t}_{n}}{T} + \mathcal{O} \left( \frac{1}{T^{3}} \right) + p \kappa + \frac{\sqrt{6} R_{\theta}}{\sqrt{T}} \nonumber \\
&\le \frac{\log \left\vert \mathcal{A} \right\vert + \left( 20 R_{\theta}^{2} + 24 \right) \beta^{-2} T^{-1}}{\beta^{-1}\sqrt{T}} + \mathcal{O} \left( \frac{R_{\theta}^{6/5} m^{-1/10} \hat{R}_{\theta}^{2/5}}{\sqrt{T}} + \frac{R_{\theta}}{\sqrt{T}} \right) \nonumber \\
& \quad + \frac{\sum^{T}_{t=1} \sum^{N}_{n=1} q_{n} \epsilon_{n}^{t+1} \left( \psi^{t}_{n} + 2 \xi^{t}_{n} \right)}{T^{\frac{3}{2}}} + \frac{\sum^{T}_{t=1} \sum^{N}_{n=1} q_{n} \dot{\epsilon}_{n}^{t} \psi^{t}_{n}}{T} + \mathcal{O} \left( \frac{1}{T^{3}} \right) + p \kappa. \nonumber
\end{alignat}
By the optimal condition, we have $\beta = 2 \sqrt{\frac{ 5 R_{\theta}^{2} + 6 }{T \log \left\vert \mathcal{A} \right\vert}}$, and the rate follows.
\end{proof}

\end{document}